\documentclass[reqno, 10pt, a4paper]{amsart}
\def\code#1{\texttt{#1}}
\usepackage[square,numbers]{natbib}
\usepackage[utf8]{inputenc} 
\usepackage[T1]{fontenc}    
\usepackage{url}            
\usepackage{booktabs}       
\usepackage{amsfonts}       
\usepackage{nicefrac}       
\usepackage{microtype}      
\usepackage{mathrsfs}
\usepackage{yfonts}
\allowdisplaybreaks

\input{style}
\author{Yves Rychener}
\author{Bahar Ta{\c{s}}kesen}
\author{Daniel Kuhn}
\thanks{The authors are with the Risk Analytics and Optimization Laboratory, EPFL Lausanne~(\texttt{yves.rychener,bahar.taskesen,} \texttt{daniel.kuhn@epfl.ch})}

\begin{document}
\title{Metrizing Fairness}

\begin{abstract}
    We study supervised learning problems that have significant effects on individuals from two demographic groups, and we seek predictors that are fair with respect to a group fairness criterion such as statistical parity (SP). A predictor is SP-fair if the distributions of predictions within the two groups are close in Kolmogorov distance, and fairness is achieved by penalizing the dissimilarity of these two distributions in the objective function of the learning problem. In this paper, we identify conditions under which hard SP constraints are guaranteed to improve predictive accuracy. We also showcase conceptual and computational benefits of measuring unfairness with integral probability metrics (IPMs) other than the Kolmogorov distance. Conceptually, we show that the generator of any IPM can be interpreted as a family of utility functions and that unfairness with respect to this IPM arises if individuals in the two demographic groups have diverging expected utilities. We also prove that the unfairness-regularized prediction loss admits unbiased gradient estimators, which are constructed from random mini-batches of training samples, if unfairness is measured by the squared $\mc L^2$-distance or by a squared maximum mean discrepancy. In this case, the fair learning problem is susceptible to efficient stochastic gradient descent (SGD) algorithms. Numerical experiments on synthetic and real data show that these SGD algorithms outperform state-of-the-art methods for fair learning in that they achieve superior accuracy-unfairness trade-offs---sometimes orders of magnitude faster. 
\end{abstract}

\maketitle


\section{Introduction}
The last decade has witnessed a surge of algorithms that have a consequential impact on our daily lives. Machine learning methods are increasingly used, for example, to decide whom to grant or deny loans, college admission, bail or parole. Even though it would be natural to expect that algorithms are free of prejudice, it turns out that cutting-edge AI techniques can learn or even amplify human biases and may thus be far from fair~\citep{ref:chouldechova2017fair,ref:dastin2018amazon,ref:propublica}. The necessity to correct algorithmic biases has propelled the growing field of fair machine learning; see, {\em e.g.}, the recent surveys~\citep{ref:barocas-hardt-narayanan,ref:berk2018fairness, ref:chouldechova2020snapshot, ref:corbett2017algorithmic, ref:mehrabi2019survey}. As of today, there exist several mathematical definitions of algorithmic fairness. All of these definitions fall into one of three main categories: (i)~notions of {\em group fairness} ask that different demographic groups have equal chances of securing beneficial outcomes \citep{ref:barocas2016big, ref:feldman2015certifying, ref:hardt2016equality, ref:pleiss2017fairness,ref:williamson2019fairness, ref:zafar2017fairness2,ref:zafar2017fairness}, (ii) notions of {\em individual fairness} demand that individuals with similar covariates should be treated similarly~\citep{ref:dwork2012fairness, ref:sharifi2019average, ref:yurochkin2020training}, and (iii) {\em causality-based fairness} notions require that protected attributes such as gender or race have no causal effect on outcomes~\citep{ref:garg2019counterfactual, ref:johnson2016impartial, ref:kusner2017counterfactual, ref:nabi2018fair,  ref:zhang2018equality, ref:zhang2018fairness}. This paper focuses on notions of group fairness, which are widely used in applications. Prominent group fairness criteria include statistical parity~(also known as demographic parity)~\citep{ ref:calders2013controlling, ref:feldman2015certifying, ref:zafar2017fairness}, equal opportunity~\citep{ref:hardt2016equality} and equalized odds~\citep{ref:hardt2016equality} as well as their probabilistic counterparts~\citep{ref:pleiss2017fairness}. All of these criteria quantify fairness in terms of the distributions of the algorithms' outputs conditional on the attributes of different demographic groups.
     
\paragraph{\textbf{Contributions.}}
We study supervised learning problems whose solutions critically impact individuals of two demographic groups, and we seek predictors that are fair with respect to one of several group fairness criteria. For ease of exposition, we will focus on statistical parity, but we will show that most of our main results readily extend to other group fairness notions. A predictor is fair with respect to statistical parity if the distributions of the predictions within the two groups are close in Kolmogorov distance, which constitutes an IPM. Fairness can thus be enforced by penalizing the dissimilarity of the two distributions in the objective function of the learning problem. In this paper we first prove that if membership in a certain demographic group provides no information about the distribution of the learning target conditional on the features, then the optimal predictor satisfies statistical parity even if it is not explicitly enforced. {\color{black} In addition, we prove that if training is based on a {\em biased} target with a sufficiently small bias, then fairness constraints are guaranteed to have a regularizing effect and to improve accuracy in view of the {\em true} target.} This reasoning provides a theoretical justification for enforcing statistical parity and may help to identify learning tasks in which fairness penalties produce desirable outcomes. Next, we flesh out conceptual and computational benefits of measuring unfairness with IPMs other than the Kolmogorov distance. On the one hand, we show that the generator of any IPM can be interpreted as a family of utility functions and that unfairness with respect to this IPM arises if individuals in the two demographic groups have diverging expected utilities. This establishes a utilitarian perspective on unfairness measures. On the other hand, we prove that unfairness-regularized learning problems are susceptible to efficient SGD algorithms if unfairness is measured by the squared $\mc L^2$-distance or by a squared maximum mean discrepancy. Unbiased gradient estimators, which are necessary for SGD to converge, are difficult to obtain because the unfairness penalty is non-linear in the distribution of the training samples. Focusing on IPMs related to $U$-statistics and constructing random batches of training samples, we can eliminate the systematic bias of the na\"ive empirical unfairness penalty. However, debiasing the unfairness penalty in this manner introduces a bias in the empirical prediction loss. We thus also need to derive a bias correction term for the empirical prediction loss. In summary, these techniques allow us to obtain unbiased gradient estimators for the overall learning objective at a low computational cost. Numerical experiments on real data show that our SGD-based approach to solve fair learning problems outperforms state-of-the-art methods in that it achieves a superior accuracy-fairness trade-off---sometimes at significantly reduced runtimes.

\paragraph{\textbf{Related work.}}
Fair supervised learning models can be categorized into three main groups: ~(i)~{\em preprocessing methods} correct biased data before the training process~\citep{ref:feldman2015certifying, ref:kamiran2012data,ref:fairgan, ref:zemel2013learning, ref:NEURIPS2021_64ff7983}, (ii)~{\em in-processing methods} incorporate fairness requirements into the training process ~\citep{ref:donini2018empirical, ref:fish2016confidence, ref:kamishima2012fairness, ref:zafar2017fairness2, ref:zafar2017fairness, roh2021fairbatch}, and (iii)~{\em post-processing methods} adjust a trained predictor to obey some desired fairness requirement~\citep{chzhen2020wsbary_reg, ref:hardt2016equality, ref:pleiss2017fairness, ref:NEURIPS2021_d9fea4ca}. 
The methods proposed in this paper add an unfairness regularizer in the form of an IPM to the training objective, and thus they belong to the \emph{in}-processing methods. More specific regularization schemes for incentivizing fairness are described in~\cite{ref:berk2017convex, jiang2020wassersteinClassif, ref:oneto2020expoliting, ref:taskesen2020distributionally}. 
The possibility to quantify unfairness via generic IPMs was informally mentioned in~\cite{ref:pfohl2021empirical} but not systematically investigated. Fairness can also be enforced rigidly via hard constraints on the predictors~\citep{ref:donini2018empirical, ref:wu2019convexity, ref:zafar2017fairness, ref:zafar2019fairness}. As the resulting constrained optimization problems are often non-convex and unsuitable for gradient descent algorithms, however, hard fairness constraints are typically relaxed in practice~\citep{ref:donini2018empirical, ref:oneto2020expoliting, ref:taskesen2020distributionally, ref:wu2019convexity, ref:zafar2017fairness}. The extent to which the solutions of these relaxed problems comply with the original fairness constraints is discussed in~\citep{ref:lohaus2020}. Unfairness penalties based on maximum mean deviations are used in~\cite{ref:oneto2020expoliting} to enforce statistical parity at an intermediate layer of a neural network predictor. As no unbiased gradient estimators are derived, however, the corresponding learning problems cannot be addressed with scalable SGD methods, which are indispensible for latent variable models~\cite{tucker2017rebar}, unrolled computations~\cite{vicol2021unbiased}, federated learning~\cite{yao2019federated}, distributionally robust optimization~\cite{ghosh2020unbiased} and generative adversarial networks~\cite{bellemare2017cramer}. 

This paper is structured as follows. Section~\ref{sec:group-fairness} reviews and unifies various group fairness criteria, whereas Section~\ref{sec:ipms} formalizes the connections between unfairness measures and IPMs. Section~\ref{sec:num_fair_learning} then addresses the numerical solution of learning problems with unfairness penalties, and Section~\ref{sec:numerical} reports on numerical results. Finally, Section~\ref{sec:conclusion} concludes. Proofs and additional background material are relegated to the appendix.

\paragraph{\textbf{Notation.}} All random vectors are defined on a probability space~$(\Omega, \mc F, \PP)$, 
and the expectation with respect to~$\PP$ is denoted by~$\EE[\cdot]$. Random vectors are denoted by capital letters ({\em e.g., $X$}), and their realizations are denoted by the corresponding lower-case letters ({\em e.g., $x$}). The cumulative distribution function (CDF) of a random vector~$X\in\R$ is denoted by~$F_X$ and satisfies~$F_X(x)=\PP[X\leq x]$ for all~$x\in\R$. Similarly, the probability distribution of~$X$ is denoted by~$\PP_X$ and satisfies~$\PP_X[\mc B]=\PP[X\in \mc B]$ for all Borel sets~$\mc B\subseteq\R^d$. We write $X\perp Y$ to indicate that the random vectors~$X$ and~$Y$ are independent under~$\PP$. For any Borel sets~$\mc X\subseteq \R^{d_\mc X}$ and~$\mc Y\subseteq \R^{d_\mc Y}$ we denote by~$\mc L(\mc X, \mc Y)$ the space of all Borel-measurable functions from~$\mc X$ to~$\mc Y$.
Given a norm~$\|\cdot\|$ on~$\R^d$, the Lipschitz modulus of~$f:\R^d \to \R$ is defined as $\lip(f) = \sup_{x, x' \in\R^d} \{|f(x)- f(x')| / \|x-x'\| : x \neq x'\}$. The indicator function~$\mathbbm 1_\mc S$ of a logical statement~$\mc S$ evaluates to~$1$ if~$\mc S$ is true and to~$0$ otherwise. 

\section{Fairness in Supervised Learning}
\label{sec:group-fairness}
We study regression and classification problems of the form
\begin{equation}
    \min\limits_{h \in \mc H} \EE[L (h(X), Y)]
    \label{eq:loss-min}
\end{equation}
that aim to predict a property~$Y \in \mc Y\subseteq\R$ (the output) of a human being characterized by a feature vector~$X\in\mc X\subseteq \R^d$ (the input). Here, $\mc H$ represents a family of Borel-measurable hypotheses~$h:\mc X\to\mathbb{R}$, and $L:\R\times\R\rightarrow\R_+$ represents a lower semi-continuous loss function that quantifies the discrepancy between the predicted output~$h(X)$ and the actual output~$Y$. Throughout the paper we assume that individuals have a protected attribute~$A\in\mc A=\{0,1\}$ that encodes their race, religion, age or sex etc. 
Note that the protected attribute~$A$ may impact the feature vector~$X$ or even be one of the features. 
If discrimination with respect to~$A$ is undesired or legally forbidden, however, one should seek classifiers and regressors~$h(X)$ that make no use of~$A$. Unfortunately, deleting~$A$ from the list of features (which is sometimes referred to as `fairness by unawareness' \citep{ref:barocas-hardt-narayanan}) is not enough to make $h$ fair because the entirety of data collected about a person typically provides enough information to infer~$A$ with high reliability. Instead of ignoring~$A$, we require the classifiers and regressors to satisfy formal statistical notions of fairness, which are defined in terms of conditional probability distributions. The following definition unifies several popular group fairness criteria.
\begin{definition}[Group-fairness]
\label{def:eps-fairness}
A hypothesis~$h\in\mc H$ is fair at the level $\varepsilon \geq 0
$ with respect to some scoring function $S: \R \times \mathcal Y \to \R$ and family of conditioning sets $\mathcal C = \{ \mc C_j\}_{j\in\mc J}$ with $\mc C_j \subseteq \mc X \times \mc Y$ for all $j \in \mc J \subseteq \mathbb N$ if
\[
\Big|\PP\Big[ S(h(X),Y) \leq \tau| A=0, (X,Y)\in \mc C_j\Big] - \PP\Big [ S(h(X),Y) \leq \tau|A = 1, (X,Y)\in   \mc C_j\Big ]\Big|\leq \eps \quad~\forall\tau \in \R,\ \forall j\in\mc J.\]
\end{definition}
Table~\ref{tab:eps-fairness-illustration} shows that different choices of the scoring function and the family of conditioning sets~$\mc C$ imply different group fairness criteria from the literature.
\begin{table}
    \centering
    \begin{tabular}{l@{\quad}|@{\quad}l@{\qquad}l}
Fairness criterion& $S(\hat y, y)$  & $\mathcal{C}$\\\hline
Statistical parity~\cite{ref:agarwal2019fair}               & $\hat y$      & $\{\mc X \times \mc Y\}$\\
Equal opportunity~\cite{ref:hardt2016equality,ref:pleiss2017fairness}                & $\hat y$      & $\{\mathcal{X} \times \{1\}\}$\\
Equalized odds~\cite{ref:hardt2016equality, ref:pleiss2017fairness}               
& $\hat y$      & $\{\mathcal{X} \times \{y\}\}_{y\in\mc Y}$\\
Risk parity~\cite{ref:donini2018empirical, ref:maity2021does} & $L(\hat y,y)$ & $\{\mc X \times \mc Y\}$\\
Conditional statistical parity~\cite{ref:corbett2017algorithmic} & $\hat y$& $\{
\mc X_s \times \mc Y\}$~\text{for some $\mathcal X_s \subseteq \mathcal X$}\\
Probabilistic predictive equality~\cite{ref:corbett2017algorithmic, ref:pleiss2017fairness}& $\hat y$      & $\{\mathcal{X} \times \{0\}\}$\\
\end{tabular}
    \caption{Group fairness criteria induced by different scoring functions and conditioning sets.}
    \label{tab:eps-fairness-illustration}
\end{table}
Note that a hypothesis~$h$ is fair at level~$\eps=0$ if and only if the conditional CDFs~$F_{S(h(X),Y)|A=0,(X,Y)\in \mc C_j}$ and~$F_{S(h(X),Y)|A=1,(X,Y)\in \mc C_j}$ match for all $j \in \mc J$, which means that~$S(h(X),Y)$ is independent of $A$ conditional on the event $(X,Y)\in\mc C_j$. Accordingly, fairness at level~$\eps>0$ implies similarity of the conditional CDFs~$F_{S(h(X),Y)|A=0,(X,Y)\in \mc C_j}$ and~$F_{S(h(X),Y)|A=1, (X,Y)\in \mc C_j}$ for all $j \in \mc J$.


All methods to be developed in this paper are compatible with all fairness notions of Table~\ref{tab:eps-fairness-illustration}. For ease of exposition, however, we will explain these methods without loss of generality under the assumption that fairness is quantified by \textit{statistical parity} (SP)~\cite{ref:agarwal2019fair}, which is the group fairness criterion obtained by setting $S(\hat y, y) = \hat y $ and $\mc C = \{\mc X \times \mc Y\}$. SP is also called demographic parity~\citep{ref:dwork2012fairness} or disparate impact~\citep{ref:feldman2015certifying}. By construction, a hypothesis~$h$ satisfies SP at level~$\eps=0$ if and only if the conditional CDFs~$F_{h(X)|A=0}$ and~$F_{h(X)|A=1}$ match, which means that~$h(X)$ is independent of $A$ {\color{black}(which we express concisely as $h(X)\perp A$)}. 
SP at level~$\eps>0$ implies similarity of the conditional CDFs~$F_{h(X)|A=0}$ and~$F_{h(X)|A=1}$.
From a conceptual point of view, SP is arguably the easiest of all fairness criteria listed in Table~\ref{tab:eps-fairness-illustration}.
SP is enforced, for example, via the US Equal Employment Opportunities Commission's 80\% rule~\cite{us1979_80percrule}, which requires that the selection rate for any race, sex, or ethnic group be at least 80\% of the rate for the group with the highest rate.

Fair hypotheses can be constructed by solving the statistical learning problem~\eqref{eq:loss-min} subject to the extra constraint that~$h$ must be $\eps$-fair with respect to SP (or any other desired group fairness notion). Clearly, the optimal value of problem~\eqref{eq:loss-min} decreases if we restrict its feasible set, that is, increasing the fairness of the optimal hypothesis comes at the expense of reducing its predictive power. This accuracy-fairness trade-off is well-documented empirically~\cite{ref:calders2009building, ref:corbett2017algorithmic, ref:menon2018cost}. If the in-sample distribution of~$(X, Y)$ used for training differs from the out-of-sample distribution used for testing, however, then imposing fairness constraints can have a regularizing effect and improve predictive power. In other words, imposing fairness constraints in sample can increase the fairness as well as the predictive power of the optimal hypothesis out of sample. For example, it can be shown that equal opportunity-fairness constraints can help to learn the Bayes optimal classifier (which maximizes the correct classification rate) even if the training samples are subject to label bias~\cite{blum2020recovering} or feature bias~\cite{ref:dutta2020there}. Similarly, disparate mistreatment-fairness constraints can help to learn the Bayes optimal classifier in the target domain even in the presence of a subpopulation shift~\cite{ref:maity2021does}.

 In the following, we will discuss the conceptual advantages and disadvantages of enforcing SP constraints in statistical learning. That is, we compare problem~\eqref{eq:loss-min} against a fair learning problem of the form
\begin{equation}
    \begin{array}{cl}
    \min\limits_{h \in \mc H} & \EE[L (h(X), Y)] \\
    \text{s.t.} & h(X)\perp A
    \end{array}
    \label{eq:loss-min-fair}
\end{equation}
which optimizes only over predictors that are SP-fair at level~$\varepsilon=0$. We first highlight that, if used na\"ively, SP constraints can have  detrimental effects on the optimal predictors. Specifically, the following example reveals that SP constraints can reduce both prediction accuracy as well as fairness by any reasonable standard.
\begin{example}[Enforcing SP can reduce accuracy and fairness]
\label{ex:price-of-fairness}
Consider a least squares regression problem of the form~\eqref{eq:loss-min} with~$L(\hat y, y)=(\hat y-y)^2$ that aims to predict the skill level~$Y\in [0,1]$ of a job candidate based on a feature vector~$X=(X_1,X_2)$, where $X_1\in[0,1]$ and~$X_2=A\in\{0,1\}$ represent the candidate's normalized college GPA and age group, respectively. Specifically, suppose that $X_2=1$ if the candidate's age is at most 40~years and that $X_2=0$ otherwise. The hypothesis space~$\mc H$ comprises all Borel-measurable functions~$h$ from $\mc X=[0,1]\times\{0,1\}$ to $\mc Y=[0,1]$. A correct and fair prediction of the skill level is critical because it will determine the candidate's salary. In the following, we set~$p_a=\PP[A=a]$ for all~$a\in\mc A$, and we assume that the skill level satisfies $Y=AX_1+(1-A)S$, where~$S$ denotes the candidate's work experience, which is not observed. Thus, the skill level matches the GPA for junior candidates and the work experience for senior candidates.  We also assume that both~$X_1$ and~$S$ are uniformly distributed on~$[0,1]$ and that~$X_1$, $S$ and~$A$ are mutually independent. Hence, the distribution of~$Y$ conditional on~$A=a$ coincides with the uniform distribution on~$[0,1]$ irrespective of~$a\in\mc A$, that is, $Y$ is independent of $A$. Moreover, one can show that the optimal value of~\eqref{eq:loss-min} amounts to~$p_0/12$, which is uniquely attained by the hypothesis
\[
    h^\star(x)=\EE[Y|X=x]=\left\{ \begin{array}{cl}
    \half & \text{if }x_2=a=0, \\
    x_1=y & \text{if }x_2=a=1.
    \end{array}\right.
\]
This confirms that, even though the protected attribute~$A$ is independent of~$Y$, the optimal hypothesis may display a non-trivial dependence on~$A$. As the conditional distributions of $h^\star(X)$ differ across the two age groups, SP is violated. Next, we solve the regression problem~\eqref{eq:loss-min-fair}, which restricts attention to SP-fair hypotheses. A tedious but routine calculation reveals that the optimal SP-fair hypothesis~is
\[
    \textstyle h_{\rm SP}^\star(x)=\EE[Y|X_1=x_1]=\half+p_1(x_1-\half),
\]
whose objective function value exceeds that of $h^\star$ because
\[
    \textstyle \EE[L (h^\star_{\rm SP}(X), Y)] = \frac{p_0}{12}\cdot\frac{1}{p_1}\left[ 1+10p_1-9p_1^2+4p_1^3\right]+\frac{p_0^2p_1}{12} >\frac{p_0}{12} = \EE[L (h^\star(X), Y)].
\]
Note that under~$h^\star_{\rm SP}$ the predicted skill level grows with the GPA independent of the protected attribute. Thus, the distribution of $h^\star_{\rm SP}(X)$ conditional on~$A=a$ coincides with the uniform distribution on the interval $[\frac{1}{2}(1-p_1), \frac{1}{2}(1+p_1)]$ irrespective of~$a\in\mc A$, indicating that the representatives of the two age groups have the same prospects of being hired into a particular salary band. This confirms that~$h^\star_{\rm SP}(X)$ is indeed SP-fair. However, the prediction~$h^\star_{\rm SP}(X)$ is not more fair than~$h^\star(X)$ by any reasonable standard. While $h^\star_{\rm SP}(X)$ is perfectly correlated with the skill level of junior candidates, it is independent of the skill level of senior applicants. Hence, $h^\star_{\rm SP}$ enforces SP by making purely random predictions that have no bearing on the actual qualifications of senior applicants, which suggests that their earning a high salary is tantamount to winning a lottery. Under the hypothesis~$h^\star$ that minimizes the prediction loss without constraints, on the other hand, all senior candidates are treated equally, which seems more `fair' than affording them a random salary.
\qed
\end{example}

{\color{black} Example~\ref{ex:price-of-fairness} shows that SP deteriorates the predictive power of the optimal hypothesis and, while ensuring fairness at the level of the population, may reduce fairness at the level of the individual. This raises the question whether there is any benefit in enforcing SP at all. We will answer this question affirmatively under three simplifying assumptions. First, we assume that the sample space~$\Omega$ is finite and that~$\mc F=2^{|\Omega|}$ is the power set $\sigma$-algebra. This assumption is unrestrictive because problem~\eqref{eq:loss-min} aims to predict properties of human beings and because the world's population is finite. As~$\Omega$ is finite, $X$ and~$Y$ can adopt only finitely many values~$x_\omega$ and~$y_\omega$ for $\omega \in\Omega$, respectively. We may thus assume without loss of generality that~$\mc X = \{x_\omega:\omega\in\Omega\}$ is finite, too. Next, we assume that $\mc H=\mc L(\mc X,\mathbb{R})$ represents the family of {\em all} measurable hypotheses. As~$\mc X$ is finite, this assumption holds if one optimizes over families of deep neural networks that enjoy universal approximation capabilities. By slight abuse of notation, we can thus identify any hypothesis~$h\in \mc H$ with a finite-dimensional vector of the form~$h=(h(x_\omega))_{\omega\in\Omega}$, and we can identify~$\mc H$ with~$\R^{|\Omega|}$. Finally, we assume that the distribution of~$(X,Y)$ is known. This is approximately true if we have access to a large training~dataset.

The following proposition shows that if the sensitive attribute~$A$ carries no information about the distribution of~$Y$ conditional on~$X$, then problem~\eqref{eq:loss-min} has an optimal solution that is fair with respect to~SP.

\begin{proposition}[Optimality implies SP]
\label{thm:optimal-decision-indep}
    Suppose that~$\Omega$ is finite, $\mc H=\mc L(\mc X,\mathbb{R})$, and~\eqref{eq:loss-min} has a minimizer~$h^\star$ that is $\PP$-almost surely unique. If $\PP_{Y|X}\perp A$, then $h^\star(X)\perp A$.
\end{proposition}

Note that the conditional distribution~$\PP_{Y|X}$ is fully determined by the finite random vector consisting of the conditional probabilities $\PP_{Y|X}[Y=y_\omega]$ for $\omega\in\Omega$. Proposition~\ref{thm:optimal-decision-indep} shows that the optimal solution of problem~\eqref{eq:loss-min} is sometimes guaranteed to be SP-fair even though SP is not enforced in~\eqref{eq:loss-min}. In these cases, SP-fairness provides a necessary optimality condition for problem~\eqref{eq:loss-min}. This result is {\em un}surprising. Indeed, in order to predict~$Y$ from~$X$, all one needs to know is the distribution of~$Y$ conditional on~$X$. Thus, the optimal prediction~$h^\star(X)$ depends on~$X$ only indirectly through the conditional distribution~$\PP_{Y|X}$. If this conditional distribution is independent of~$A$, then, clearly, $h^\star(X)$ must be independent of~$A$, too. Hence, the condition $\PP_{Y|X}\perp A$ translates the statement ``{\em $A$ is irrelevant for predicting~$Y$}'' into the language of mathematics. We expect it to hold in many cases of algorithmic bias that were scandalized in the news. 
}

The following example shows that Proposition~\ref{thm:optimal-decision-indep} ceases to hold if the assumption~$\PP_{Y|X}\perp A$ is replaced with the simpler assumption~$Y\perp A$. 
\begin{example}[$Y\perp A$ does not imply SP]
\label{ex:YperpA-counter}
Consider a classification problem of the form~\eqref{eq:loss-min} with feature space $\mc X=\{0,1\}^2$, label space $\mc Y=\{0,1\}$, 0-1 loss function $L(\hat{y}, y)=\mathbbm 1_{\hat y\neq y}$ and hypothesis space $\mc H=\mc L(\mc X,\mc Y)$, where $X_2=A$ is a protected attribute. Assume that $X_1$ is independent of~$A$ and that its marginal distribution is given by~$\PP[X_1=0]=0.8$ and~$\PP[X_1=1]=0.2$. The distribution of~$A$ is irrelevant for this example. Finally, the conditional distribution of~$Y$ given~$X$ is completely determined by the success probabilities in Table~\ref{tbl:yperpa-p}. Based on this information, it is easy to verify that
\[
    \PP[Y=1|A=a]= \sum_{x_1\in\{0,1\}} \PP[Y=1|X_1=x_1,\,A=a]\, \PP[X_1=x_1] = 0.5\quad\forall a\in\mc A,
\]
and thus we have $Y\perp A$. Similarly, we find
\[
    \PP[\,\PP[Y=1|X]=0.4\,|\,A=a] = \left\{\begin{array}{ll} 0.2 & \text{if }a=1,\\
    0 & \text{if }a=0,
    \end{array}\right.
\]
which is sufficient to imply that $\mathbb P_{Y|X} \not\perp A$. Thus, Proposition~\ref{thm:optimal-decision-indep} does not apply. Thanks to its simplicity, problem~\eqref{eq:loss-min} can be solved analytically, and its unique optimal solution $h^\star$ is fully characterized by the information in Table~\ref{tbl:yperpa-h}. Next, one readily verifies that
$$
\PP[h^\star(X)=1|A=a]=\PP[X_1=a]=\left\{\begin{array}{ll} 0.2 & \text{if }a=1,\\
    0.8 & \text{if }a=0,
    \end{array}\right.
$$
which reveals that $h^\star(X)\not\perp A$. Hence, $h^\star$ fails to satisfy SP even though~$Y$ is independent of~$A$.
\begin{table}%
  \centering
  \subfloat[][${\mathbb{P}[Y=1|X]}$]{
  \begin{tabular}{lc c}
  \toprule
    &$X_1=0$&$X_1=1$\\\midrule
         $X_2=1$&  0.4\phantom{0}&0.9\phantom{0}\\
         $X_2=0$& 0.55&0.3\phantom{0} \\
         \bottomrule
\end{tabular}\label{tbl:yperpa-p}
  }%
  \qquad
  \subfloat[][$h^\star(X)$]{
  \begin{tabular}{lc c}
  \toprule
    &$X_1=0$&$X_1=1$\\\midrule
         $X_2=1$&  0&1\\
         $X_2=0$& 1&0 \\
         \bottomrule
\end{tabular}\label{tbl:yperpa-h}
  }
  \caption{Success probabilities and optimal decisions for Example~\ref{ex:YperpA-counter}}
  \label{tbl:yperpa}%
\end{table}
\qed
\end{example}

While $Y\perp A$ does {\em not} induce SP, we know from Proposition~\ref{thm:optimal-decision-indep} that the condition $\mathbb P_{Y|X}\perp A$ ensures the existence of an SP-fair optimizer. From Example~\ref{ex:YperpA-counter} it is thus clear that $Y\perp A$ does {\em not} imply $\mathbb P_{Y|X}\perp A$. Conversely, $\mathbb P_{Y|X}\perp A$ does also {\em not} imply $Y\perp A$ in general. To see this, just assume that $Y=A$ and that $X$ is independent of $Y$. In this case $\mathbb P_{Y|X}\perp A$ is satisfied, but $Y\perp A$ is not.

The next proposition shows that the family $\mc H_{\rm fair}=\{h\in\mc H: h(X)\perp A\} $ of all SP-fair hypotheses within~$\mc H$ consists of a finite union of linear subspaces of~$\mc H\cong\R^{|\Omega|}$.

\begin{proposition}[Geometry of $\mc H_{\rm fair}$]
    \label{prop:H_fair}
    If~$\Omega$ is finite and $\mc H=\mc L(\mc X,\R)$, then the set $\mc H_{\rm fair}=\{h\in\mc H: h(X)\perp A\} $ of all SP-fair hypotheses is non-empty and represents a union of finitely many linear subspaces of~$\mc H$.
\end{proposition}

Proposition~\ref{prop:H_fair} implies that~$\mc H_{\rm fair}$ is generically non-convex. Hence, the fair learning problem~\eqref{eq:loss-min-fair} constitutes a---potentially challenging---non-convex optimization problem even if the loss function is convex. Next, we will show that enforcing SP can improve prediction accuracy if the training data is biased. To this end, we consider a family of outputs $Y_\delta$ parametrized by~$\delta\geq 0$. Here, $Y_0$ represents the true {\em unbiased} output we actually want to predict. For any~$\delta>0$, however, $Y_\delta$ represents a {\em biased} proxy for~$Y_0$. We will study the task of predicting~$Y_0$ from~$X$ if we have only access to the distribution of~$(X,Y_\delta)$ for some~$\delta>0$ but {\em not} that of~$(X,Y_0)$.

\begin{example}[Training on biased data]
\label{ex:healthcare-bias}
A recent study uncovered racial bias in an algorithm utilized by the U.S.\ health care system for making health-related decisions \cite{ref:obermeyer2019dissecting}. The study revealed that black patients classified at the same risk level as white patients were, in fact, in poorer health. This bias arose because the algorithm used health care costs as a proxy for health needs and because less money is typically allocated to black patients with the same level of need. The authors conclude that ``{\em the choice of convenient, seemingly effective proxies for ground truth can be an important source of algorithmic bias in many contexts.}'' \qed
\end{example}

The following theorem shows that if~$A$ is irrelevant for predicting the true target~$Y_0$ but training is based on a biased target~$Y_\delta$, then enforcing SP improves the predictive power of the optimal hypothesis in view of~$Y_0$.

\begin{theorem}[SP improves accuracy]
\label{thm:hard-fairness-accuracy}
    Suppose that~$\Omega$ is finite, $\mc H=\mc L(\mc X,\R)$, $L(\hat y,y)$ is strongly convex in~$\hat y$, and $\varphi(h,\delta)=\EE[L(h(X), Y_\delta)]$ is once and twice continuously differentiable in~$\delta$ and~$h$, respectively. For all $\delta\geq 0$, let $h^\star_\delta$ and $h^\star_{{\rm fair},\delta}$ be minimizers of $\min_{h\in\mc H}\varphi(h,\delta)$ and $\min_{h\in\mc H_{\rm fair}}\varphi(h,\delta)$, respectively.  If $\PP_{Y_0|X}\perp A$ and~$\delta$ is sufficiently small, then we have $\varphi(h\opt_\delta,0) \geq \varphi(h^\star_{\textrm{fair},\delta},0)$, and the inequality is strict unless~$h\opt_\delta\in\mc H_{\rm fair}$.
\end{theorem}

\begin{figure}
    \centering
    \begin{subfigure}[t]{0.4\columnwidth}
    \includegraphics[width=\linewidth]{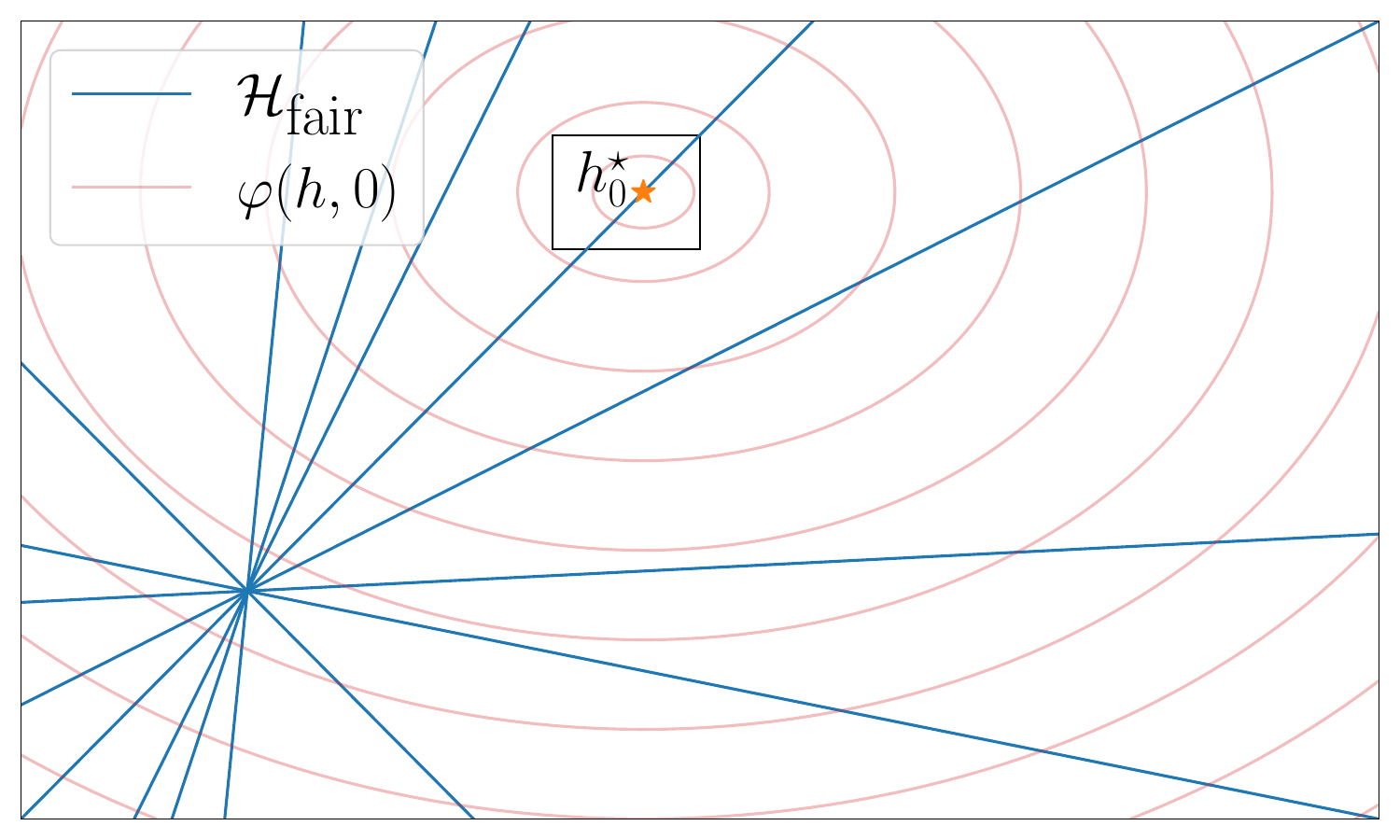}
    \caption{Global visualization of $\mc H_{\rm fair}$}
    \label{fig:sp-illustration-hfair}
    \end{subfigure}\hspace{5mm}
    \begin{subfigure}[t]{0.4\columnwidth}
    \includegraphics[width=\linewidth]{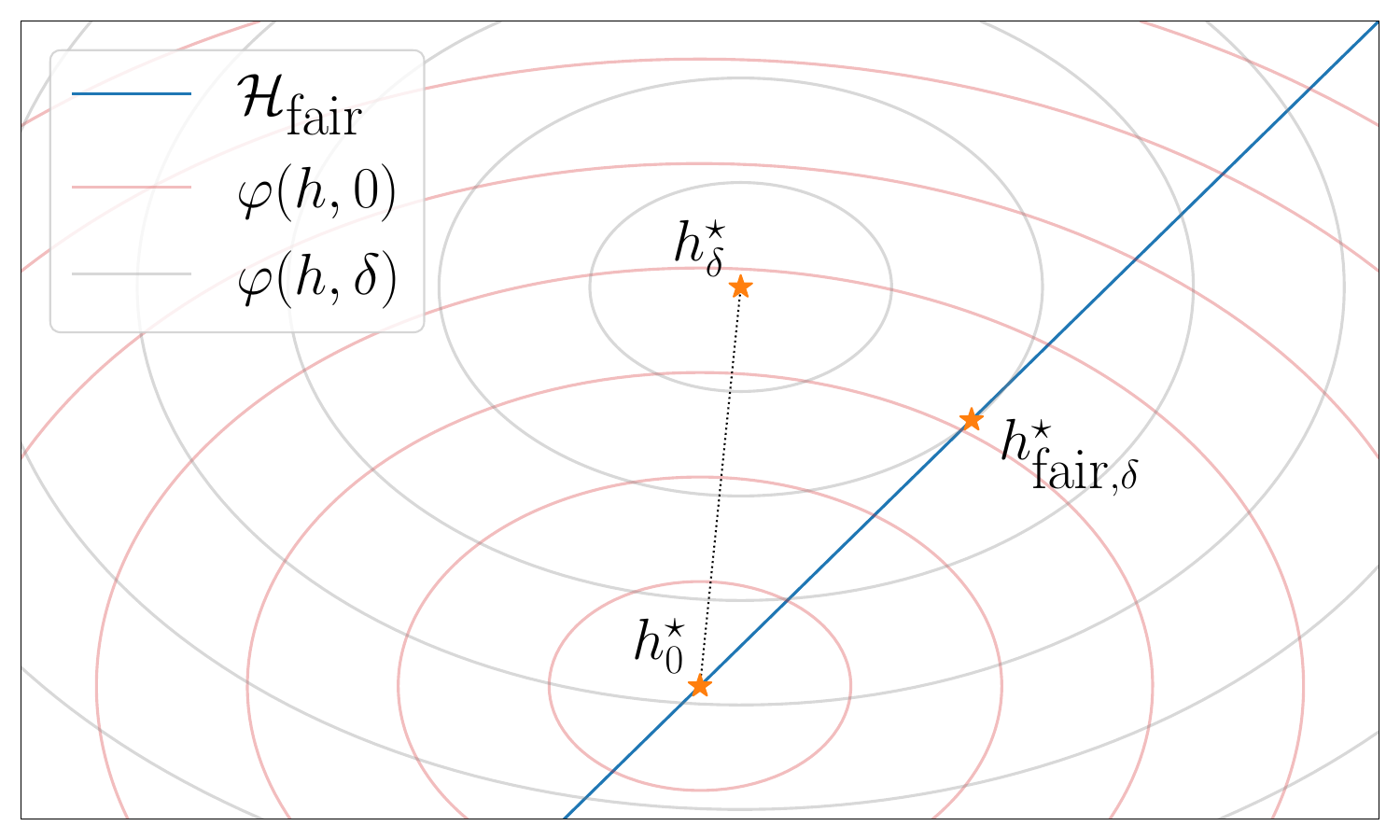}
    \caption{Local visualization of $\mc H_{\rm fair}$}
    \label{fig:sp-illustration-fair}
    \end{subfigure}
    \caption{ Proof of Theorem~\ref{thm:hard-fairness-accuracy}: The right chart zooms into the neighborhood of $h\opt_0$. The red (gray) lines represent the contours of the unbiased (biased) objective functions.}
    \label{fig:sp-illustration}
\end{figure}

Theorem~\ref{thm:hard-fairness-accuracy} implies that if biased outputs~$Y_\delta$ for~$\delta>0$ are used for training, then SP constraints reduce the prediction loss of the optimal hypothesis if two conditions are satisfied: {\em (i)}~the sensitive attribute~$A$ is irrelevant for predicting the true target~$Y_0$ ({\em i.e.}, $\PP_{Y_0|X}\perp A$) and {\em (ii)}~the bias is small ({\em i.e.}, $\delta$ is close to~$0$). Under the conditions of Theorem~\ref{thm:hard-fairness-accuracy}, we thus face a win-win situation, that is, SP improves both the fairness as well as the predictive power of the optimal hypothesis. Empirical evidence supporting this theoretical insight is provided in \cite{wick2019fairnesstradeoff}. The proof of Theorem~\ref{thm:hard-fairness-accuracy} is illustrated in Figure~\ref{fig:sp-illustration}. Figure~\ref{fig:sp-illustration-hfair} represents~$\mc H$ as a plane and visualizes the linear subspaces that form~$\mc H_{\rm fair}$ by blue lines. Note that~$\varphi(h,0)$ represents the expected loss of a hypothesis~$h$ with respect to the true unbiased output~$Y_0$. Its contours are visualized by red lines. By Proposition~\ref{thm:optimal-decision-indep}, which applies because~$A$ is irrelevant for predicting~$Y_0$, the optimal predictor~$h_0^\star$ that minimizes~$\varphi(h,0)$ over~$\mc H$ is SP-fair and thus resides on one of the blue lines. Figure~\ref{fig:sp-illustration-fair} zooms into the neighborhood of~$h^\star_0$. If we do not have data for the true target~$Y_0$ but only for the biased target~$Y_\delta$, then we cannot minimize~$\varphi(h,0)$. Instead, we have to minimize the perturbed objective function~$\varphi(h,\delta)$, which is constructed from the biased target~$Y_\delta$. In Figure~\ref{fig:sp-illustration-fair} the contours of the perturbed objective function are visualized by the gray lines. If the bias parameter~$\delta$ is small, then the hypothesis~$h^\star_\delta$ that minimizes~$\varphi(h,\delta)$ over~$\mc H$ is close to~$h^\star_0$. Clearly, $h^\star_\delta$ has a higher prediction error than~$h^\star_0$ with respect to the true unbiased objective function~$\varphi(h,0)$ with the red contours. In addition, $h^\star_\delta$ generically violates SP, that is, it does {\em no longer} reside on the blue line. The optimal {\em fair} hypothesis~$h^\star_{{\rm fair},\delta}$ that minimizes the biased prediction loss~$\varphi(h,\delta)$ over~$\mc H_{\rm fair}$ is obtained by projecting~$h^\star_\delta$ onto~$\mc H_{\rm fair}$. Thus, the prediction error of~$h^\star_{{\rm fair},\delta}$ with respect the true unbiased objective function~$\varphi(h,0)$ is smaller than that of~$h^\star_\delta$. This shows that if~$A$ is irrelevant for predicting~$Y_0$ and the bias parameter~$\delta$ is small, then the prediction error of the solution of the biased learning problem improves with respect to the true unbiased objective function if we enforce SP.

Note that if~$\delta$ is large, then~$h^\star_\delta$ may no longer be close to the linear subspace of~$\mc H_{\rm fair}$ that contains~$h^\star_0$. In this case, $h^\star_\delta$ and~$h^\star_{{\rm fair},\delta}$ may in fact reside on {\em different} linear subspace of~$\mc H_{\rm fair}$, and the prediction error of~$h^\star_{{\rm fair},\delta}$ with respect to the true unbiased objective function may thus be {\em higher} than that of~$h^\star_0$. This shows that SP constraints improve prediction accuracy only if the bias of the training outputs is sufficiently small.

\begin{remark}[Loss functions]
The requirement that~$\varphi(h,\delta)$ be strongly convex in~$h$ is restrictive, but it holds in the important special case when~$L(\hat y,y)=(\hat y-y)^2$. It could be relaxed at the expense of a more cumbersome proof. Alternatively, we could restrict~$\mc H$ to a reproducible kernel Hilbert space~$\mbb H_K\subseteq \mc H$ and allow for objective functions of the form $\varphi(h,\delta)=\EE[L(h(X),Y_\delta)] +\|h\|^2_\mbb H$ with a convex (but not necessarily strongly convex) loss function~$L$ and a Tikhonov regularizer induced by the Hilbert norm~$\|\cdot\|_{\mbb H_K}$. \qed
\end{remark}

\begin{remark}[Bias models]
\label{rem:bias-models}
Theorem~\ref{thm:hard-fairness-accuracy} remains valid for a broad spectrum of bias models. For example, in least-squares regression it holds for biased outputs of the form $Y_\delta=\delta Y_1+(1-\delta)Y_0$, and in classification with cross-entropy loss it holds for biased outputs of the form $Y_\delta=B_\delta Y_1+(1-B_\delta)Y_0$, where $B_\delta$ is a Bernoulli random variable with success probability~$\delta$ that is independent of all other sources of randomness.
\qed
\end{remark}

A critical challenge in achieving algorithmic fairness is the choice of the sensitive attribute~$A$. 
In principle, {\em any} component~$X_i$ of the feature vector~$X$ and any threshold~$\tau\in\R$ can be used to construct two demographic groups characterized by the binary attribute~$A=\mathbbm 1_{X_i\leq \tau}$. From a philosophical point of view, it is desirable to ensure that {\em every} subpopulation is treated fairly. It is easy to see, however, that only a trivial {\em constant} hypothesis can be SP-fair with respect to {\em all} possible subpopulations. We thus have to settle for a more modest goal. Theorem~\ref{thm:hard-fairness-accuracy} provides guidance for selecting {\em meaningful} sensitive attributes. Specifically, $A$ constitutes a `good' sensitive attribute if it is irrelevant for predicting the true target~$Y_0$ ($\PP_{Y_0|X}\perp A$) {\em and} if there is reason to believe that is relevant for predicting the biased target ($\PP_{Y_\delta|X}\not\perp A$). Indeed, the first condition ensures that $h^\star_0\in\mc H_{\rm fair}$, and the second condition implies that, generically, $h^\star_\delta\notin\mc H_{\rm fair}$. Enforcing SP with respect to all sensitive attributes satisfying these two conditions will improve prediction accuracy.

\section{Unfairness Measures and Integral Probability Metrics}
\label{sec:ipms}
Recall that a hypothesis $h\in\mc H$ satisfies SP at level $\varepsilon$ if and only if the absolute difference between the CDFs of $h(X)$ conditional on $A=0$ and $A=1$ is uniformly bounded by $\varepsilon$. Equivalently, the Kolmogorov distance~\citep{ref:shorack2000probability} between $\mathbb{P}_{h(X)|A=0}$ and $\mathbb{P}_{h(X)|A=1}$ is at most~$\varepsilon$ \cite{chzhen2020wsbary_reg}. Hence, the Kolmogorov distance between $\mathbb{P}_{h(X)|A=0}$ and $\mathbb{P}_{h(X)|A=1}$ quantifies the degree of {\em un}fairness of~$h$. The Kolmogorov distance is an example of an integral probability metric (IPM) \citep{ref:muller_1997, ref:sriperumbudur2012empirical}.

\begin{definition}[Integral probability metric]
\label{def:ipm}
Let~$w\in\mc L(\R^n,[1,\infty))$ be a weight function, define~$\mc L_w(\R^n,\R)$ as the set of all functions~$\psi \in\mc L(\R^n,\R)$ with~$\sup_{z\in\R^n} |\psi(z)|/w(z)<\infty$, and define~$\mc Q_w(\R^n)$ as the set of all probability measures~$\QQ$ on~$\R^n$ with~$\int_{\R^n}w(z)\QQ(\dd z)<\infty$. The integral probability metric on~$\mc Q_w(\R^n)$ with generator~$\Psi\subseteq \mc L_w(\R^n,\R)$ is then given by
{\em
\[
    \mathcal{D}_{\Psi}(\QQ_1, \QQ_2) = \sup\limits_{\psi\in \Psi} \left|  \int_{\R^n} \psi(z)\, \QQ_1(\dd z) - \int_{\R^n} \psi(z) \,\QQ_2( \dd z)\right|\quad \forall \QQ_1, \QQ_2\in \mc Q_w(\R^n). 
\]
}
\end{definition}
Note that $\mathcal{D}_{\Psi}$ is a pseudo-metric for any~$\Psi\subseteq \mc L_w(\R^n,\R)$. Indeed, $\mathcal{D}_{\Psi}(\QQ_1, \QQ_2)$ is non-negative, symmetric in~$\QQ_1$ and~$\QQ_2$, vanishes if~$\QQ_1=\QQ_2$ and satisfies the triangle inequality. Moreover, $\mathcal{D}_{\Psi}$ is a metric if~$\Psi$ separates points in~$\mc Q_w(\R^n)$, in which case~$\mathcal{D}_{\Psi}(\QQ_1, \QQ_2)$ vanishes only if~$\QQ_1=\QQ_2$. The Kolmogorov distance is indeed an IPM as it is generated by the family of step-functions of the form $\psi(y)=\mathbbm 1_{y\leq\tau}$ parametrized by $\tau\in\R$. We emphasize, however, that other IPMs have previously been used to quantify unfairness~\cite{ref:dwork2012fairness}. A prominent example is the 1-Wasserstein distance.
\begin{definition}[Wasserstein distance]
\label{def:wass-distance}
If $w(z)=1+\|z\|$ for some norm~$\|\cdot\|$ on~$\R^n$, then the 1-Wasserstein distance (or Kantorovich distance) between~$\QQ_1, \QQ_2\in\mc Q_w(\R^n)$ is
{\em
\begin{equation*}
    \mathcal{W}(\QQ_1, \QQ_2) = \inf\limits_{\pi \in \Pi(\QQ_1, \QQ_2)} \int_{\R^n \times \R^n} \|z- z' \|\,  \pi(\dd z, \dd z'),
\end{equation*}}where~$\Pi(\QQ_1, \QQ_2)$ denotes the set of all joint distributions or couplings of the random vectors~$Z\in\R^n$ and~$Z'\in\R^n$ with marginal distributions~$\QQ_1$ and~$\QQ_2$, respectively.
\end{definition}
By the classical Kantorovich-Rubinstein theorem, the 1-Wasserstein distance is indeed an IPM.
\begin{lemma}[Kantorovich-Rubinstein theorem {\cite[Remark~6.5]{ref:villani2008optimal}}] 
The 1-Wasserstein distance~$\mathcal{W}$ coincides with the IPM~$\mathcal{D}_\Psi$ generated by the set $\Psi = \{\psi \in \mc L(\R^n, \R) : \lip(\psi) \leq 1\}$ of all Lipschitz continuous test functions with Lipschitz modulus of at most~$1$.
\label{leamma:kantorovich-ipm}
\end{lemma}

For univariate distributions, the Wasserstein distance reduces to the $\mathcal{L}^1$-distance~\cite{ref:salvemini1943sul}.
\begin{definition}[$\mc L^p$-distance]
\label{def:lp_distance}
If $n=1$ and $w(z)=1+|z|$, then the $\mc L^p$-distance of~$\QQ_1,\QQ_2\in \mc Q_w(\R)$ is given by the $\mc L^p$-norm distance of their CDFs~$F_{\QQ_1}$ and~$F_{\QQ_2}$, respectively, that is,
\[
    d_{p}(\mbb Q_1,\mbb Q_2) =\|F_{\QQ_1}-F_{\QQ_2}\|_{\mc L^p}= \left(  \int_{\R} \left|F_{\QQ_1}(z) - F_{\QQ_2}(z)\right|^p \textrm{\em d} z\right)^{1/p}
\]
for~$p \in[1, \infty)$ and~$d_\infty(\mbb Q_1, \mbb Q_2) =\|F_{\QQ_1}-F_{\QQ_2}\|_{\mc L^\infty} =\sup_{z\in\R}|F_{\QQ_1}(z) - F_{\QQ_2}(z)|$ for~$p=\infty$.
\end{definition}

It can be shown that any $\mathcal{L}^p$-distance on a space of univariate distributions is in fact an IPM.

\begin{lemma}[Duals of~$\mc L^p$-distances~{\cite[{Lemma~1}]{dedecker2007lpdual}}]
\label{lemma:lpdual}
For any pair of conjugate exponents $p,q\in[1,\infty)$ with $1/p+1/q=1$, the $\mc L^p$-distance $d_p$ on~$\mc Q_w(\R^1)$ coincides with the IPM~$\mc D_\Psi$ generated by the set~$\Psi=\{\psi\in W_0^{1,q}(\R):\|\psi'\|_{\mc L^q}\leq 1\}$, where $W_0^{1,q}(\R)$ denotes the Sobolev space of all absolutely continuous functions~$\psi:\R\to\R$ whose derivative~$\psi'$ has a finite $\mc L^q$-norm.
\end{lemma}

For $p=\infty$, the $\mc L^p$-distance collapses to the Kolmogorov distance. The squared $\mc L^2$-distance is sometimes called the Cram{\'e}r distance~\citep{ref:cramer1928composition}, which can violate the triangle inequality and is thus only a semi-metric. The two-fold multiple of the squared $\mc L^2$-distance is also known as the energy distance~\citep{ref:baringhaus2004new, ref:szekely2004testing}. Hence, the square root of the (univariate) energy distance is an IPM. The square root of the (multivariate) energy distance is also an instance of a maximum mean discrepancy (MMD).

\begin{definition}[Maximum mean discrepancy]\label{def:mmd}
If $K \in\mathcal L(\R^n \times \R^n,\R)$ is a positive definite symmetric kernel and $w\in\mathcal L(\R^n ,[1,\infty))$ satisfies $\sup_{z\in\R^n} K(z,z')/w(z)<\infty$ for all~$z'\in\R^n$, then the maximum mean discrepancy between $\QQ_1, \QQ_2 \in \mc Q_w(\R^n)$ relative to~$K$ is given by
\begin{align*}
    {\rm d_{\rm MMD}}(\QQ_1, \QQ_2) &=\bigg(\int_{\R^n \times \R^n} K(z,z')\, \QQ_1(\dd z) \QQ_1(\dd z') + \int_{\R^n \times \R^n}  K(z, z') \, \QQ_2(\dd z)\QQ_2(\dd z') \\
    &\hspace{1cm}- 2\int_{\R^n\times \R^n} K(z, z')\,  \QQ_1(\dd z) \QQ_2(\dd z') \bigg)^\half.
\end{align*}
\end{definition}

\begin{lemma}[\cite{ref:sriperumbudur2012empirical}]
The MMD distance~$ d_{\textrm{\rm MMD}}$ induced by~$K$ matches the IPM~$\mc D_{\Psi}$ induced by the unit ball~$\Psi = \{\psi \in \mathbb H_K: \| \psi\|_{\mathbb H_K} \leq 1 \}$ in the reproducing kernel Hilbert space $\mbb H_K$ corresponding to~$K$.
\end{lemma}

If $K$ is any kernel satisfying $\|z -z'\| = K(z,z) + K(z', z') - 2 K(z, z')$, {\em e.g.}, if $K$ is a distance-induced kernel (see Appendix~\ref{app:dist-kernel}), then $2\cdot d^2_\textrm{MMD}$ reduces to the energy distance~\cite[Theorem~22]{ref:sejdinovic2013equivalence}. Another popular IPM is the total variation distance.
\begin{definition}[Total variation distance] The total variation distance of~$\QQ_1,\QQ_2\in\mc Q_1(\R^n)$ is $\mathcal{TV}(\QQ_1,\QQ_2) = \sup_{B \in \mc B(\R^n)} |\QQ_1(B) - \QQ_2(B)|$, where~$\mc B(\R^n)$ is the Borel $\sigma$-algebra on~$\R^n$.
\end{definition}
By construction, the total variation distance is an IPM generated by the indicator functions of the Borel sets~$B\in\mc B(\R^n)$. One can thus show that its maximal generator \cite[Definition~3.1]{ref:muller_1997} is given by~$\Psi=\{\psi\in\mc L(\R^n,\R):\|\psi\|_{\mc L^\infty}\leq 1\}$. A summary of the discussed IPMs is given in Table~\ref{tab:ipm_examples}.

\begin{table}[h]
\centering
\caption{Generators of commonly used IPMs}
\vspace{-5pt}
    \begin{tabular}[t]{ll}
    \toprule
        IPM $\mc D_{\Psi}$&Generator~$\Psi$\\ \midrule
        Kolmogorov distance&$\{\psi \in \mc L(\R, \R) : \exists \tau \in \R \text{ s.t. } \psi(y)=\mathbbm 1_{y\leq\tau}~\forall y\in\R\}$\\
        Wasserstein distance&$\{\psi\in \mc L(\R^n, \R): \lip(\psi) \leq 1\}$   \\
        $\mc L^p$-distance & $\{\psi\in W_0^{1,q}(\R): \|\psi'\|_{\mc L^q} \leq 1\}$, where $1/p+1/q = 1$\\
        $ \sqrt{\text{Cram{\'e}r distance}}$ & $\{\psi \in W_0^{1,2}(\R): \|\psi'\|_{\mc L^2}\leq 1\}$\\
        $\sqrt{ \text{Energy distance}}$&$\{\psi \in W_0^{1,2}(\R): \|\psi'\|_{\mc L^2} \leq \sqrt{2}\}$ \\
        $\text{Maximum mean discrepancy}$&$\{\psi \in \mathbb H_K: \|\psi\|_{\mathbb H_K} \leq 1\}$ \\
        Total variation distance&$\{\psi \in \mc L(\R^n, \R) : \|\psi\|_{\mc L^\infty} \leq 1\}$ \\
        \bottomrule
    \end{tabular}
    \label{tab:ipm_examples}
\end{table}
Different IPMs induce different unfairness measures. Observe first that each IPM~$\mathcal D_\Psi$ of Table~\ref{tab:ipm_examples} satisfies the identity of indiscernibles. This means that $\mathcal{D}_{\Psi}(\PP_{h(X)|A=0}, \PP_{h(X)|A=1})= 0$ if and only if~$\PP_{h(X)|A=0}$ coincides with~$\PP_{h(X)|A=1}$ or, put differently, if and only if $h$ satisfies SP at level~$0$. Hence, all IPMs agree on what it means for a hypothesis to be perfectly fair. However, they attribute different levels of unfairness to hypotheses with $\PP_{h(X)|A=0}\neq \PP_{h(X)|A=1}$. Fix now an IPM $\mc D_\Psi$ defined on $\mc Q_w$, and assume that $\mathbb P_{h(X)}\in \mc Q_w$ for all~$h\in\mc H$. In this case, the test functions $\psi\in\Psi$ can be viewed as utility functions, which are routinely used to model preferences under uncertainty~\cite{ref:von2007theory}. Specifically, if the prediction $h(X)$ impacts the well-being of a person with feature vector $X$, {\em e.g.}, if $h(X)$ determines the person's salary, then the expected utility $\mathbb E[\psi(h(X))|A=a]$ quantifies the person's expected satisfaction with the prediction~$h(X)$ if that person belongs to class~$a\in\mc A$. We can now use a hypothetical experiment to introduce a notion of {\em utilitarian fairness}. Imagine that you are asked to assess a hypothesis~$h$ before birth, that is, before knowing any of your own personal traits such as your feature vector~$X$, class~$A$ or utility function~$\psi$. In this situation, it is natural to call~$h$ $\varepsilon$-fair if the expected utilities conditional on~$A=0$ and~$A=1$ differ at most by~$\varepsilon$ for any~$\psi\in\Psi$, {\em i.e.}, if
\begin{equation}
    |\EE[\psi(h(X)) | A =0 ] -\EE[\psi(h(X))| A=1] | \leq\eps\quad \forall \psi \in \Psi.
    \label{eq:local-soc-fair-exp}
\end{equation}
This utilitarian perspective gives a physical interpretation to unfairness measures induced by IPMs.

\begin{lemma}[Utilitarian fairness]\label{lemma:soc_fair}
For any given hypothesis $h\in\mathcal{H}$ and any given IPM~$\mc D_\Psi$ on~$\mc Q_w$ such that $\mathbb P_{h(X)}\in\mc Q_w$, we have that $\mathcal{D}_{\Psi}(\PP_{h(X)|A=0}, \PP_{h(X)|A=1})\leq \varepsilon$ if and only if~\eqref{eq:local-soc-fair-exp} holds.
\end{lemma}


In summary, any IPM $\mc D_\Psi$ can be used to measure the deviation from perfect statistical parity, and its generator~$\Psi$ can be viewed as a family of utility functions. The degree of unfairness that $\mc D_\Psi$ assigns to any given hypothesis thus always has a utilitarian interpretation. 
{In addition, one can show that measuring unfairness via IMPs readily yields high-probability generalization bounds (see Theorem~\ref{thm:generalization} in Appendix~\ref{app:aux}).}
We conclude that, conceptually, there is no objective reason for preferring the Kolmogorov metric (which is used in the standard definition of SP in~\cite{ref:agarwal2019fair}, for example) over other IPMs. In the remainder of the paper we will argue, however, that other IPMs have distinct computational advantages over the Kolmogorov metric.


\section{Numerical Solution of Fair Learning Problems}
\label{sec:num_fair_learning}
From now on we assume that~$L$ is convex and that~$\mc H=\{h_\theta:\theta\in\Theta\}$, where~$\Theta$ is an open subset of a Euclidean space, while the parametric hypothesis~$h_\theta(x)$ is Lipschitz continuous in~$x$ for every fixed~$\theta$ and Lipschitz continuous in~$\theta$ for every fixed~$x$. For example, $h_\theta$ could be a linear hypothesis with gradient~$\theta$ or a multi-layer neural network whose weight matrices are encoded by~$\theta$. To find an optimal trade-off between prediction loss and unfairness, we can solve the fair learning problem
\begin{equation}\label{eq:fair-metric-learning}
    \min_{\theta\in\Theta} ~ \mathbb{E}[L(h_\theta(X), Y)] + \mathcal{U}(h_\theta),
\end{equation}
which differs from~\eqref{eq:loss-min} only in that its objective function involves an unfairness penalty in the form of $\mathcal{U}(h_\theta) = \rho(\mathcal{D}_{\Psi}(\PP_{h_\theta(X)|A=0}, \PP_{h_\theta(X)|A=1}))$,
where $\mathcal{D}_{\Psi}$ is an IPM and $\rho:\mathbb R_+\rightarrow\mathbb R_+$ is a smooth and non-decreasing regularization function. Below we will discuss how to solve problem~\eqref{eq:fair-metric-learning} to local optimality when the distribution~$\PP_{(X,Y,A)}$ of~$X$, $Y$ and~$A$ is only indirectly observable through training samples. 

\subsection{Empirical Risk Minimization}
\label{sec:empirical}
Assume that~$(\hat X_i,\hat Y_i,\hat A_i)$, $i\in\mathbb N$, is a stochastic process of independent and identically distributed (i.i.d.) training samples, all of which follow the probability distribution~$\PP_{(X,Y,A)}$. Define now~$\tau^a_t\in\mathbb N$ for each~$t\in\mathbb N$ and~$a\in\mc A$ via the recursion $\tau^a_{t}=\inf\{i>\tau^a_{t-1}:\hat A_i=a\}$ initialized with~$\tau^a_0=0$. Thus,~$\tau^a_t$ represents the index of the $t$-th sample in class~$a$. Note that~$\tau^a_{t}$ is a stopping time in the sense that the event~$\{\tau^a_{t}=i\}$ belongs to the $\sigma$-algebra generated by~$\hat X_1,\ldots,\hat X_i$ for any~$i\in\mathbb N$. To exclude trivialities, we assume from now on that~$\PP[A=a]>0$ for all~$a\in\mc A$, which implies that~$\tau^a_t$ is $\PP$-almost surely finite. Next, define~$\hat X^a_t=\hat X_{\tau^a_t}$ as the feature vector of the $t$-th sample in class~$a$. By~\cite[Lemma~5.3.4]{ref:chow1997probability}, the random vectors~$\hat X^a_t$, $t\in\mathbb N$, are i.i.d.\ for all~$a\in\mc A$. A simple generalization of the same argument shows that these random vectors are independent across all~$t\in\mathbb N$ and~$a\in\mc A$. In addition, $\hat X^a_t$ follows the probability distribution~$\PP_{X|A=a}$ of~$X$ conditional on~$A=a$ because
\begin{align*}
     \PP\left[\hat X^a_t\leq x\right] & =\PP\left[\hat X_1^a\leq x\right]=\sum_{i\in\mathbb N} \PP\left[\hat X_i\leq x,\,\tau^a_1=i\right] \\ & =\sum_{i\in\mathbb N}\PP\left[\hat X_i\leq x,\,\hat A_1\neq a,\ldots,\hat A_{i-1}\neq a,\hat A_i=a\right]\\
     & =\sum_{i\in\mathbb N}\PP\left[X\leq x,\,A=a\right] \PP\left[A\neq a\right]^{i-1}=\frac{\PP\left[X\leq x,\,A=a\right]}{1-\PP[A\neq a]}= \PP\left[X\leq x|A=a\right]
\end{align*}
for all~$x\in\mc X$, where the second equality exploits the law of total probability, the fourth equality holds because the training samples are i.i.d., and the sixth equality follows from the observation that~$1-\PP[A\neq a]=\PP[A=a]$. Next, for any fixed~$N\in\mathbb N$ define~$T^a_N=\max_{t\ge 0}\{t:\tau^a_t\leq N\}$, which counts how many out of the first~$N$ training samples belong to class~$a$. One can show that the random counters~$T^0_N$ and~$T^1_N$ are independent of~$\hat X^a_t$ for all~$t\in\mathbb N$ and~$a\in\mc A$. Using these notational conventions, we introduce conditional empirical distributions
\[
    \hat\PP^a_{N,\theta} =\frac{1}{T_N^a} \sum_{t=1}^{T^a_N} \delta_{h_\theta(\hat X^a_t)} \quad\forall a\in\mc A,
\]
where~$\delta_z$ denotes the Dirac measure at~$z\in\R$, and the empirical risk minimization problem
\begin{equation}
    \min_{\theta \in\Theta} ~ \frac{1}{N}\sum\limits_{i=1}^N L(h_\theta(\hat X_i), \hat Y_i) + \rho\left( \mc D_{\Psi}(\hat\PP^0_{N,\theta}, \hat\PP^1_{N,\theta} ) \right).
    \label{eq:saa}
    \tag{ERM}
\end{equation}
Problem~\eqref{eq:saa} is generically non-convex even if the hypotheses depend linearly on~$\theta$, and even if~$L$ and~$\rho$ are convex. Indeed, if~$\mc D_\Psi$ is an MMD distance and $\rho$ is quadratic, for example, then one can show that the unfairness penalty in~\eqref{eq:saa} represents a sum of convex and concave functions of~$\theta$. The lack of convexity is of little concern if the hypotheses represent multi-layer neural networks, which already display a non-trivial dependence on~$\theta$. We therefore propose to solve problem~\eqref{eq:saa} to local optimality via gradient descent-type algorithms. Problem~\eqref{eq:saa} is indeed amenable to such algorithms for commonly used hypothesis classes ({\em e.g.} linear hypotheses or neural networks with piecewise linear activation functions) and standard loss functions provided that $\rho(\mc D_{\Psi}(\mathbb Q_1, \mathbb Q_2))$ is piecewise differentiable with respect to the support points of any discrete distributions~$\mathbb Q_1$ and~$\mathbb Q_2$. The computational complexity of a gradient step critically depends on the choice of~$\rho$ and~$\mathds D_\psi$. For example, when $\mathds D_\psi$ is an MMD distance, evaluating the gradient of the unfairness penalty in~\eqref{eq:saa} requires~$\Omega(N^2)$ arithmetic operations, which is prohibitive for large training sets. 

\subsection{Stochastic Approximation}\label{sec:stochastic-approximation}
The state-of-the-art method for solving empirical risk minimization problems over large datasets and complex hypothesis spaces is the stochastic gradient descent (SGD) algorithm~\citep{ref:robbins1951stochastic} or its variants such as Adam~\citep{ref:kingma2014adam, ref:sashank2018convergence, ref:zhang2018improved} or Adadelta~\citep{ref:zeiler2012adadelta}.
In their simplest form, these algorithms mimic ordinary gradient descent but approximate each gradient using a single training sample. Thus, SGD uses far less memory and time per iteration than gradient descent at the cost of noisy updates. Unfortunately, it is not evident how SGD can be applied to problem~\eqref{eq:fair-metric-learning} because there is no meaningful way to estimate the unfairness penalty from one sample.
Additionally, SGD is not parallelizable and thus cannot exploit the full power of multicore CPUs and GPUs.
Mini-batch SGD uses gradient estimators constructed from several training samples and thus interpolates between gradient descent and plain vanilla SGD. A key advantage of mini-batch SGD is that it is amenable to efficient implementations. In order to guarantee converge, mini-batch SGD requires {\em unbiased} estimators for the gradients of the objective function of problem~\eqref{eq:fair-metric-learning}. Unfortunately, the empirical risk itself, that is, the objective function of~\eqref{eq:saa}, constitutes a {\em biased} estimator for the objective function of~\eqref{eq:fair-metric-learning}. This is a direct consequence of the following lemma.

\begin{lemma}[The empirical unfairness penalty is biased]
\label{lem:biased-empirical-risk}
For all~$N\in\mathbb N$ and~$\theta\in\Theta$ we have
\[
    \EE\left[ \rho\left(\left. \mc D_{\Psi}(\hat\PP^0_{N,\theta}, \hat\PP^1_{N,\theta} )\right) \right| T^0_N,T^1_N\ge 1\right] \geq \rho\left( \mc D_{\Psi}(\PP_{h_\theta(X)|A=0} , \PP_{h_\theta(X)|A=1})\right).
\]
The inequality is strict if $\rho$ increases strictly and  $\PP_{h_\theta(X)|A=0}=\PP_{h_\theta(X)|A=1}$ is no Dirac distribution.
\end{lemma}
Conditioning on~$T^a_N\ge 1$ ensures that the empirical distribution~$\hat\PP^a_{N,\theta}$ is well-defined.
Lemma~\ref{lem:biased-empirical-risk} shows that the empirical risk provides a {\em biased} estimator for the true risk, which suggests that the gradients of the empirical risk provide 
{\em biased} estimators for the gradients of the true risk. Nevertheless, the fair statistical learning problem~\eqref{eq:fair-metric-learning} is amenable to efficient SGD-type algorithms when~$\mc D_\Psi=d_{\rm MMD}$ is an MMD metric and~$\rho(z)=\lambda z^2$ is a quadratic penalty function with~$\lambda\ge 0$. We will now show that, in this special case, one can construct unbiased gradient estimators by using {\em random} batches of training samples. To this end, we set $\tau_1=1$ and define~$\tau_b$ for $b\geq 2$ recursively as the smallest integer satisfying $\tau_b- \tau_{b-1} \geq \bar N$ such that the set~$\mathcal I_b=\{\tau_{b},\ldots,\tau_{b + 1}-1\}$ contains the indices of at least two training samples of each class~$a\in\mathcal A$. By construction, $|\mc I_b|$ is not smaller than a given target batch size~$\bar N$. Defining~$\mc I_b^a=\{i\in\mc I_b:\hat A_i=a\}$, we further have $|\mc I_b^a|\geq 2$ for each~$a\in\mc A$. Using a similar reasoning as in Section~\ref{sec:empirical}, one can show that the index~$\tau_b$ of the first sample in the $b$-th batch is $\mathbb P$-almost surely finite and constitutes a stopping time for every~$b\in\mathbb N$. Conditional on~$|\mc I^a_b|=N^a$, one can also show that~$\{\hat X_i:i\in\mc I^a_b\}$ is a family of~$N^a$ i.i.d.\ training features governed by~$\PP_{X|A=a}$ for every~$a\in\mc A$. By construction, both the cardinality~$|\mc I_b|$ of the $b$-th batch and the cardinality $|\mc I^a_b|$ of its subfamily corresponding to any class~$a\in\mc A$ are random. Then,

\begin{align*}
    \hat U_{b}(\theta) = \sum\limits_{a \in \mc A}\ \sum\limits_{i, j \in \mc I_b^a,\, i\neq j} \!\! \frac{K(h_\theta(\hat X_{i}), h_\theta(\hat X_{j}))}{|\mc I_{b}^{a}| (|\mc I^{a}_b|\!-\!1)} -  2\sum\limits_{\substack{i\in \mc I_{b}^{0},\, j\in \mc I_{b}^{1}}} \!\!\frac{ K(h_\theta(\hat X_{i}), h_\theta(\hat X_{j}))}{|\mc I^{0}_{b}| |\mc I^{1}_{b}|} 
\end{align*}
is an unbiased estimator for the squared MMD distance between $\PP_{h_\theta(X)|A=0}$ and $\PP_{h_\theta(X)|A=1}$; see Proposition~\ref{prop:unbiased_mmd}. Note that if the number of training samples in each class $a\in\mc A$ was {\em deterministic} and not smaller than~2, then~$\hat U_b(\theta)$ would reduce to a classical~$U$-statistic and thus constitute the minimum variance unbiased estimator for the squared MMD distance \cite[\S~5]{ref:serfling2009approximation}. Given a stream of i.i.d.\ training samples~$(\hat X_i,\hat Y_i,\hat A_i)$, $i\in\mathbb N$, however, there is always a positive probability that a batch of deterministic cardinality contains less than two samples of one class, in which case $\hat U_{b}(\theta)$ is not defined. Hence, working with batches of random cardinality seems unavoidable to correct the bias in the empirical unfairness penalty. Unfortunately, this randomness introduces a bias in the empirical prediction error~$|\mc I_b|^{-1} \sum_{i \in \mc I_b} L(h_\theta(\hat X_i), \hat Y_i)$; see Proposition~\ref{prop:empirical_loss_Ib_bias}. While one could construct an unbiased estimator for the prediction error using only the first~$\bar{N}$ training samples in the $b$-th batch, this would amount to sacrificing the last $|\mc I_b|-\bar N$ training samples and thus be data-inefficient. We circumvent this problem by introducing bias correction terms defined via the auxiliary function
\[
    \Delta(N, n) = \mathbbm 1_{N = \bar N } + \frac{N}{2(N-1)} \mathbbm{1}_{(N > \bar N )\wedge (n = 2)} + \frac{N}{N-1} \mathbbm 1_{(N > \bar N) \wedge ( n = N-2) },
\]
where $N\in\{\bar N,\bar N+1,\ldots\}$ and $n \in\{2,\ldots,N-2\}$. Using this definition, we can prove that
\[
    \hat R_b(\theta) =\frac{1}{|\mc I_b|} \sum\limits_{a\in\mc A}  \sum\limits_{i \in \mc I_b^a}\Delta(|\mc I_b|, |\mc I_b^a|) L(h_\theta(\hat X_i), \hat Y_i)
\]
constitutes an unbiased estimator for~$\EE[L(h_\theta(X), Y)]$; see Proposition~\ref{prop:empirical_risk_unbiased_est}. In summary, Propositions~\ref{prop:unbiased_mmd} and~\ref{prop:empirical_risk_unbiased_est} imply that~$\hat R_b(\theta) + \lambda \hat U_{b}(\theta)$ is an unbiased estimator for the objective function of the fair learning problem~\eqref{eq:fair-metric-learning}. Unbiased gradient estimators are available under the following {technical assumption, which is satisfied by all kernels and loss functions commonly used in the machine learning literature.
}

\begin{assumption}[Uniform integrability]
    \label{ass:regularity}
    The loss function $L$ as well as the kernel function $K$ are piecewise differentiable, and every~$\theta\in\Theta$ has a neighborhood~$\Theta_0\subseteq\Theta$ such that 
\begin{equation*}
    \EE\left[\sup_{\theta\in\Theta'_0} \left\| \nabla_\theta K(h_\theta (\hat X_1), h_\theta(\hat X_2)) \right\|_2 \right]<\infty\quad\text{and} \quad \EE\left[\sup_{\theta\in\Theta'_0} \left\| \nabla_\theta L(h_\theta(\hat X_1),\hat Y_1)\right\|_2 \right]<\infty,
\end{equation*}
where~$\Theta_0'$ is the subset of~$\Theta_0$ on which the gradients exist. 
\end{assumption}

\begin{theorem}[Unbiased gradient estimators]
\label{thm:unbiasedness_batch}
If~$\mc D_\Psi=d_{\rm{MMD}}$ is an MMD metric, $\rho(z)=\lambda z^2$ is a quadratic penalty function with $\lambda\ge 0$ and Assumption~\ref{ass:regularity} holds, then~$\nabla_{\theta}\hat R_b(\theta) + \lambda\nabla_\theta\hat U_b(\theta)$ constitutes an unbiased estimator for the gradient of the objective function of problem~\eqref{eq:fair-metric-learning}.
\end{theorem}

In practice, one may use automatic differentiation to evaluate the gradient estimators of Theorem~\ref{thm:unbiasedness_batch} ({\em e.g.}, using Autograd in PyTorch or GradientTape in TensorFlow). {SGD-type algorithms using these unbiased estimators are guaranteed to converge, in expectation, to a stationary point of problem~\eqref{eq:fair-metric-learning} provided that the following mild continuity and smoothness conditions are satisfied. 

\begin{assumption}[Continuity and smoothness conditions]\label{ass:smoothness} The following hold.
\begin{enumerate}
    \item[(a)] The hypothesis $h_\theta(x)$ is $S_h$-smooth and $L_h$-Lipschitz continuous in $\theta$ for all fixed~$x\in\mc X$.
    \item[(b)] The loss function $L(\hat y, y)$ is $S_L$-smooth and $L_L$-Lipschitz continuous in~$\hat y$ and admits an unbiased gradient oracle with bounded variance for all fixed~$y\in\mc Y$.
    \item[(c)] The Kernel function $K(y_1,y_2)$ is $S_K$-smooth and $L_K$-Lipschitz continuous in~$(y_1,y_2)$.
\end{enumerate}
\end{assumption}

Assumptions~(a) and~(b) imply that the loss function $L(h_\theta(x),y)$ is smooth in $\theta$ and thus ensure the convergence of SGD-type algorithms in the absence of unfairness penalties. Assumptions~(a) and~(c) imply that $K(h_\theta(x_1),h_\theta(x_2))$ is smooth in $\theta$. Note that many commonly used kernels such as the radial basis function or sigmoid kernels satisfy this assumption. Thus, Assumption~\ref{ass:smoothness} does not impose any restrictive conditions beyond those that are already needed in the absence of unfairness penalties.

\begin{theorem}[Convergence rate of SGD]
\label{thm:convergence-rate}
Suppose that Assumptions~\ref{ass:regularity} and~\ref{ass:smoothness} hold, $\mc D_\Psi=d_{\rm{MMD}}$ is an MMD metric, $\rho(z)=\lambda z^2$ is a quadratic penalty function with $\lambda\ge 0$ and $F(\theta)=\mathbb{E}[L(h_\theta(X), Y)] + \mathcal{U}(h_\theta)$ is the objective function of~\eqref{eq:fair-metric-learning}. Then, there exist $\overline{\gamma},C>0$ such that the following holds: If $\{\theta_t\}_{t=1}^T$ are the iterates of a mini-batch SGD algorithm with gradient estimator $\nabla_{\theta}\hat R_b(\theta) + \lambda\nabla_\theta\hat U_b(\theta)$ and stepsize~$\gamma\in (0,\overline{\gamma}]$, then
$$
    \frac{1}{T}\sum_{t=1}^T \|\nabla F(\theta_t)\|^2 \leq\frac{2}{\gamma T} \left(F(\theta_1)-\min_{\theta\in\Theta} F(\theta)\right) + \gamma C.
$$
\end{theorem}
\begin{proof}[Proof of Theorem~\ref{thm:convergence-rate}]
From Theorem~\ref{thm:unbiasedness_batch} we know that the gradient estimator at hand is unbiased under Assumption~\ref{ass:regularity}. As the unfairness penalty is proportional to a squared MMD metric, it becomes evident from Definition~\ref{def:mmd} that~$F$ is representable as the expected value of a sum of composite integrands. By Assumptions~\ref{ass:smoothness}, these composite integrands inherit the smoothness and Lipschitz continuity from their component functions. In addition, by~\cite[Lemma~1]{ref:Glassermann-differentiability}, which applies thanks to Assumption~\ref{ass:regularity},
it is easy to show that~$F$ is $S_F$-smooth for some~$S_F>0$. Using similar arguments, one can show that $\hat R_b(\theta) + \lambda \hat U_b(\theta)$ is $L_{\hat F}$-Lipschitz for some $L_{\hat F}>0$.  
We thus have $\|\nabla_{\theta}\hat R_b(\theta) + \lambda\nabla_\theta\hat U_b(\theta)\|\leq L_{\hat F}$ for all~$\theta\in\Theta$. Popoviciu's inequality on variances then implies that ${\rm Var}(\nabla_{\theta}\hat R_b(\theta) + \lambda\nabla_\theta\hat U_b(\theta))\leq L_{\hat F}^2/4$. The claim finally follows from~\cite[Lemma~A.1]{hu2021bias} if we set $\overline{\gamma}=\frac{1}{2S_{ F}}$ and $C=S_F L_{\hat F}^2/4$. 
\end{proof}

Theorem~\ref{thm:convergence-rate} implies that, under standard regularity conditions, SGD-type algorithms based on our unbiased gradient estimator converge to a stationary point of problem~\eqref{eq:fair-metric-learning} at rate~$\mc O(1/\sqrt{T})$.
}

\section{Numerical Experiments}\label{sec:numerical}
Our approach to fair learning consists in solving problem~\eqref{eq:fair-metric-learning}, where the unfairness penalty is constructed from the $\mc L^2$-distance $\mathcal{D}_\Psi=d_2$ and the quadratic regularization function~$\rho(z) = 2\lambda z^2$ with regularization weight~$\lambda\geq 0$. Hence, the unfairness penalty is proportional to the energy distance. Moreover, as the $\mc L^2$-distance is a special case of an MMD metric with a distance-induced kernel, we can use the techniques developed in Section~\ref{sec:num_fair_learning} to obtain unbiased batch gradients (see Theorem~\ref{thm:unbiasedness_batch}) and solve problem~\eqref{eq:fair-metric-learning} with the Adam optimizer~\citep{ref:kingma2014adam}. In the following we refer to the procedure outlined above as {\em metrized fair learning} (MFL). 
MFL is implemented in PyTorch~\citep{ref:NEURIPS2019_bdbca288}, and all experiments are run on an Intel~i7-10700 CPU~(2.9GHz) computer with 32GB RAM. The corresponding codes are available at
{\url{https://github.com/RAO-EPFL/Metrizing-Fairness}.} 

\subsection{Online Learning}\label{sec:exp-online}
In the first experiments we use synthetic and real data to show that even slightly biased gradient estimators can have a detrimental effect on the performance of SGD-type algorithms.


\subsubsection{Regression}
Consider first a fair regression problem with~$d=10$ features, where the feature vector~$X$ follows the uniform distribution on~$[0,1]^9\times\{0,1\}$, the protected attribute~$A$ coincides with the last component of~$X$, and the target satisfies~$Y = \max_{j\in[5]} \langle s_{j}, X\rangle$ for some independent random vectors~$s_{j}$ that follow the uniform distribution on~$[-2,2]^{10}$. The training samples~$(\hat X_i,\hat Y_i,\hat A_i)$, $i\in\mathbb N$, are drawn independently from~$\PP_{(X,Y,A)}$, where the vectors~$s_{j}$ are re-sampled only when $i$ is a multiple of~$2{,}000$ and kept constant otherwise. We solve the resulting regression problem with the proposed MFL method, where the loss function and the regularization weight are set to~$L(\hat y,y)=(\hat y-y)^2$ and~$\lambda=2$, respectively, while the hypothesis space~$\mc H$ is identified with the family of all neural networks with one hidden layer accommodating 20~nodes, ReLU (Rectified Linear Unit) activation functions at the hidden layer and linear activation functions at the output layer. The target batch size~$\bar N$ is either set to $4$ or to~$50$, and the choice of all other hyperparameters is detailed in Appendix~\ref{app:trainig_details}.
To showcase the merits of unbiased gradient estimators, we solve the regression problem at hand also with a variant of MFL, which uses the gradient of the empirical risk corresponding to a given batch of $\bar N$ training samples as a biased estimator for the gradient of the true risk.

Figure~\ref{fig:online-learning} visualizes the out-of-sample risk of the optimal regressor output by the original unbiased MFL approach and its biased variant as a function of the target batch size~$\bar N\in\{4,50\}$ and the number of training samples. The out-of-sample risk is evaluated on 1{,}000 test samples that follow the same distribution as the last seen training sample. The solid lines represent averages and the shaded areas visualize standard errors corresponding to five independent replications of the same experiment. Choosing the larger target batch size~$\bar N=50$ slows down convergence of both methods because a gradient step can only be implemented once at least~50 training samples have been accumulated. However, if the gradient estimators are biased, then a large target batch size is needed to ensure convergence to a low out-of-sample risk. Indeed, the out-of-sample risk of the biased MFL method with~$\bar N=4$ saturates already after about 500 samples at a relatively high level, while the out-of-sample risk corresponding to~$\bar N=50$ continues to decrease at a steady rate even after 4{,}000 training samples. Since standard MFL works with unbiased gradient estimators, it does not suffer from such a trade-off, that is, smaller values of~$\bar N$ are preferable both in terms of convergence speed and in terms of the out-of-sample risk at equilibrium. Hence, increasing~$\bar N$ only reduces the variance of the gradient estimators. 

\begin{figure}[t]
    \centering

    \begin{subfigure}[t]{0.49\columnwidth}
    \includegraphics[width=\linewidth]{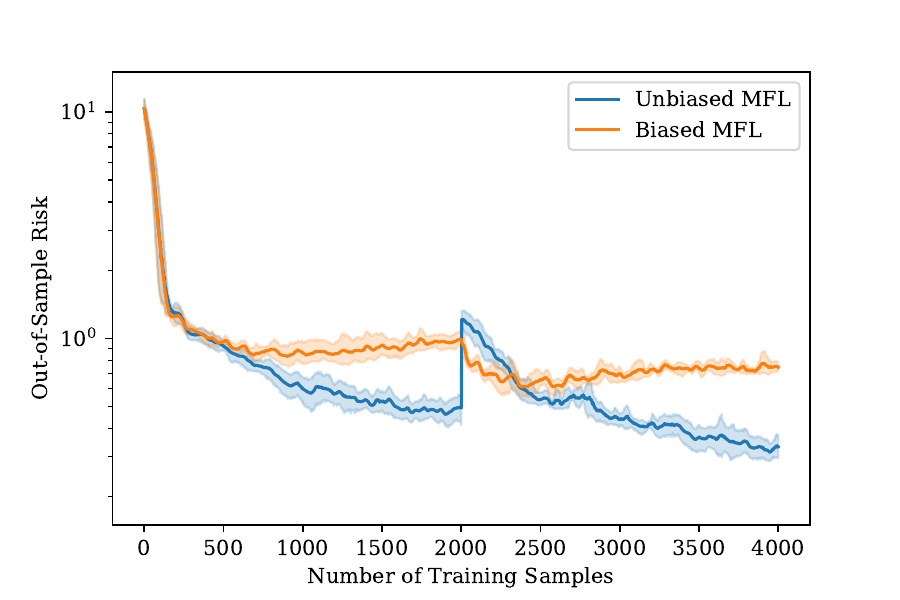}
    \caption{$\bar{N}=4$}
    \end{subfigure}
    \begin{subfigure}[t]{0.49\columnwidth}
    \includegraphics[width=\linewidth]{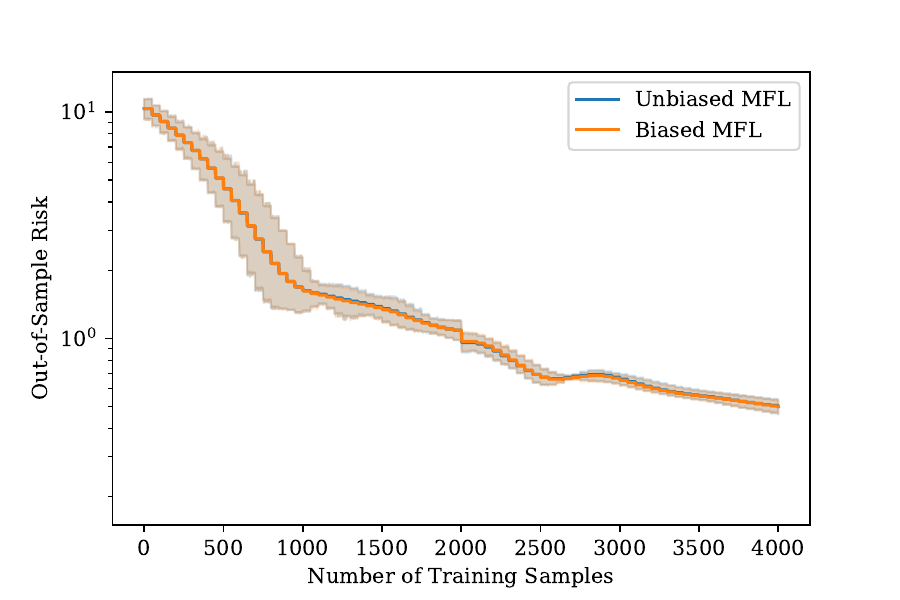}
    \caption{$\bar{N}=50$}
    \end{subfigure}
    
    \caption{Test loss of the optimal regressors output by the biased and unbiased MFL methods for target batch sizes $\bar{N}\in\{4,50\}$ as a function of the number of training samples}
    \label{fig:online-learning}
\end{figure}

The key insights of this first experiment can be summarized as follows. Methods that use biased gradient estimators either converge quickly to a low-quality solution (if $\bar N$ is small) or 
converge slowly to a high-quality solution (if $\bar N$ is large). Methods that use {\em un}biased gradient estimators, on the other hand, 
converge quickly to a high-quality solution (if~$\bar N$ is small). Quick convergence to equilibrium is particularly desirable when the distribution of the samples changes over time, {\em e.g.}, due to occasional regime shifts as in our experiment.

\subsubsection{Classification}
Consider now the fair classification problem based on the drug dataset~\citep{ref:fehrman2017five} described in Appendix~\ref{app:numerical}. An infinite stream of training samples~$(\hat X_i,\hat Y_i,\hat A_i)$, $i\in\mathbb N$, is generated by repeatedly concatenating shuffled copies of the training set. Note, however, that the resulting data stream is only approximately i.i.d. We then solve the fair classification problem with our MFL approach, where the loss function and the regularization weight are set to the cross entropy loss~\cite[\S~16.5.4]{ref:murphy2012machine} and $\lambda=1$, respectively, while the hypothesis space~$\mathcal H$ is identified with the family of all neural networks with one hidden layer accommodating 16 nodes, ReLU activation functions at the hidden layer and sigmoid activation functions at the output layer. The target batch size~$\bar N$ is treated as a free parameter, and the choice of all other hyperparameters is detailed in Appendix~\ref{app:trainig_details}. To showcase the merits of unbiased gradient estimators, we solve the classification problem at hand also with a variant of MFL, which uses the gradient of the empirical risk corresponding to a given batch of $\bar N$ training samples as a biased estimator for the gradient of the true risk. Finally, we use the biased variant of MFL to solve an `unfair' version of the classification problem with~$\lambda=0$. This is tantamount to solving the classification problem without an unfairness penalty via empirical risk minimization. In this case, the gradient estimators cease to be biased even though they are constructed from deterministic data batches that are not guaranteed to contain at least two samples of each class~$a\in\mathcal A$. 


Figure~\ref{fig:ablation} visualizes the accuracy-fairness trade-off of the optimal classifiers for various target batch sizes~$\bar N$ (color-coded), where SP-unfairness is measured by the Kolmogorov metric. The training of all classifiers is stopped after passing through $500$ shuffled copies of the training set (which contains $1{,}413$ samples), and their accuracy and SP-unfairness are evaluated on the test set (which contains 472 samples). Dots represent averages and error bars represent standard errors corresponding to ten independent replications of the same experiment. 
Figure~\textsc{\ref{fig:ablation-unregularized}} shows that classical empirical risk minimization without an unfairness penalty yields classifiers with an accuracy of at least $80\%$ and a relatively high level of SP-unfairness of about~$20\%$. We also observe that the algorithm's performance is insensitive to the batch size~$\bar N$.
Next, Figure~\textsc{\ref{fig:ablation-biased}} shows that
including an unfairness penalty with weight $\lambda=1$ and solving the classification problem via biased MFL significantly reduces unfairness at the expense of slightly reducing accuracy. Using biased gradient estimators ostensibly renders the performance highly sensitive to the batch size~$\bar N$. This is in stark contrast to standard MFL with {\em un}biased gradient estimators, which outputs classifiers that are largely independent of the target batch size and display a consistently high accuracy and low unfairness; see Figure~\textsc{\ref{fig:ablation-unbiased}}. Since MFL uses the squared $\mathcal L^2$-distance (that is, the energy distance) to penalize unfairness, one can show that the per-iteration memory usage of the Adam optimizer scales as~$\Theta(\bar N^2)$. Hence, even though MFL with large mini-batches is susceptible to parallel implementations, its memory consumption poses a challenge as~$\bar N$ increases. 
It has also been observed that small batch sizes improve the generalization performance of the trained classifiers and speed up convergence \cite{ref:keskar2016large, ref:lecun2012efficient, ref:masters2018revisiting, ref:wilson2003general}. We thus conclude that small batch sizes are preferable from a computational perspective. However, their computational benefits can only be harnessed when unbiased gradient estimators are available.


\begin{figure}
    \centering
    \begin{subfigure}[t]{0.32\columnwidth}
    \includegraphics[width=\linewidth]{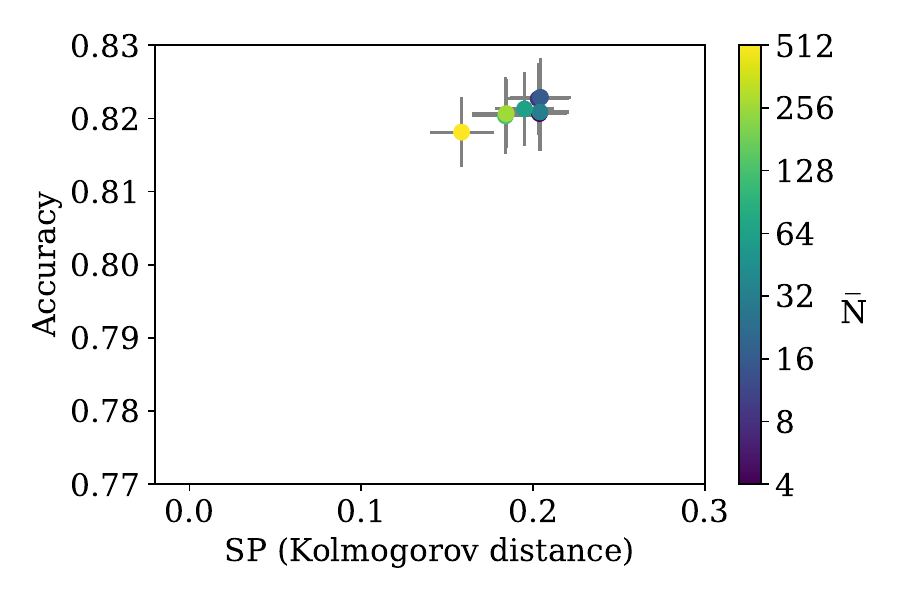}
    \caption{Training w/o unfairness penalty}
    \label{fig:ablation-unregularized}
    \end{subfigure}\hspace{2mm}
    \begin{subfigure}[t]{0.32\columnwidth}
    \includegraphics[width=\linewidth]{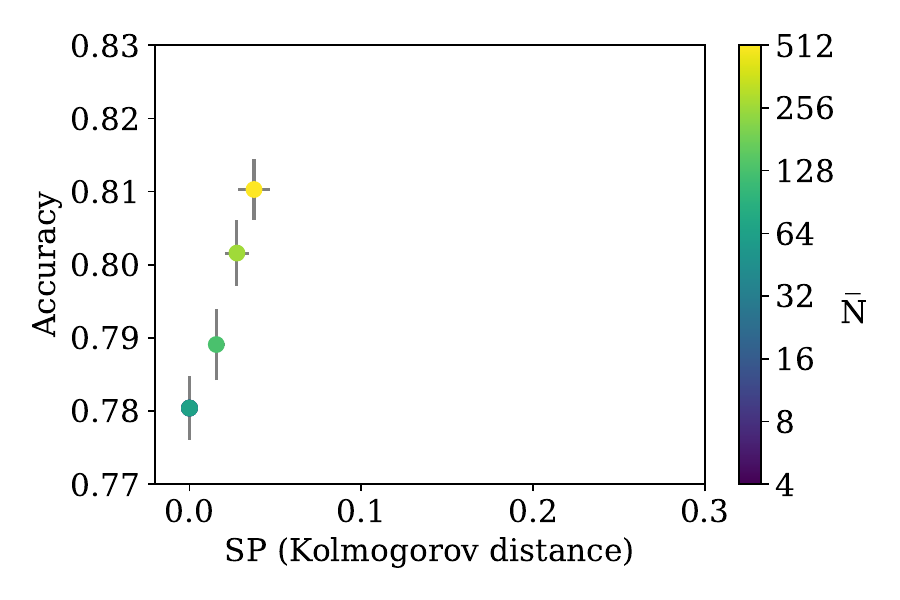}
    \caption{Training with biased gradient estimators and with unfairness penalty}
    \label{fig:ablation-biased}
    \end{subfigure}\hspace{2mm}
    \begin{subfigure}[t]{0.32\columnwidth}
    \includegraphics[width=\linewidth]{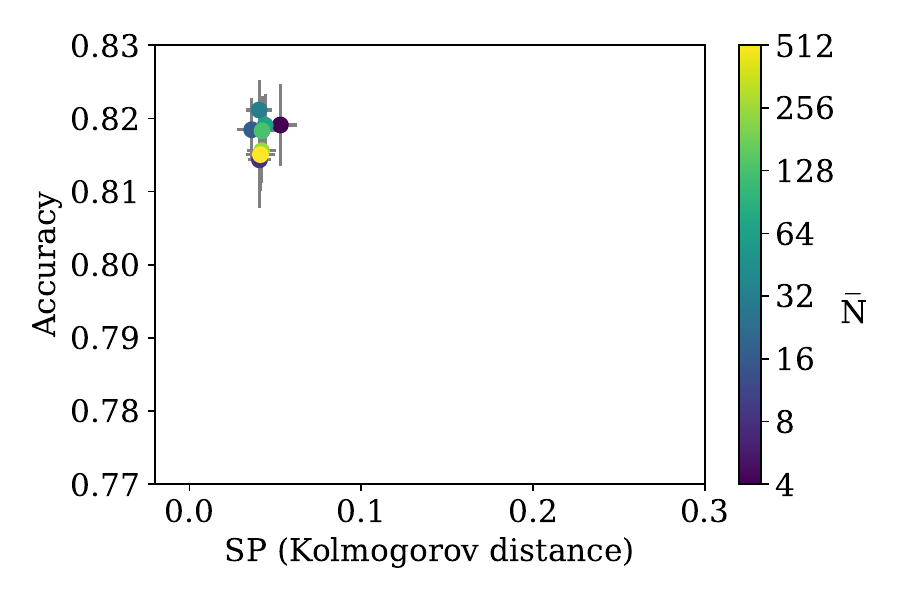}
    \caption{Training with unbiased gradient estimators and with unfairness penalty}
    \label{fig:ablation-unbiased}
    \end{subfigure}
    \caption{Impact of the target batch size $\Bar{N}$ (color-coded) on the means (dots) and std.\ errors (error bars) of the accuracy and the SP-unfairness of the trained classifiers on test data}
    \label{fig:ablation}
\end{figure}

\subsection{Offline Learning}\label{sec:exp-offline}
Oftentimes we have no access to an infinite stream of independent training samples but are only given a finite training dataset in tabular form. In this offline setting, the best we can hope for is to solve the empirical risk minimization problem~\eqref{eq:saa}. This problem can be interpreted as an instance of~\eqref{eq:fair-metric-learning}, where~$\PP_{(X,A,Y)}$ corresponds to the discrete empirical distribution on the given training samples, and thus the methods of Section~\ref{sec:num_fair_learning} can readily be applied to find a local minimizer. Specifically, an infinite stream of independent training samples can be generated by sampling from the given dataset. However, problem~\eqref{eq:saa} can also be solved with a simpler offline MFL method that uses batches of a fixed deterministic size. This is possible because the given dataset can be used to generate two separate (finite) data streams, each containing independent samples from only one class~$a\in\mathcal A$. Note that the class probabilities $p_a=\PP[A=a]$ for $a\in\mc A$ are easy to compute if~$\PP_{(X,A,Y)}$ represents the empirical distribution. Using the two data streams, we can then construct batches that contain exactly $\lceil p_0\bar N\rceil$ samples of class~$0$ and $\bar N-\lceil p_0\bar N\rceil$ samples of class~$1$. In this case, the batch cardinalities~$|\mc I_b|$, $|\mc I_b^0|$ and~$|\mc I_b^1|$ become deterministic, and one readily verifies that the gradient estimator derived in Section~\ref{sec:num_fair_learning} remains unbiased. The only restriction is that one must choose a sufficiently large batch size~$\bar N$ to ensure that~$|\mc I_b^a|\geq 2$ for every~$a\in\mathcal A$. All experiments in this section are based on tabular datasets and thus use offline MFL to solve the empirical risk minimization problem~\eqref{eq:saa}. More precisely, below we assess the accuracy-fairness trade-offs and runtimes of offline MFL on five standard datasets (Drug~\citep{ref:fehrman2017five}, Communities\&Crime (CC)~\citep{ref:redmond2002data,ref:Dua2019}, Compas~\citep{ref:propublica}, Adult~\citep{ref:Dua2019} and Student Grades~\citep{cortez2008student_dataset}).

\subsubsection{Regression}
\label{sec:offline-regression} 
In all regression experiments, we solve an instance of~\eqref{eq:saa}, where~$L$ is the squared error. In linear regression, $\mathcal H$ is set to the family of all linear hypotheses. In neural network regression, on the other hand, $\mathcal H$ comprises all neural networks with one hidden layer accommodating 20 nodes, ReLU activation functions at the hidden layer and linear activation functions at the output layer. For further details see Appendix~\ref{app:trainig_details}. We compare offline MFL against two methods by Berk et al.~\citep{ref:berk2017convex}, which train a linear regressor by solving a convex optimization problem implemented in CVXPY~\citep{agrawal2018CVXPY}.
We sweep the unfairness penalty parameter~$\lambda$ of the MFL method from~$10^{-5}$ to~$10^{3}$ for the Student Math and Student Portugese datasets, and from~$10^{-5}$ to~$10^{2}$ for the CC dataset, all in~50 equal steps on a log~scale. Similarly, we sweep the unfairness penalty parameter~$\lambda$ of the methods by Berk et al.~\citep{ref:berk2017convex} from~$10^{-2}$ to~$10^{5}$ in~50 equal steps on a log~scale.

Figure~\ref{fig:R2_fairness_plots} visualizes the trade-off between the goodness-of-fit (measured by the coefficient of determination~$R^2$) and SP-unfairness (measured by the Kolmogorov metric) of the trained hypotheses on test data averaged over 10 independent replications of the same experiment with randomly permuted data. Table~\ref{tab:auc_regression} reports the average areas under the respective goodness-of-fit-unfairness curves (AUC; see Appendix~\ref{app:auc} for a precise definition) as well as the average training times. We observe that MFL consistently attains superior AUC values and is at least one order of magnitude faster than the baseline methods on the large Communities and Crime dataset. On the smaller Student Math and Student Portugese datasets the runtimes of MFL remain competitive.

\begin{figure}
    \begin{center}
    \begin{subfigure}[t]{0.32\columnwidth}
    \centering
    \includegraphics[width=\columnwidth]{  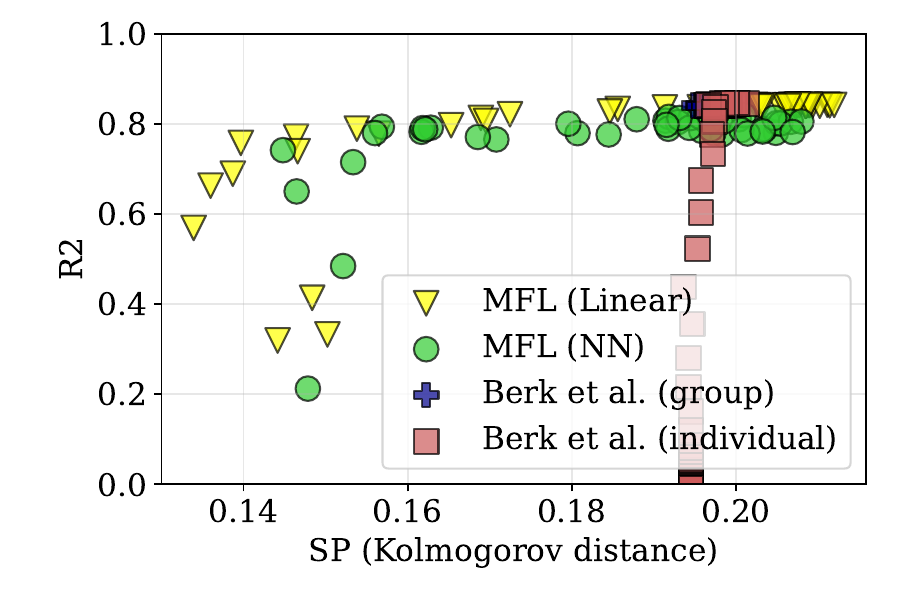}
    \vspace{-.7cm}
    \caption{Student Portugese}
    \label{fig:studentsportugese_regression}
    \end{subfigure}
    \begin{subfigure}[t]{0.32\columnwidth}
    \centering
    \includegraphics[width=\columnwidth]{  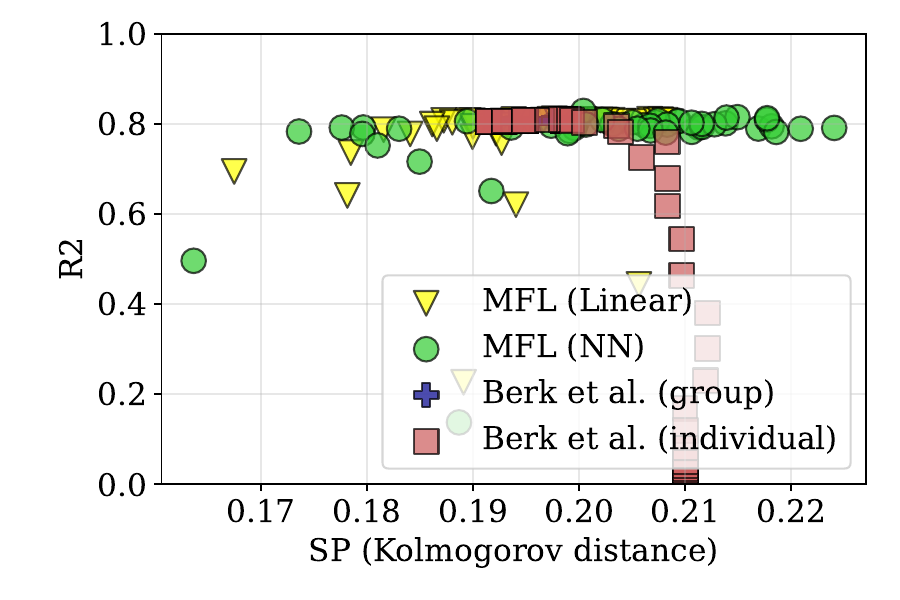}
    \vspace{-.7cm}
    \caption{Student Math}
    \label{fig:studentsmath_regression}
    \end{subfigure}
    \begin{subfigure}[t]{0.32\columnwidth}
    \centering
    \includegraphics[width=\columnwidth]{ 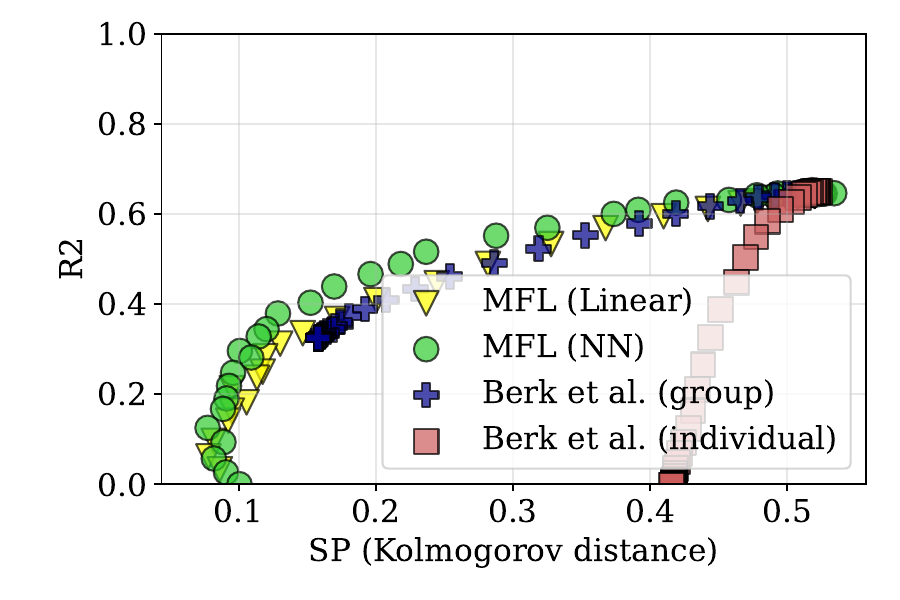}
    \vspace{-.7cm}
    \caption{Communities$\&$Crime}
    \label{fig:communitiescrime_regression}
    \end{subfigure}
    \end{center}
\centering
\vspace{-.4cm}
    \caption{$R^2$ vs SP-unfairness on test data for regression tasks averaged over 10 simulations}
    \label{fig:R2_fairness_plots}
\end{figure}

\begin{table}[h]
    \centering
    \caption{AUC values (mean $\pm$ std.\ error) and training times (mean) for regression tasks}
    \vspace{-5pt}
    \begin{tabular}{llcccc}
    \toprule
         Dataset &Performance &$\substack{\text{\citet{ref:berk2017convex}} \\ \text{(individual)}}$& $\substack{\text{\citet{ref:berk2017convex}} \\\text{(group)}}$ & MFL~(Linear)&MFL~(NN)\\\midrule
         \multirow{2}{*}{Student Math} &AUC&0.531$\pm$0.042  &0.430$\pm$0.055 &0.624$\pm$0.029 &\textbf{0.634$\pm$0.034}\\
         &Time & 1.48 secs &0.21 secs &4.07 secs &4.69 secs\\ \midrule
         \multirow{2}{*}{Student Portugese} &AUC &0.679$\pm$0.013 &0.642$\pm$0.029 & \textbf{0.733$\pm$0.012} &0.704$\pm$0.016\\
         &Time&4.67 secs &1.10 secs& 6.12 secs &7.06 secs\\ \midrule
         \multirow{2}{*}{CC}&AUC &0.357$\pm$0.008 & 0.503$\pm$0.009 & 0.517$\pm$0.007 & \textbf{0.541$\pm$0.005}\\
         &Time& 710.23 secs&168.49 secs & 11.39 secs &16.54 secs\\
         \bottomrule
    \end{tabular}
    \label{tab:auc_regression}
\end{table}

\subsubsection{Classification} 
\label{sec:offline-classification}
In all classification experiments, we solve an instance of \eqref{eq:saa}, where $L$ is the cross entropy loss~\cite[\S~16.5.4]{ref:murphy2012machine}. As in Section~\ref{sec:offline-regression}, we work with two different hypothesis spaces~$\mathcal H$. In linear classification, we set $\mathcal H$ to the family of all linear hypotheses. In neural network classification, on the other hand, we set it to the family of all neural networks with one hidden layer accommodating 16 nodes, ReLU (Rectified Linear Unit) activation functions at the hidden layer and sigmoid activation functions at the output layer. 
For further details see Appendix~\ref{app:trainig_details}. We compare offline MFL against three baselines: (i)~a convexified empirical logistic regression model with relaxed fairness constraints proposed by~\citet{ref:zafar2017fairness}, (ii)~an instance of~\eqref{eq:saa} with cross entropy loss proposed by~\citet{ref:cho2020fair}, which is solved with the Adam optimizer using biased gradient estimators, and (iii)~an instance of~\eqref{eq:saa} with cross entropy loss and a Sinkhorn divergence-based unfairness penalty proposed by~\citet{ref:oneto2020expoliting}, which is solved via gradient descent. Detailed information on these baselines is also provided in Appendix~\ref{app:trainig_details}. In all models we sweep the unfairness penalty parameter~$\lambda$ from~$10^{-5}$ to~10 in~25 equal steps on a log~scale.

\begin{figure}
    \begin{center}
    \begin{subfigure}[t]{0.24\columnwidth}
    \centering
    \includegraphics[width=\columnwidth]{ 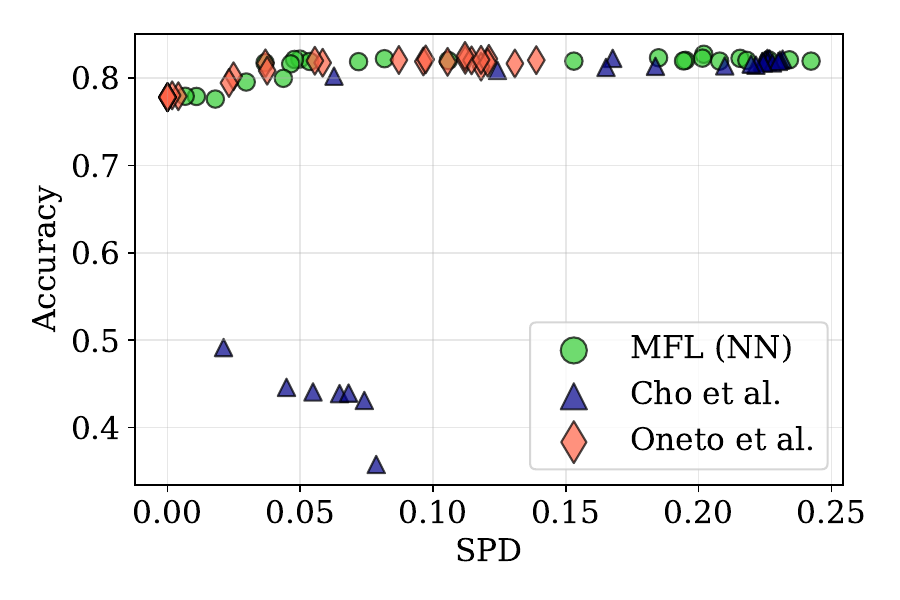}
    \vspace{-.7cm}
    \caption{Drug}
    \label{fig:drug}
    \end{subfigure}
    \begin{subfigure}[t]{0.24\columnwidth}
    \centering
    \includegraphics[width=\columnwidth]{ 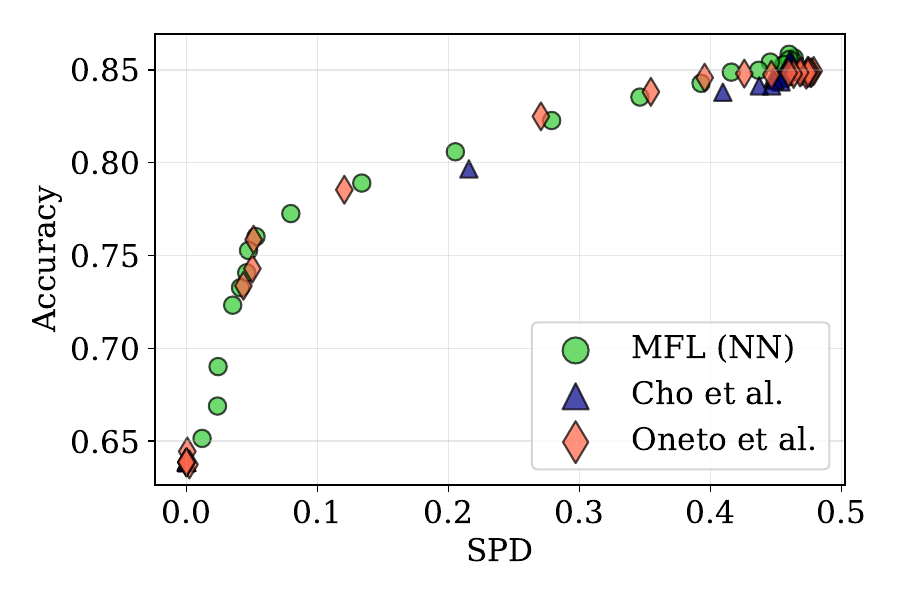}
    \vspace{-.7cm}
    \caption{Communities$\&$Crime}
    \end{subfigure}
    \begin{subfigure}[t]{0.24\columnwidth}
    \centering
    \includegraphics[width=\columnwidth]{ 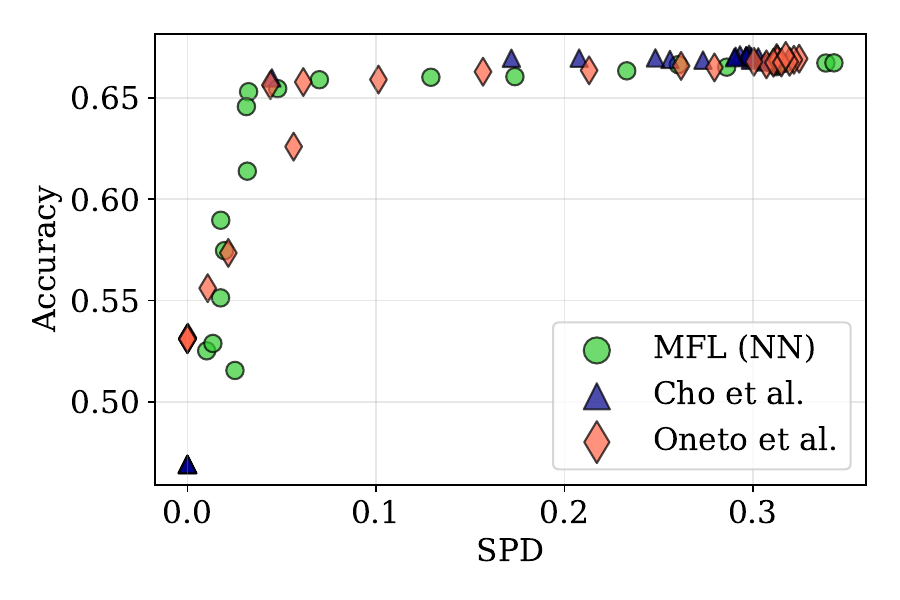}
    \vspace{-.7cm}
    \caption{Compas}
    \end{subfigure}
    \begin{subfigure}[t]{0.24\columnwidth}
    \centering
    \includegraphics[width=\columnwidth]{ 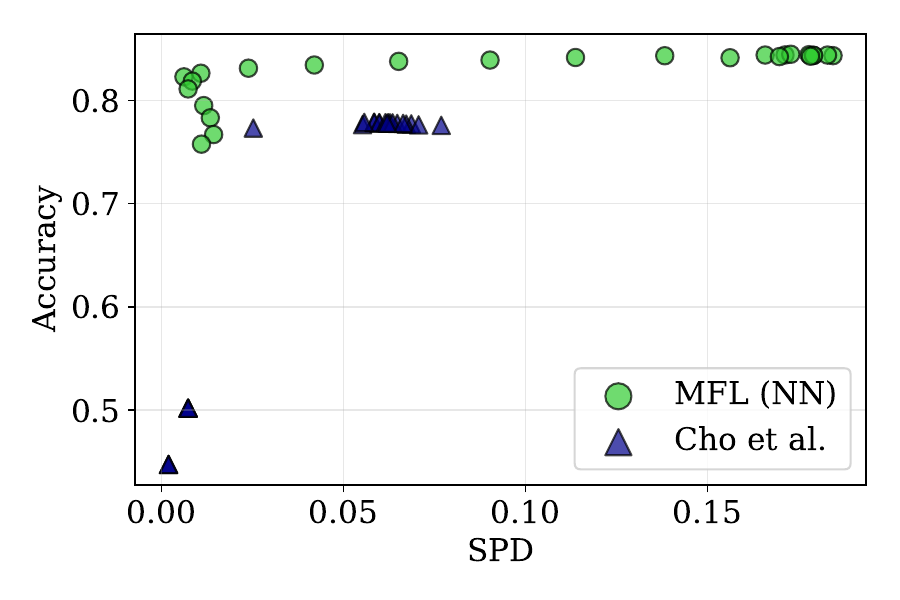}
    \vspace{-.7cm}
    \caption{Adult}
    \label{fig:adult}
    \end{subfigure}
    \begin{subfigure}[t]{0.24\columnwidth}
    \centering
    \includegraphics[width=\columnwidth]{ 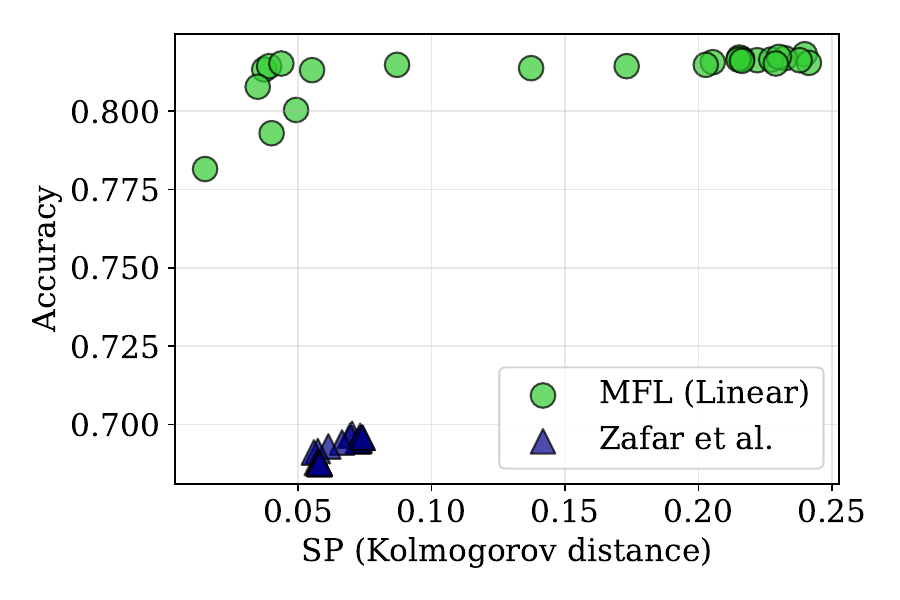}
    \vspace{-.7cm}
    \caption{Drug}
    \label{fig:drug-linear}
    \end{subfigure}
    \begin{subfigure}[t]{0.24\columnwidth}
    \centering
    \includegraphics[width=\columnwidth]{ 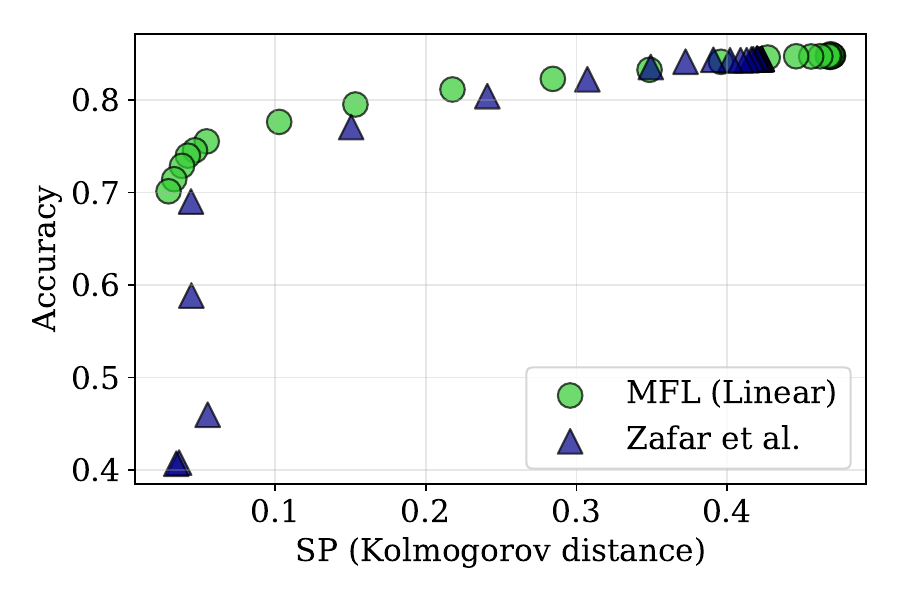}
    \vspace{-.7cm}
    \caption{Communities$\&$Crime}
    \end{subfigure}
    \begin{subfigure}[t]{0.24\columnwidth}
    \centering
    \includegraphics[width=\columnwidth]{ 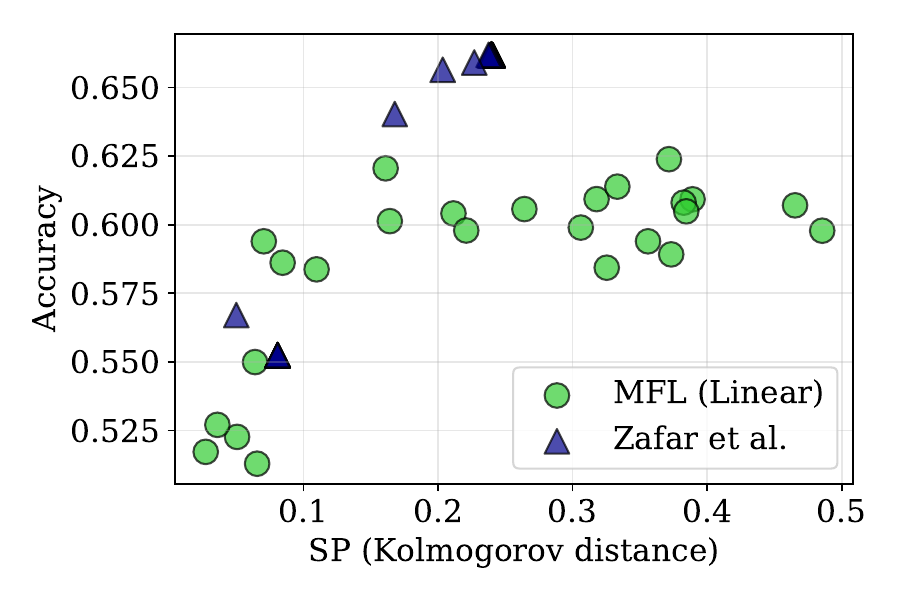}
    \vspace{-.7cm}
    \caption{Compas}
    \end{subfigure}
    \begin{subfigure}[t]{0.24\columnwidth}
    \centering
    \includegraphics[width=\columnwidth]{ 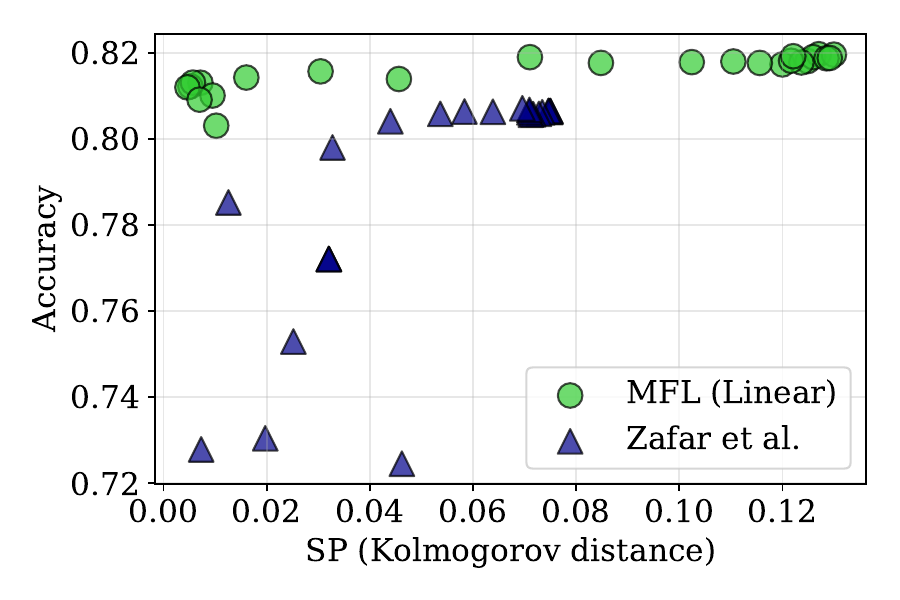}
    \vspace{-.7cm}
    \caption{Adult}
    \label{fig:adult-linear}
    \end{subfigure}
    \end{center}
\centering
\vspace{-.4cm}
    \caption{Accuracy vs SP-unfairness on test data for neural network-based (\textsc{\ref{fig:drug}}--\textsc{\ref{fig:adult}}) and linear (\textsc{\ref{fig:drug-linear}}--\textsc{\ref{fig:adult-linear}}) classification tasks averaged over 10 simulations}
    \label{fig:accuracy_fairness_plots}
\end{figure}
\begin{table}[h]
    \centering
    \caption{AUC values (mean $\pm$ std.\ error) and training times (mean) for classification tasks}
    \vspace{-5pt}
    \begin{tabular}{ll|ccc|cc}
    \toprule
         Dataset &Performance & \citet{ref:cho2020fair} & \citet{ref:oneto2020expoliting}&MFL~(NN)&\citet{ref:zafar2017fairness}&MFL~(Linear)\\\midrule
         \multirow{2}{*}{Drug} &AUC  &0.786$\pm$0.008 &\textbf{0.825$\pm$0.005} &0.823$\pm$0.005&0.668$\pm$0.013&\textbf{0.814$\pm$0.004}\\
         &Time &10.54 secs &416.18 secs &6.66 secs&0.35 secs& 5.89 secs \\ \midrule
         \multirow{2}{*}{CC}&AUC &0.817$\pm$0.004 &\textbf{0.829$\pm$0.004} &\textbf{0.829$\pm$0.003}&0.787$\pm$0.004 &\textbf{0.812$\pm$0.003}\\
         &Time& 13.52 secs& 217.27 secs & 8.39 secs&81.05 secs& 5.93 secs\\ \midrule
         \multirow{2}{*}{Compas} &AUC &\textbf{0.667$\pm$0.004} & 0.666$\pm$0.003 &0.661$\pm$0.003&\textbf{0.625$\pm$0.006 }&0.605$\pm$0.004\\
         &Time&17.33 secs& 2299.37 secs &9.99 secs&0.44 secs & 9.87 secs\\ \midrule
         \multirow{2}{*}{Adult}&AUC &0.769$\pm$0.002 &- &\textbf{0.842$\pm$0.000}&0.799$\pm$0.000 &\textbf{0.820$\pm$0.001}\\
         &Time&141.35 secs &- &149.61 secs& 29.45 secs& 159.80 secs\\ 
         \bottomrule
    \end{tabular}
    \label{tab:auc_time}
\end{table}

Figure~\ref{fig:accuracy_fairness_plots} visualizes the accuracy-fairness trade-off of the trained hypotheses on test data, where SP-unfairness is measured by the Kolmogorov metric, averaged over 10 independent replications of the same experiment with randomly permuted data. Note that the methods by~\citet{ref:cho2020fair} and~\citet{ref:oneto2020expoliting} allow for neural network classifiers, while the method by~\citet{ref:zafar2017fairness} can only handle linear classifiers and is thus at a disadvantage. Therefore, we compare the methods by~\citet{ref:cho2020fair} and~\citet{ref:oneto2020expoliting} only against MFL with neural network classifiers (Figures~\textsc{\ref{fig:drug}}--\textsc{\ref{fig:adult}}) and the method by~\citet{ref:zafar2017fairness} only against MFL with linear hypotheses (Figures~\textsc{\ref{fig:drug-linear}}--\textsc{\ref{fig:adult-linear}}). Table~\ref{tab:auc_time} reports the corresponding average AUC values as well as the average training times. The method by \citet{ref:oneto2020expoliting} failed to train on the Adult dataset due to memory constraints. Note that MFL performs favorably vis-\`a-vis its competitors in that it often attains the highest AUC values, while training faster than the other neural network-based methods. The method by~\citet{ref:zafar2017fairness} relies on convex optimization and thus provides fast computation when the input dimension~$d$ is low, but it cannot learn non-linear input-output relations. In addition, even when restricted to linear classifiers, MFL usually outperforms the method by~\citet{ref:zafar2017fairness} with regard to AUC. As \citet{ref:oneto2020expoliting} do not use unbiased gradient estimators, they resort to classical gradient descent algorithms to solve their empirical risk minimization models. For large training datasets, the time necessary for training is thus no longer competitive.

{
\subsubsection{Pre- and Post-Processing Methods}
So far we have compared MFL against other in-processing methods. However, fairness can also be enforced via \textit{pre}-processing methods, which eliminate biases from the dataset before training, and \textit{post}-processing methods, which alter the prediction scores of a pre-trained hypothesis. We now compare MFL against the \emph{pre}-processing methods by~\citet{jiang2020identifying}, who correct label biases by re-weighting the training samples, \citet{roh2021fairbatch}, who use adaptive unbiased minibatches for training, and~\citet{ref:kamiran2012data}, who employ \emph{massaging}~(M) or \emph{reweighing}~(R) to change some labels close to the decision surface or to create an adaptive reweighing of the dataset, respectively. 
In addition, we compare MFL against the  \emph{post}-processing method by~\citet{pmlr-v202-xian23b}, who solve a Wasserstein barycenter problem subject to fairness constraints in order to determine a threshold for mapping model scores to fair label predictions.

The next experiment focuses on the neural network-based classification tasks described in Section~\ref{sec:offline-classification}. Figure~\ref{fig:accuracy_fairness_plots_prepost} visualizes the accuracy-fairness trade-off of the trained hypotheses on test data, where SP-unfairness is measured by the Kolmogorov metric, averaged over 10 independent replications of the same experiment with randomly permuted data. Note that none of the \emph{pre}-processing methods has tunable fairness hyperparameters, and the \emph{post}-processing method cannot attain high fairness levels on the test set. It is therefore not meaningful to report AUC values. We observe that the pre-processing methods are dominated by MFL, which has the added benefit that the trade-off between accuracy and fairness can be tuned. The post-processing method by~\citet{pmlr-v202-xian23b} displays a similar performance as MFL but covers only a part of the Pareto frontier in the unfairness-accuracy plane. 
Note that MFL can reach higher fairness levels because it enforces SP on the \emph{scores} instead of the \emph{predicted labels}, which has a stronger regularizing effect. 
\begin{figure}
    \begin{center}
    \begin{subfigure}[t]{0.24\columnwidth}
    \centering
    \includegraphics[width=\columnwidth]{ 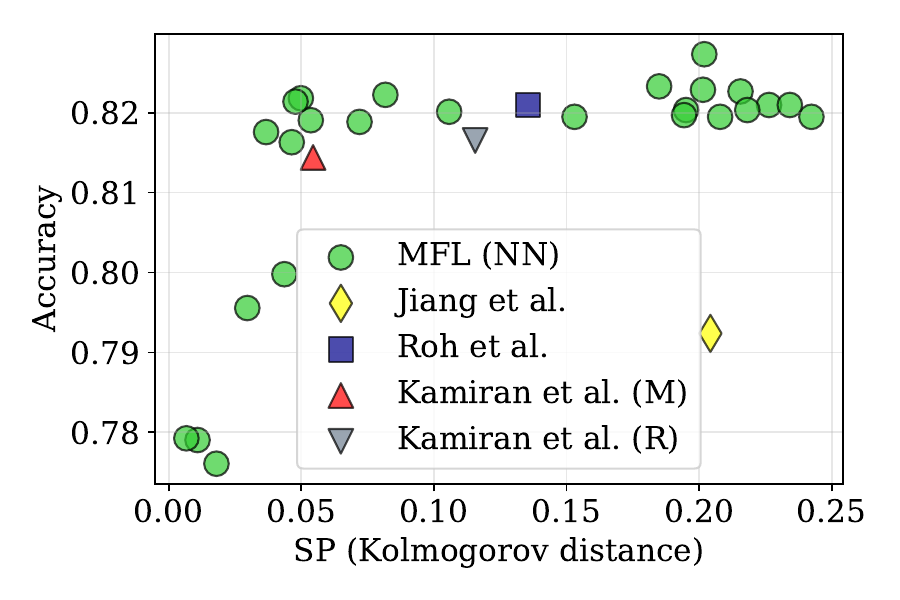}
    \vspace{-.7cm}
    \caption{Drug}
    \label{fig:drug_pre}
    \end{subfigure}
    \begin{subfigure}[t]{0.24\columnwidth}
    \centering
    \includegraphics[width=\columnwidth]{ 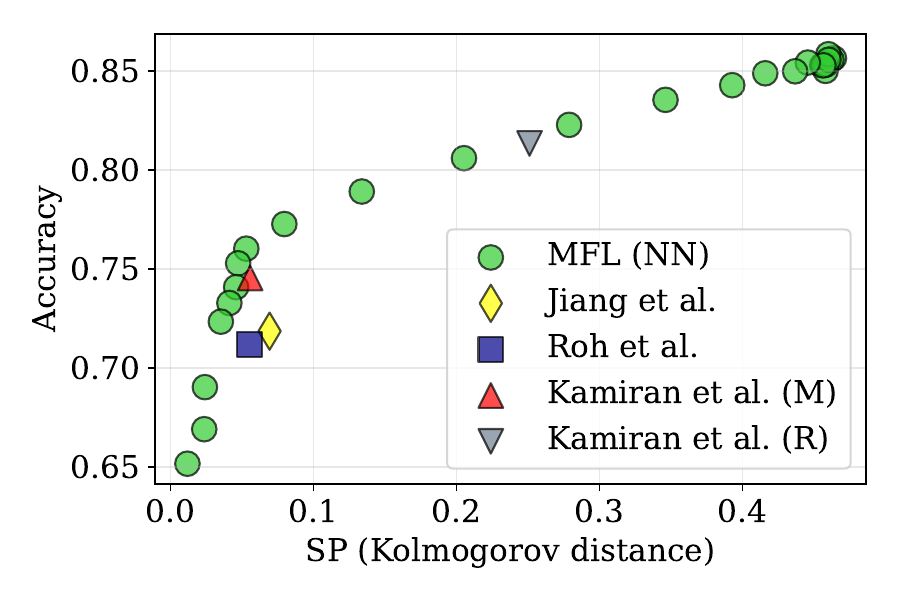}
    \vspace{-.7cm}
    \caption{Communities$\&$Crime}
    \end{subfigure}
    \begin{subfigure}[t]{0.24\columnwidth}
    \centering
    \includegraphics[width=\columnwidth]{ 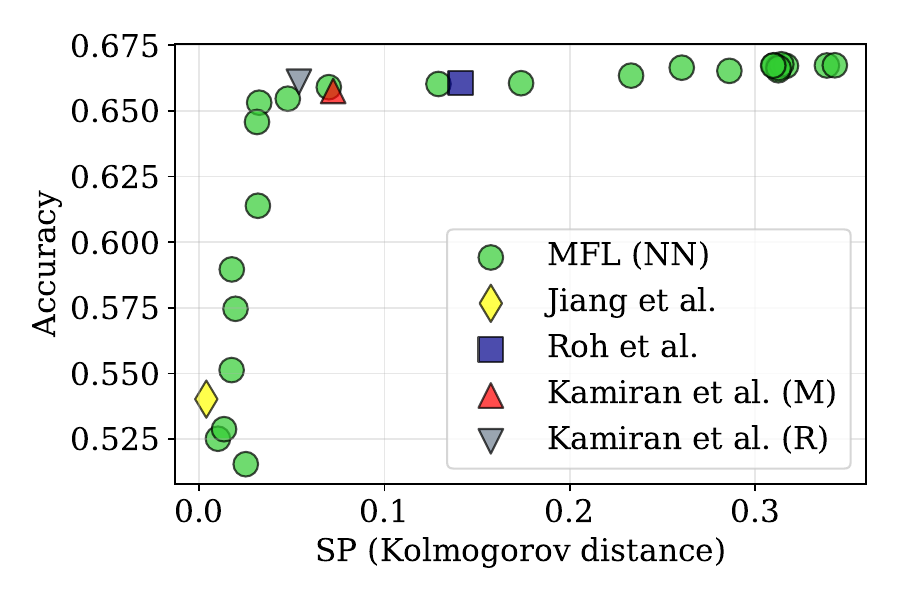}
    \vspace{-.7cm}
    \caption{Compas}
    \end{subfigure}
    \begin{subfigure}[t]{0.24\columnwidth}
    \centering
    \includegraphics[width=\columnwidth]{ 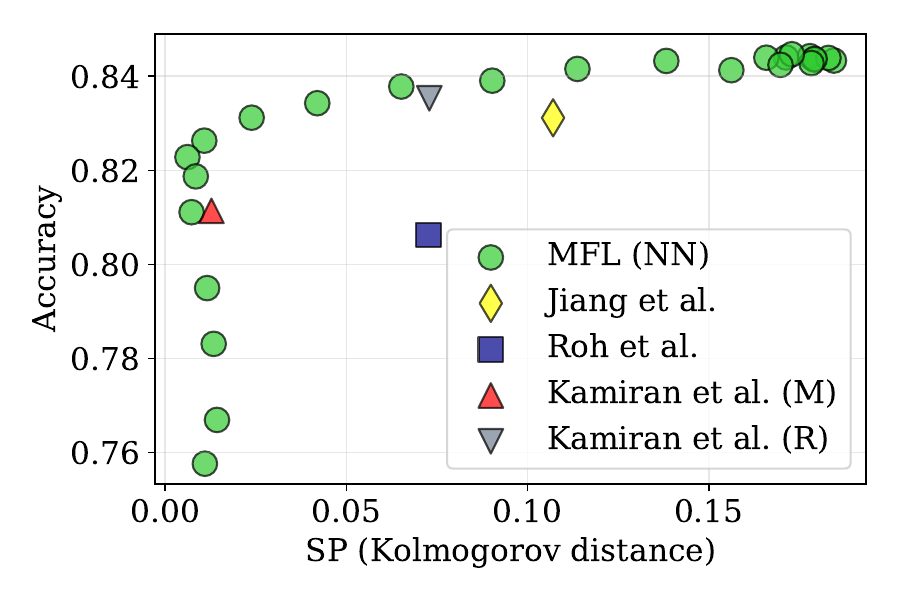}
    \vspace{-.7cm}
    \caption{Adult}
    \label{fig:adult_pre}
    \end{subfigure}
    \begin{subfigure}[t]{0.24\columnwidth}
    \centering
    \includegraphics[width=\columnwidth]{ 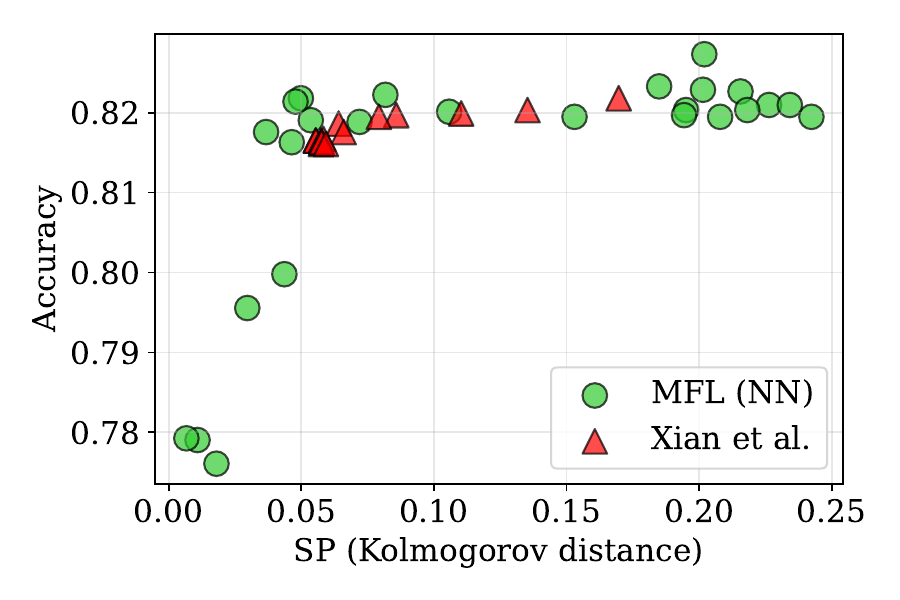}
    \vspace{-.7cm}
    \caption{Drug}
    \label{fig:drug_post}
    \end{subfigure}
    \begin{subfigure}[t]{0.24\columnwidth}
    \centering
    \includegraphics[width=\columnwidth]{ 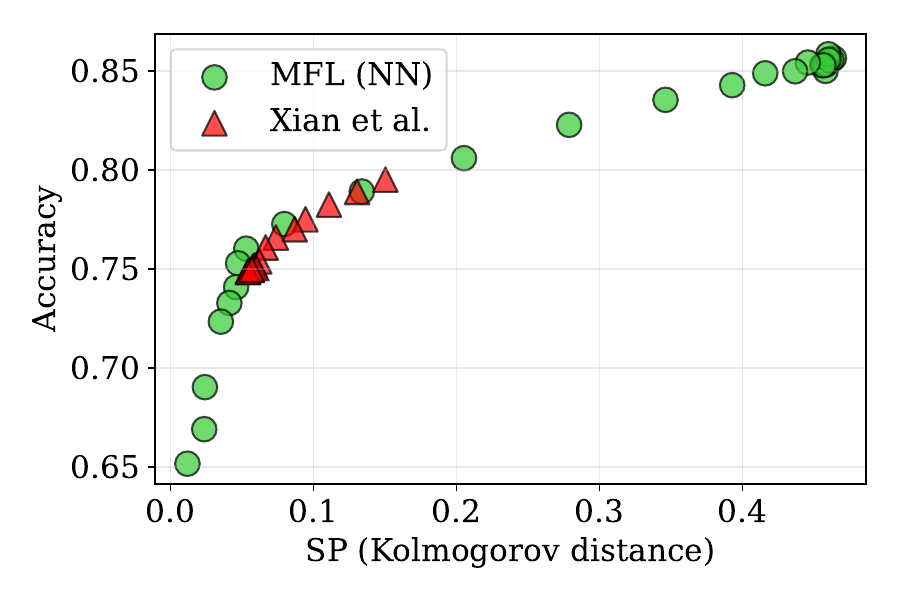}
    \vspace{-.7cm}
    \caption{Communities$\&$Crime}
    \end{subfigure}
    \begin{subfigure}[t]{0.24\columnwidth}
    \centering
    \includegraphics[width=\columnwidth]{ 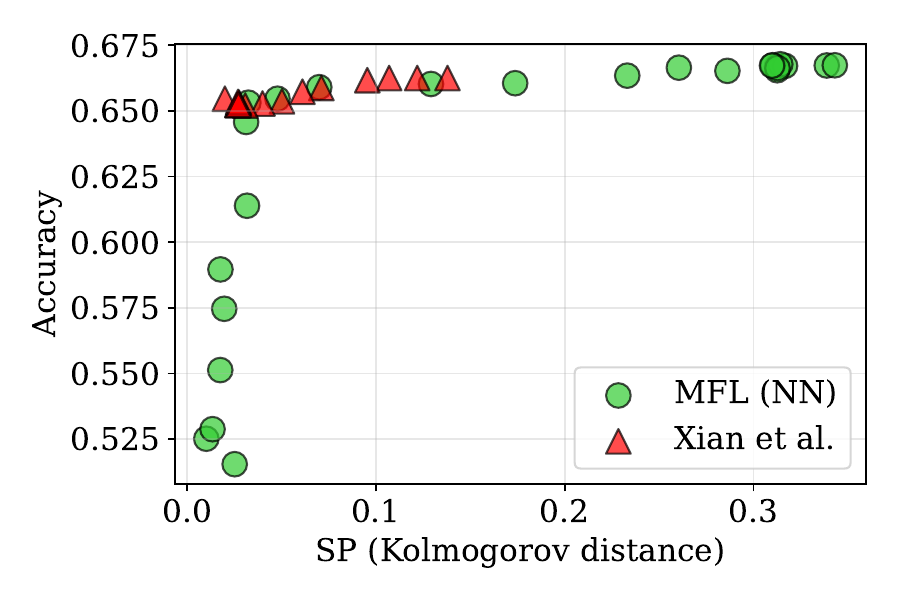}
    \vspace{-.7cm}
    \caption{Compas}
    \end{subfigure}
    \begin{subfigure}[t]{0.24\columnwidth}
    \centering
    \includegraphics[width=\columnwidth]{ 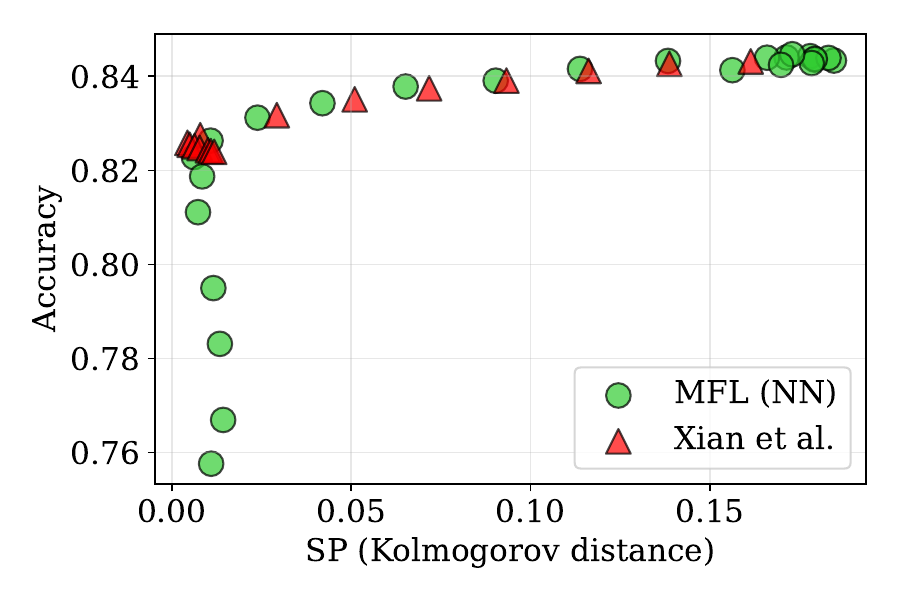}
    \vspace{-.7cm}
    \caption{Adult}
    \label{fig:adult_post}
    \end{subfigure}
    \end{center}
\centering
\vspace{-.4cm}
    \caption{Comparison of MFL against pre-processing (\textsc{\ref{fig:drug_pre}}--\textsc{\ref{fig:adult_pre}}) and post-processing (\textsc{\ref{fig:drug_post}}--\textsc{\ref{fig:adult_post}}) methods on neural network-based classification tasks averaged over 10 simulations.}
    \label{fig:accuracy_fairness_plots_prepost}
\end{figure}

}

\subsubsection{Alternative Fairness Criteria}
We now abandon our standard assumption that fairness is measured by SP and show that the methods developed in this paper extend to other group fairness criteria such as equal opportunity (EO) \cite{ref:hardt2016equality,ref:pleiss2017fairness}; see also Table~\ref{tab:eps-fairness-illustration}.
By definition, a hypothesis~$h\in\mathcal H$ is EO-fair if
the conditional distributions $\PP_{h(X) \leq \tau  | A=0, Y=1}$ and $\PP_{h(X) \leq \tau | A = 1, Y=1}$ match. Moreover, EO-unfairness can again be quantified by any IPM $\mathcal D_\Psi$ from Table~\ref{tab:ipm_examples}, that is, a natural unfairness penalty to be included in the fair learning model~\eqref{eq:fair-metric-learning} is $\mathcal{U}(h_\theta) = \rho(\mathcal{D}_{\Psi}(\PP_{h_\theta(X)|A=0, Y=1}, \PP_{h_\theta(X)|A=1, Y=1}))$, where $\rho$ is a smooth and non-decreasing regularization function. All techniques developed so far readily apply to this generalized setting with obvious modifications.

We now rerun all offline classification experiments from Section~\ref{sec:offline-classification} with the exact same parameter choices but use EO instead of SP as the fairness criterion. To ensure that each batch contains sufficiently many training samples with $\hat A_i=a$ as well as~$\hat Y_i=1$ for every~$a\in\mathcal A$, we simply increase the batch size to~512 in the experiments based on the Drug and Communities and Crime datasets. Table~\ref{tab:auc-acc-eo} reports the AUC values and the training times of the different methods, where EO-unfairness is measured by the Kolmogorov metric. All results are averaged over 10 independent replications of the same experiment. We find that MFL once again compares favorably against the three benchmark methods. Regarding AUC, MFL outperforms most of the benchmarks and is highly competitive with the rest of the methods. Moreover, MFL with neural networks trains multiple orders of magnitude faster than all other neural network-based methods.

\begin{table}[h!]
    \centering
    \caption{AUC values (mean $\pm$ std.\ error) and training times (mean) for classification tasks}
    \vspace{-5pt}
    \begin{tabular}{l@{\;\,}l|c@{\;\,}c@{\;\,}c|c@{\;\,}c}
    \toprule
         Dataset &Performance & \citet{ref:cho2020fair} & \citet{ref:oneto2020expoliting}&MFL~(NN)&\citet{ref:zafar2017fairness}&MFL~(Linear)\\\midrule
         \multirow{2}{*}{Drug} &AUC  &0.789$\pm$0.013 &\textbf{0.823$\pm$0.005} &0.815$\pm$0.005&0.653$\pm$0.021&\textbf{0.795$\pm$0.004}\\
         &Time &12.29 secs &57.78 secs &1.97 secs&0.10 secs& 1.63 secs \\ \midrule
         \multirow{2}{*}{CC}&AUC &0.821$\pm$0.006 &0.829$\pm$0.01 &\textbf{0.830$\pm$0.008}&0.749$\pm$0.006 &\textbf{0.812$\pm$0.004}\\
         &Time& 16.38 secs& 152.36 secs & 2.41 secs&26.12 secs& 2.09 secs\\ \midrule
         \multirow{2}{*}{Compas} &AUC &0.656$\pm$0.004 & \textbf{0.666$\pm$0.004} &0.660$\pm$0.004&\textbf{0.616$\pm$0.008}&0.599$\pm$0.004\\
         &Time&18.71 secs& 1703.84 secs &3.14 secs&0.12 secs & 2.89 secs\\ \midrule
         \multirow{2}{*}{Adult}&AUC &0.765$\pm$0.004 &- &\textbf{0.841$\pm$0.001}&0.740$\pm$0.000 &\textbf{0.788$\pm$0.004}\\
         &Time&154.44 secs &- &29.93 secs& 7.94 secs& 26.60 secs\\ 
         \bottomrule
    \end{tabular}
    \label{tab:auc-acc-eo}
\end{table}
\section{Concluding Remarks}
\label{sec:conclusion}
In this paper, we show that the generator of an IPM can be identified with a family of utility functions, which implies that IMPs have intuitive appeal as fairness regularizers that promote statistical parity or other notions of group fairness. We also demonstrate that particular IPMs such as the $\mc L^2$-distance, the square root of the energy distance or any MMD metric are computationally attractive because they lead to fair learning models that are susceptible to efficient batch SGD-type algorithms. These algorithms often outperform state-of-the-art methods for fair learning on standard datasets. We believe that our intuitive and scalable approach to fair learning is easy to use, thus hopefully boosting the adoption of fair learning models. At the same time, we acknowledge certain risks and limitations of the results and methods presented in this paper. More research is required to design fair learning processes that are immunized against adversarial attacks and missing data. This research might benefit from the observation that the severity of adversarial attacks can also be quantified by IPMs~\citep{ref:shafieezadeh2015distributionally}. In addition, our numerical experiments revealed that the performance of a fair predictor can be highly sensitive to the choice of the unfairness penalty parameter~$\lambda$. For a nuanced discussion on how to tune~$\lambda$ via a two-step validation process we refer to~\citep{ref:donini2018empirical}. Finally and most importantly, it should be kept in mind that not every application may require fairness constraints or that the imposed fairness constraints may have undesirable effects. Example~\ref{ex:price-of-fairness} in Section~\ref{sec:group-fairness} describes a scenario in which the imposition of statistical parity does not improve the well-being of the members of the protected group and even deteriorates the well-being of most people. Therefore, before deploying fair learning models in the real world, one should carefully assess all potential impacts of the chosen fairness regularizer. Our results of Section~\ref{sec:group-fairness} suggest that enforcing SP-fairness is generally desirable if it is known that the protected attribute~$A$ carries no information about the distribution of the output~$Y$ conditional on the input~$X$. For a critical assessment of arguments for and against statistical parity we refer to \cite{raz2021group}. We stress again that the methods developed in this paper readily extend to most group fairness notions such as equal opportunity~\citep{ref:hardt2016equality} or equalized odds~\citep{ref:hardt2016equality}. We further expect that the unbiased gradient estimators developed in Section~\ref{sec:num_fair_learning} are not only useful for SGD algorithms but also for other subsample-based methods such as stochastic gradient boosting~\cite{friedman2002SGB}.
\appendix

\renewcommand\thesection{\Alph{section}}

\renewcommand{\theequation}{A.\arabic{equation}}
\renewcommand{\thefigure}{A.\arabic{figure}}
\renewcommand{\thetable}{A.\arabic{table}}
\section*{Appendix}
This appendix is organized as follows. Appendix~\ref{app:aux} provides an auxiliary generalization bound, Appendix~\ref{app:proofs} contains all proofs omitted from the main text and provides several auxiliary results, Appendix~\ref{sec:discussions} provides background information on distance-induced kernels and highlights several generalizations of our main results, and Appendix~\ref{app:experiments} details the datasets as well as the training processes used in the numerical experiments and reports on additional simulation results.

\section{Generalization Bounds}\label{app:aux}
Throughout this section we assume that unfairness is measured by a generic IPM $\mathcal{D}_{\Psi}$; see Definition~\ref{def:ipm}.
We show that, for all sufficiently large sample sizes~$N$, the empirical prediction accuracy and the empirical unfairness approximate the population accuracy and the population unfairness, respectively. This result exploits Vapnik's classic generalization bound, which we restate below for convenience, and which depends on the error function $\mathcal E:\mbb N\times\mbb N\times(0,1)\to \R$ defined through $\mathcal{E}(\mathcal{V},N,\delta) = \frac{4}{N} (\mathcal{V}\left(\log(2N/\mathcal{V})+1\right)-\log(\delta/4))$.

\begin{theorem}[Generalization bound~\cite{vapnik1999nature}]
\label{thm:vapnik}
Consider binary functions $Q_\theta\in\mathcal L(\mathcal{X}\times\mathcal{Y},\{0,1\})$ indexed by~$\theta\in\Theta$, and assume that $\{(\hat X_i,\hat Y_i)\}_{i=1}^N$ are $N$ i.i.d.\ samples from $\mathbb{P}_{(X,Y)}$, where $N$ strictly exceeds the Vapnik-Chervonenkis (VC) dimension~$\mc V$ of the function class $\{Q_\theta\}_{\theta\in\Theta}$. For any $\delta\in(0,1)$, we then have 
    $$
        \frac{1}{N}\sum_{i=1}^NQ_\theta(\hat X_i,\hat Y_i)-\frac{1}{2}\sqrt{\mathcal{E}(\mathcal{V},N,\delta)}\; \leq\; \mathbb{E}[Q_\theta(X,Y)]\;\leq \; \frac{1}{N}\sum_{i=1}^NQ_\theta(\hat X_i,\hat Y_i)+\frac{1}{2}\sqrt{\mathcal{E}(\mathcal{V},N,\delta)}\quad\forall \theta\in\Theta
    $$
    with probability at least $1-\delta$.
\end{theorem}

\begin{theorem}[Generalization bounds for accuracy and unfairness]\label{thm:generalization}
Assume that $\mc Y=\{0,1\}$, and consider binary hypotheses~$h_\theta \in \mathcal{L}(\mc{X},\mc Y)$ indexed by~$\theta\in\Theta$. Define $Q_\theta\in\mathcal L(\mathcal{X}\times\mathcal{Y},\{0,1\})$ through $Q_\theta(x,y)=\mathbbm{1}_{h_\theta(x)=y}$ for all $x\in\mc X$ and $y\in\mc Y$, and denote the VC dimension of the function class $\{Q_\theta\}_{\theta\in\Theta}$ by~$\mathcal{V}$. Assume also that $\{(\hat X_i,\hat Y_i, \hat A_i)\}_{i=1}^N$ are $N$ i.i.d.\ samples from $\mathbb{P}_{(X, Y, A)}$, and select~$\delta \in(0,1)$. Then, conditional on the event $T^0_N,T^1_N>\mc V$, the prediction accuracy and the unfairness induced by some IPM $\mc D_\Psi$ simultaneously satisfy
    \begin{align*}
        \mathbb{P} \left[h_\theta(X)=Y\right] &\geq \frac{1}{N}\sum_{i=1}^N Q_\theta(\hat X_i,\hat Y_i) -\frac{1}{2}\sqrt{\mathcal{E}(\mathcal{V},N,\delta/3)} \quad \forall \theta\in\Theta
    \end{align*}
    and 
    \begin{align*}
        \mc D_{\Psi}( \PP_{h_\theta(X) | A = 0} ,  \PP_{h_\theta(X) | A=1}) &\leq  \mc D_{\Psi}(\hat\PP^0_{N,\theta} , \hat\PP^1_{N,\theta})+\frac{\alpha}{2}\sqrt{\mathcal{E}(\mathcal{V},T^0_N, \delta/3)}+\frac{\alpha}{2}\sqrt{\mathcal{E}(\mathcal{V},T^1_N, \delta/3)} \quad \forall \theta\in\Theta
    \end{align*}
with probability at least $1-\delta$, where $\alpha=\sup_{\psi\in \Psi}|\psi(0)-\psi(1)|$.
\end{theorem}

\begin{proof}[Proof of Theorem~\ref{thm:generalization}]
Assume first that there are $t_0$ i.i.d.\ samples from class~$0$ and $t_1=N-t_0$ i.i.d.\ samples from class~$1$ for some deterministic $t_0,t_1>\max\mc V$. By Theorem~\ref{thm:vapnik}, the population accuracy thus satisfies
\begin{equation}
    \mathbb{P}[h_\theta(X)=Y] \geq \frac{1}{N}\sum_{i=1}^N Q_\theta(\hat X_i,\hat Y_i) -\frac{1}{2}\sqrt{\mathcal{E}(\mathcal{V},N,\delta/3)}\quad \forall \theta\in\Theta
    \label{eq:ineq-acc}
\end{equation}
with probability at least $1 - \delta/3$. To analyze the unfairness, we temporarily use $\PP^a$ as a shorthand for the distribution $\mathbb{P}_{h_\theta(X)|A=a}$ and $\hat\PP^a$ as a shorthand for the distribution $\hat\PP^a_{N,\theta}$ for all~$a\in\mc A$. As~$\mc D_{\Psi}$ is a semi-metric and thus satisfies the triangle inequality, the population unfairness satisfies
\begin{equation}
    \label{eq:population-unfariness-inequality}
    \mathcal{D}_{\Psi}(\PP^0,\PP^1)\leq \mathcal{D}_{\Psi}(\hat\PP^0, \hat\PP^1) + \mathcal{D}_{\Psi}(\PP^0, \hat\PP^0) + \mathcal{D}_{\Psi}(\PP^1, \hat\PP^1).
\end{equation}
For any fixed~$a\in \mc A$, the definition of the IPM $\mc D_\Psi$ then implies that
\begin{align*}
    \mathcal{D}_{\Psi}(\PP^a, \hat\PP^a)&=\sup\limits_{\psi\in \Psi} \left|  \int_{\R}\psi(\hat y)\, \PP^a({\rm d} \hat y) - \int_{\R} \psi(\hat y) \,\hat\PP^a( \dd \hat y)\right|\nonumber\\
    &=\sup\limits_{\psi\in \Psi}  \left| \psi(0)\, \mathbb{P}^a[\{0\}]+\psi(1)\, \mathbb{P}^a[\{1\}] - \psi(0) \,\hat{\mathbb{P}}^a[\{0\}] -\psi(1) \,\hat{\mathbb{P}}^a[\{1\}]\right| \nonumber\\
    &\leq\sup\limits_{\psi\in \Psi}\left|\psi(0)-\psi(1)\right|\cdot \left|\mathbb{P}^a[\{1\}]-\hat{\mathbb{P}}^a[\{1\}]\right| =\alpha\cdot \left|\mathbb{P}^a[\{1\}]-\hat{\mathbb{P}}^a[\{1\}]\right|,
\end{align*}
where the second equality holds because~$h_\theta$ is a binary function, which implies that $\PP^a$ and~$\hat\PP^a$ are supported on~$\{0,1\}$, and the inequality holds because $\PP^a[\{0\}]  = 1 - \PP^a[\{1\}]$ and $\hat \PP^a[\{0\}]  = 1 - \hat \PP^a[\{1\}]$. By the definitions of $\hat\PP^a=\hat\PP^a_{N,\theta}$ and $Q_\theta(x,1)$, one can verify that $\hat\PP^a[\{1\}]=t_a^{-1}\sum_{t=1}^{t_a} Q_\theta(\hat X^a_t,1)$. By Theorem~\ref{thm:vapnik}, we thus find
\begin{equation*}
    \left|\mathbb{P}^a[\{1\}]-\hat{\mathbb{P}}^a[\{1\}]\right| = \left|\PP^a[\{1\}]-\frac{1}{t_a}\sum_{t=1}^{t_a} Q_\theta(\hat X^a_t,1)\right| \leq\frac{1}{2}\sqrt{\mathcal{E}(\mathcal{V},t_a, \delta/3)} \quad\forall\theta\in\Theta,
\end{equation*}
which in turn implies via~\eqref{eq:population-unfariness-inequality} that $ \mathcal{D}_{\Psi}(\PP^a, \hat\PP^a)\leq \frac{\alpha}{2}\sqrt{\mathcal{E}(\mathcal{V},t_a, \delta/3)}$ for all $\theta\in\Theta$, with probability $1-\delta/3$. Next, using the union bound to combine the above estimates for~$a = 0$ and~$a=1$, we then obtain
\begin{equation}
    \mathcal{D}_{\Psi}(\PP^0,\PP^1) \leq  \mc D_{\Psi}(\hat\PP^0 , \hat\PP^1)+\frac{\alpha}{2}\sqrt{\mathcal{E}(\mathcal{V},T^0_N, \delta/3)}+\frac{\alpha}{2}\sqrt{\mathcal{E}(\mathcal{V},T^1_N, \delta/3)} \quad \forall \theta\in\Theta
    \label{eq:unf-inequality-2}
\end{equation}
with probability at least~$1-2\delta/3$. Applying the union bound once again reveals that~\eqref{eq:ineq-acc} and~\eqref{eq:unf-inequality-2} are simultaneously satisfied with probability at least~$1-\delta$. In the remainder we abandon the assumption that the number of samples in each class is known deterministically. Instead, we only assume that the random number~$T^a_N$ of samples in class~$a$ strictly exceeds~$\mc V$ for all~$a\in\mc A$. We thus have 
\begin{align*}
    \PP\big[\text{\eqref{eq:ineq-acc} and~\eqref{eq:unf-inequality-2} hold} \,\big| \,  T^0_N,T^1_N > \mc V \big] & =\sum_{t_0 =\mc V+1}^{N-\mc V-1} \PP\big[\text{\eqref{eq:ineq-acc} and~\eqref{eq:unf-inequality-2} hold} \, \big| \, T^0_N=t_0\big] \cdot \PP\big[ T^0_N=t_0  \,\big| \,  T^0_N,T^1_N > \mc V  \big]\\ 
    & \geq(1-\delta)\sum_{t_0 =\mc V+1}^{N-\mc V-1}  \PP\big[ T^0_N=t_0  \,\big| \,  T^0_N,T^1_N > \mc V  \big] = 1-\delta,
\end{align*}
where the inequality follows from the first part of the proof. Hence, the claim follows.
\end{proof}

\section{Proofs}
\label{app:proofs}
\begin{proof}[Proof of Proposition~\ref{thm:optimal-decision-indep}]
By the law of iterated conditional expectations, problem~\eqref{eq:loss-min} is equivalent to
$$
    \min\limits_{h \in \mc H} \EE[\EE[L (h(X), Y)|X]].
$$
Recall now that the loss function~$L$ is lower semi-continuous, which implies via Fatou's Lemma that the conditional expectation $\EE[L (\hat y, Y)|X=x]$ is lower semi-continuous in~$\hat y\in\mathbb R$. Recall also that the hypothesis space $\mc H=\mc L(\mc X,\mathbb R)$ contains all real-valued Borel-measurable functions and that problem~\eqref{eq:loss-min} is assumed to be solvable. Therefore, \cite[Theorem~14.60]{rockafellar2009variational} implies that problem~\eqref{eq:loss-min} is equivalent to
\[
    \EE\left[\min\limits_{\hat y\in\mathbb R}\EE[L (\hat y, Y)|X]\right]
\]
and that it admits an optimal solution~$h^\star\in\mc H$ that satisfies
\begin{equation}
    \label{eq:parametric}
   h^\star(X) = \arg\min_{\hat{y}\in\mathbb{R}} \EE[L (\hat{y}, Y)|X] \quad\PP\text{-a.s.}
\end{equation}
As~$\Omega$ is finite, the objective function the parametric optimization problem~\eqref{eq:parametric} can be represented as
$$
    \EE[L (\hat{y}, Y)|X] = \sum_{\omega\in\Omega} \PP_{Y|X}[Y=y_\omega] \, L(\hat{y}, y_\omega).
$$
That is, the conditional expectation reduces to a sum, which depends on~$X$ only indirectly through the conditional probability distribution~$\PP_{Y|X}=(\PP_{Y|X}[Y=y_\omega])_{\omega\in\Omega}$, which is assumed to be independent of the protected attribute~$A$. Thus, the optimal prediction~$h^\star(X)$ depends on~$X$ only indirectly through~$\PP_{Y|X}$, too. Using $\text{I}(Z; Z')$ to denote the mutual information between two random variables~$Z$ and~$Z'$, we then find
$$
    0\leq \text{I}(h^\star(X);A)\leq \text{I}(\PP_{Y|X};A)=0,
$$
where the second inequality exploits the data processing inequality, and the equality holds because~$\PP_{Y|X}$ and~$A$ are independent. Thus, we may conclude that~$\text{I}(h^\star(X);A)=0$, which implies that~$h^\star(X)$ and $A$ are also independent. This observation completes the proof. 
\end{proof}

\begin{proof}[Proof of Proposition~\ref{prop:H_fair}]
Fix any hypothesis~$h\in\mc H$. As~$\Omega$ is finite, $h$ can only adopt finitely many distinct values~$h_k$, $k=1,\ldots,K$. Define~$\Omega_k=\{\omega\in\Omega: h(x_\omega)=h_k\}$ for every~$k=1,\ldots,K$. Then, we have
\begin{align*}
    h(X)\perp A~\iff & ~\EE_{\PP}[\varphi(h(X))|A=0]=\EE_{\PP}[\varphi(h(X))|A=1]\quad\forall \varphi\in \mc L(\R,\R)\\
    \iff & ~ \sum_{\omega\in\Omega} \varphi(h(x_\omega)) \left( \PP(X=x_\omega|A=0)-\PP(X=x_\omega|A=1)\right)=0 \quad\forall \varphi\in \mc L(\R,\R)\\
    \iff & ~ \sum_{k=1}^K \varphi(h_k) \sum_{\omega\in\Omega_k} \left( \PP(X=x_\omega|A=0)-\PP(X=x_\omega|A=1)\right)=0 \quad\forall \varphi\in \mc L(\R,\R)\\
    \iff & ~ \sum_{\omega\in\Omega_k} \PP(X=x_\omega|A=0) = \sum_{\omega\in\Omega_k} \PP(X=x_\omega|A=1) \quad \forall k \in [K].
\end{align*}
Thus, $h\in\mc H$ is SP-fair if and only if there exists a partition $\{\Omega_k\}_{k=1}^K$ of~$\Omega$ such that~$h$ is constant on~$\Omega_k$ and such that~$\Omega_k$ has the same probability under the two conditional distributions~$\PP_{X|A=0}$ and~$\PP_{X|A=1}$ for any~$k$. 

Fix now an arbitrary partition $\{\Omega_k\}_{k=1}^K$ of~$\Omega$ such that~$\Omega_k$ has the same probability under~$\PP_{X|A=0}$ and~$\PP_{X|A=1}$ for any~$k$. Then, the family of all functions~$h\in\mc H$ that are constant on each subset of this partition forms a linear space of SP-fair hypotheses. Note that there is at least one eligible partition (that is, the trivial partition with~$K=1$ and~$\Omega_1=\Omega$), which ensures that~$\mc H_{\rm fair}$ is non-empty. However, there may exist several other (at most finitely many) eligible partitions, each of which contributes a different linear subspace of SP-fair hypothesis. This implies that~$\mc H_{\rm fair}$ represents indeed a union of finitely many linear subspaces of~$\mc H$.
\end{proof}
{\color{black}
\begin{proof}[Proof of Theorem~\ref{thm:hard-fairness-accuracy}]
As $L(\hat y, y)$ is strongly convex in~$\hat y$, the objective function $\varphi(h,\delta)$ is strongly convex in~$h$. Thus, the minimizers $h^\star_\delta$ and $h^\star_{\textrm{fair},\delta}$ exist for all~$\delta\geq 0$. As~$\mc H$ is convex, $h^\star_\delta$ is unique for all~$\delta\geq 0$. In addition, as~$\PP_{Y_0|X}\perp A$, Proposition~\ref{thm:optimal-decision-indep} implies that $h\opt_0 \in\mathcal{H}_{\rm fair}$, which in turn implies that~$h^\star_{{\rm fair},0} =h\opt_0$ is unique. Recall now from Proposition~\ref{prop:H_fair} that~$\mathcal{H}_{\rm fair}$ can be represented as a union of finitely many distinct linear subspaces~$\{\mc H_k\}_{k=1}^K$ of~$\mc H$. We prove the claim first under the assumption that~$h^\star_0$ belongs to exactly {\em one} of these subspaces. Without loss of generality, we may thus assume that~$h^\star_0\in\mc H_1$ and that $\mathcal{H}_1=\{h\in\mathcal H: Sh=0\}$ for some surjective linear map $S:\mc H\to\R^m$ with $m \leq |\Omega|$. By construction, we therefore have
\[
    \min_{h\in\mc H_1} \varphi(h,\delta) < \min_{h\in\mc H_k} \varphi(h,\delta)\quad \forall k=2,\ldots,K
\]
for all sufficiently small $\delta\geq 0$. Indeed, the inequality holds for~$\delta=0$ because~$h^\star_0\in \mc H_1$. By the smoothness and strong convexity assumptions on~$\varphi(h,\delta)$, one can use the implicit function theorem to show that the optimal value function $\min_{h\in\mc H_k}\varphi(h,\delta)$ is continuously differentiable in~$\delta$ for all $k=1,\ldots,K$. Hence, the above inequality holds for all sufficiently small $\delta\geq 0$. This implies that $h^\star_{{\rm fair},\delta}$ constitutes the unique solution of the {\em convex} optimization problem $\min_{h\in\mc H}\{ \varphi(h,\delta):Sh=0\}$ for all sufficiently small~$\delta\geq 0$.

Use now $f(\delta) = \varphi(h\opt_\delta, 0) - \varphi(h\opt_{\textrm{fair},\delta}, 0)$ to denote the excess loss of $h\opt_\delta$ over $h\opt_{\textrm{fair},\delta}$ measured with respect to the true target~$Y_0$. As~$h^\star_{{\rm fair},0} =h\opt_0$, it is clear that~$f(0) = 0$. In the remainder of the proof we will show that~$f(\delta ) > 0$ for all sufficiently small $\delta>0$. Denoting partial derivatives with respect to~$\delta$ by dots, we find 
\begin{equation*}
        \dot f(0) =\langle \nabla_h \varphi(h\opt_0,0), \dot h\opt_{0} \rangle - \langle \nabla_h \varphi(h\opt_{\textrm{fair},0},0), \dot h\opt_{\textrm{fair},0} \rangle = 0,
\end{equation*}

where the second equality holds because $\nabla_h \varphi(h,0)$ vanishes at~$h^\star_{{\rm fair},0} =h\opt_0$. Similarly, we find
\begin{align*}
    \ddot f(0)= &\langle \dot h\opt_{0}, \nabla_h^2\varphi(h\opt_0, 0)\, \dot h\opt_{0}\rangle - \langle \dot h\opt_{\textrm{fair},0}, \nabla_h^2\varphi(h\opt_{\textrm{fair},0},0) \,\dot h\opt_{\textrm{fair},0} \rangle.
\end{align*}
Our goal is to show that $\ddot f(0)>0$. To this end, we need to derive explicit formulas for~$\dot h\opt_0$ and~$\dot h\opt_{\textrm{fair},0}$. For ease of exposition, we introduce some notational shorthands, that is, we define $c\in\R^{|\Omega|}$ and $B\in\mbb S^{|\Omega|}$ as
\[
    c=\nabla_h \dot \varphi(h^\star_0,0)=\nabla_h \dot \varphi(h^\star_{\textrm{fair},0},0)\quad\text{and}\quad B=\nabla^2_h \varphi(h^\star_0,0)=\nabla^2_h \varphi(h^\star_{\textrm{fair},0},0).
\]
Note that~$B$ is positive definite and thus invertible because~$\varphi(h,\delta)$ is strongly convex in~$h$. 
The optimal `unfair' hypothesis $h^\star_\delta$ is uniquely determined by the first-order optimality condition $\nabla_h\varphi(h,\delta)=0$. The implicit function theorem thus implies that $\dot h\opt_0=-B^{-1}c$. Similarly, the optimal `fair' hypothesis $h\opt_{\textrm{fair},\delta}$ is uniquely determined by the stationarity condition $\nabla_h \varphi(h,\delta)+  S^\top u=0$ and the primal feasibility condition $Sh=0$, where $u\in\R^m$ denotes the Lagrange multiplier of the fairness constraint $Sh=0$. 
The implicit function theorem and a standard formula for the inversion of two-by-two block matrices \cite[Proposition~2.8.7]{bernstein2009matrix} then imply that
\begin{equation*}
    \dot h\opt_{\textrm{fair},0}=-\left(B^{-1}-B^{-1} S^\top(SB^{-1}S^\top)^{-1}SB^{-1}    \right)c.
\end{equation*}
Substituting the formulas for~$\dot h\opt_0$ and~$\dot h\opt_{\textrm{fair},0}$ in terms of~$c$ and~$B$ into the formula for~$\ddot f(0)$ yields
\begin{align*}
    \ddot f(0) =& \langle c, B^{-1} c\rangle - \langle c, \left(B^{-1}-B^{-1} S^\top(SB^{-1}S^\top)^{-1}SB^{-1}    \right)^\top  B \left(B^{-1}-B^{-1} S^\top(SB^{-1}S^\top)^{-1}SB^{-1}    \right), c\rangle \\
    =& \langle c, B^{-1} S^\top(SB^{-1}S^\top)^{-1}SB^{-1} c\rangle\geq 0.
\end{align*}
The inequality holds because~$B$ is positive definite and~$S$ has full rank, which implies that~$(SB^{-1}S^\top)^{-1}$ is positive definite, too. In addition, the inequality is {\em strict} unless $SB^{-1}c=0$, that is, unless $\dot h^\star_0=-B^{-1}c \in \mc H_1$. This shows that if~$h^\star_\delta$ is {\em unfair}, that is, if~$h^\star_\delta \notin \mc H_1$, then the optimal fair hypothesis~$h^\star_{\textrm{fair},\delta}$ attains a strictly smaller prediction loss than~$h^\star_\delta$ for all sufficiently small~$\delta>0$.

If~$h^\star_0$ belongs to several different linear subspaces of fair hypotheses, then the above arguments apply to each of these (finitely many) subspaces separately. Details are omitted for brevity.
\end{proof}
}

\begin{proof}[Proof of Lemma~\ref{lemma:soc_fair}]
The claim is an immediate consequence of Definition~\ref{def:ipm} and the relation
\[
    \EE[\psi(h(X)) | A =a ] = \int_\mathbb R \psi(\hat y) \,\PP_{h(X)|A=a}(\dd \hat y)\quad \forall a\in\mc A,
\]
which follows from the variable transformation $\hat y \leftarrow f(\omega)$, where $f : \Omega\to \mathbb R$ is defined through $f(\omega) = h(X(\omega))$ for all $\omega\in\Omega$, and the observation that $\PP_{h(X)|A=a}$ is the pushforward measure of the conditional probability measure~$\PP[\cdot |A=a]$ under the transformation $f$. 
\end{proof}

\begin{proof}[Proof of Lemma~\ref{lem:biased-empirical-risk}]
To avoid clutter, we use~$\PP^a_{\theta}$ as a shorthand for~$\PP_{h_\theta(X)|A=a}$, $a\in\mc A$. As the penalty function~$\rho$ is convex and non-decreasing, a repeated application of Jensen's inequality yields
\begin{align*}
    & \EE\left[ \rho\left(\left. \mc D_{\Psi}(\hat\PP^0_{N,\theta}, \hat\PP^1_{N,\theta} )\right) \right| T^0_N,T^1_N\ge 1\right] \geq \rho\left( \EE\left[  \left. \mc D_{\Psi}(\hat\PP^0_{N,\theta}, \hat\PP^1_{N,\theta} ) \right| T^0_N,T^1_N\ge 1\right] \right)\\
    &\hspace{2cm} =\rho\left(\EE\left[ \left.\sup_{\psi\in \Psi} \left|  \int_{\R} \psi(z)\, \hat\PP^0_{N,\theta}(\dd z) - \int_{\R} \psi(z) \,\hat\PP^1_{N,\theta}( \dd z)\right| \right| T^0_N,T^1_N\ge 1\right]\right)\\
    & \hspace{2cm} \geq \rho\left(\sup_{\psi\in \Psi} \left| \EE\left[ \left. \int_{\R} \psi(z)\, \hat\PP^0_{N,\theta}(\dd z) - \int_{\R} \psi(z) \,\hat\PP^1_{N,\theta}( \dd z) \right| T^0_N,T^1_N\ge 1 \right] \right|\right)\\
    &\hspace{2cm} = \rho\left(\sup_{\psi\in \Psi} \left| \EE\left[ \left. \frac{1}{T_N^0} \sum_{t=1}^{T^0_N} \psi(h_\theta(\hat X_t^0)) - \frac{1}{T_N^1} \sum_{t=1}^{T^1_N} \psi(h_\theta(\hat X_t^0)) \right| T^0_N,T^1_N\ge 1 \right] \right|\right)\\
    & \hspace{2cm} =\rho\left( \sup_{\psi\in \Psi} \left|  \int_{\R} \psi(z)\, \PP^0_{\theta}(\dd z) - \int_{\R} \psi(z) \,\PP^1_{\theta}( \dd z)\right|\right) =\rho\left( \mc D_{\Psi}(\PP^0_{\theta}, \PP^1_{\theta})\right),
\end{align*}
where the first two equalities follow from the definitions of the IPM~$\mc D_\Psi$ and the empirical distributions~$\hat\PP^0_{N,\theta}$ and~$\hat\PP^1_{N,\theta}$, respectively, while the third equality exploits the independence of~$\hat X^a_t$ and~$T^a_N$ for every~$t\in\mathbb N$ and~$a\in\mc A$. Thus, the desired inequality follows. It remains to be shown that this inequality is strict if~$\rho$ is strictly increasing and if $\PP^1_{\theta}=\PP^0_{\theta}$ and is {\em not} a Dirac distribution. In this case, as $T^0_N + T^1_N=N$, the law of total expectation implies that
\begin{align*}
    &\EE\left[  \left. \mc D_{\Psi}(\hat\PP^0_{N,\theta}, \hat\PP^1_{N,\theta} ) \right| T^0_N,T^1_N\ge 1\right]\\
    & \qquad =\sum_{t_0=1}^{N-1}\EE\left[  \left. \mc D_{\Psi}(\hat\PP^0_{N,\theta}, \hat\PP^1_{N,\theta} ) \right| T^0_N=t_0\right]\PP[T^0_N=t_0] >0 = \mc D_{\Psi}(\PP^0_{\theta}, \PP^1_{\theta}),
\end{align*}
where the inequality holds because all terms in the sum on the second line are strictly positive. Indeed, we have~$\PP[T^0_N=t_0]>0$ for every $t_0=1,\ldots,N-1$ because of our assumption that~$\PP[A=a]>0$ for all~$a\in\mc A$. Similarly, as the support of the data-generating distributions $\PP^1_{\theta}=\PP^0_{\theta}$ is not a singleton, the distance between the empirical distributions~$\hat\PP^0_{N,\theta}$ and~$\hat\PP^1_{N,\theta}$ conditional on $T_N^0=t_0$ is strictly positive with a strictly positive probability for every $t_0=1,\ldots,N-1$. The second equality in the above expression holds because $\PP^1_{\theta}=\PP^0_{\theta}$. The claim now follows because $\rho$ is strictly increasing.
\end{proof}

To prove Theorem~\ref{thm:unbiasedness_batch}, we first show that~$\hat U_b(\theta)$ constitutes an {\em unbiased} estimator for the unfairness regularizer in the objective function of problem~\eqref{eq:fair-metric-learning} (Proposition~\ref{prop:unbiased_mmd}). This result critically relies on the randomized construction of the $b$-th batch of training samples. Next, we demonstrate that the empirical prediction loss~$|\mc I_b|^{-1} \sum_{i \in \mc I_b} L(h_\theta(\hat X_i), \hat Y_i)$ constitutes a {\em biased} estimator for the true expected prediction loss in the objective function of problem~\eqref{eq:fair-metric-learning} (Lemma~\ref{lemma:marginal_est_biased} and Proposition~\ref{prop:empirical_loss_Ib_bias}). This is another (yet undesirable) consequence of the randomized construction of the $b$-th batch of training samples. Finally, we prove that~$\hat R_b(\theta)$ constitutes an {\em unbiased} estimator for the expected prediction loss (Lemma~\ref{lemma:correcting_bias_marginals} and Proposition~\ref{prop:empirical_risk_unbiased_est}). This result critically relies on the construction of the bias correction term $\Delta(|\mc I_b|, |\mc I_b^a|)$, which is needed to counteract the effects of the randomness of the $b$-th batch. Theorem~\ref{thm:unbiasedness_batch} is finally proved by combining Proposition~\ref{prop:unbiased_mmd} and Proposition~\ref{prop:empirical_risk_unbiased_est}.

\begin{proposition} 
The statistic~$\hat U_b(\theta)$ constitutes an unbiased estimator for~$d^2_{\rm{MMD}}(\PP_{X | A=0}, \PP_{X | A=1})$.
\label{prop:unbiased_mmd}
\end{proposition}
\begin{proof}
As in the proof of Lemma~\ref{lem:biased-empirical-risk}, we use~$\PP_\theta^a$ as a shorthand for~$\PP_{h_\theta(X) | A=a}$ for every $a \in \mc A$ in order to avoid clutter. To prove that~$\hat U_b(\theta)$ is an unbiased estimator for~$d^2_{\text{MMD}}(\PP_\theta^0, \PP_\theta^1)$, first note that
\begin{align*}
    &\EE\bigg[\frac{1}{|\mc I_{b}^{a}| (|\mc I^{a}_b|\!-\!1)}\sum\limits_{i, j \in \mc I_b^a,\, i\neq j}  K(h_\theta(\hat X_{i}), h_\theta(\hat X_{j})) \bigg]\\
    & \qquad = \EE\bigg[\EE\bigg[ \frac{1}{|\mc I_{b}^{a}| (|\mc I^{a}_b|\!-\!1)}\sum\limits_{i, j \in \mc I_b^a,\, i\neq j}  K(h_\theta(\hat X_{i}), h_\theta(\hat X_{j}))\bigg|\, |\mc I_b^a| \bigg] \bigg]\\
    & \qquad = \int_{\R^n \times \R^n} K(h_\theta(x), h_\theta(x'))\, \PP_\theta^a(\diff x) \, \PP_\theta^a(\diff x')\quad \forall a\in\mc A,
\end{align*}
where the second equality holds because, conditional on~$|\mc I_b^a|$, the samples~$\hat X_i$, $i\in \mc I^a_b$, are independent and governed by the distribution~$\PP_\theta^a$. Similarly, we have
\begin{align*}
    &\EE\bigg[\frac{2}{|\mc I^{0}_{b}| |\mc I^{1}_{b}|} \sum\limits_{\substack{i\in \mc I_{b}^{0},\, j\in \mc I_{b}^{1}}} K(h_\theta(\hat X_{i}), h_\theta(\hat X_{j})) \bigg]\\
    & \qquad = \EE\bigg[\EE\bigg[ \frac{2}{|\mc I^{0}_{b}| |\mc I^{1}_{b}|} \sum\limits_{\substack{i\in \mc I_{b}^{0},\, j\in \mc I_{b}^{1}}} K(h_\theta(\hat X_{i}), h_\theta(\hat X_{j})) \bigg|\, |\mc I_b^0|,\,|\mc I_b^1| \bigg] \bigg]\\
    & \qquad = 2\int_{\R^d \times \R^d} K (h_\theta(x), h_\theta(x')) \, \PP_\theta^0 (\diff x)\, \PP_\theta^1(\diff x').
\end{align*}
The claim then follows directly from the definitions of~$\hat U_b(\theta)$ and~$d^2_{\text{MMD}}(\PP_\theta^0, \PP_\theta^1)$.
\end{proof}

In the following we use~$p_a$ as a shorthand for the marginal probability~$\PP[A = a]$ for each~$a \in \mc A$ in order to avoid clutter. The next lemma is needed to prove that the empirical prediction loss is biased.

\begin{lemma}
Fix any~$a\in\mc A=\{0,1\}$, and define~$a'=1-a$. Then, the statistic~$|\mc I^a_b|/|\mc I_b|$ constitutes a biased estimator for the class probability~$p_a$, that is, we have $\EE\left[{|\mc I^a_b|}/{|\mc I_b|}\right] = p_a (1+ \beta_a)$, where
\begin{equation*}
    \beta_a = p_{a'}^{\bar N- 1} - p_{a'} p_a^{\bar N - 2} + \frac{2p_a}{p_{a'}^2} \left(\log(p_a)+\sum\limits_{N=1}^{\bar N} \frac{p_{a'}^N}{N}  \right) -\frac{2 p_{a'}^2}{p_a^3} \left(\log(p_{a'})+ \sum\limits_{N=1}^{\bar N}  \frac{p_a^N}{N} \right).
\end{equation*}
\label{lemma:marginal_est_biased}
\end{lemma}

\begin{proof}
We first characterize the joint distribution of the two random variables~$|\mc I_b|$ and $|\mc I^a_b|$. As the batch~$\mc I_b$ must contain at least~$\bar N$ samples in total, it is clear that its cardinality~$|\mc I_b|$ may only adopt an integer value~$N\geq\bar N$. In addition, as~$\mc I_b$ contains at least two samples of each class, it is also clear that~$|\mc I_b^a|$ may only adopt an integer value~$n=2,\ldots,N-2$. The probabilities~$\PP[|\mc I_b| = N, |\mc I_b^a| = n]$ of the possible scenarios~$(N,n)$ can be calculated as follows. If $N=\bar N$, then we have
\begin{subequations}
\label{eq:prob}
\begin{align}
    \PP[|\mc I_b| = \bar N, \,\mc I_b^a = n] =
    \begin{cases}
    \binom{\bar N}{n} p_a^n p_{a'}^{\bar N - n} &\text{if}~n = 2,\ldots, \bar N-2, \\ 
    0 &\text{otherwise}.
    \end{cases}
    \label{eq:prob_barN_n}
\end{align}
Indeed, the probability that the batch $\mc I_b$ contains $n$ samples of class~$a$ and $\bar N-n$ samples of class~$a'$ in a particular order is $p_a^n p_{a'}^{\bar N - n}$, and there are $\binom{\bar N}{n}$ possibilities of ordering these samples. Assume next that the batch~$\mc I_b$ contains $N > \bar N$ samples. In this case we have
\begin{align}
    \PP[|\mc I_b| = N, \,\mc I_b^a = n] =\begin{cases}
    (N-1)\, p_a^2 \, p_{a'}^{N-2} &\text{if}~n = 2, \\ 
    (N-1)\, p_a^{N-2} \, 
    p_{a'}^2 &\text{if}~
    n= N-2, \\ 
    0 &\text{otherwise.}
    \end{cases}
    \label{eq:prob_N_n}
    \end{align}
\end{subequations}
Note that the batch cardinality $N=|\mc I_b|$ can only exceed~$\bar N$ if the first~$\bar N$ samples do not include at least two representatives of each class. If there is none or only one sample of class~$a$ among the first~$\bar N$ samples, for example, we continue to add samples until the batch includes exactly two samples of class~$a$. Thus, the last ({\em i.e.}, the $N$-th) sample of the batch must belong to class~$a$, and there must be exactly one other sample belonging to class~$a$. This other sample can reside in any of the first $N-1$ positions in the batch. Thus, the probability of the event $|\mc I_b^a|=2$ is $(N-1)p_a^2p_{a'}^{N-2}$. If there is none or only one sample of class~$a'$ among the first~$\bar N$ samples, then a similar reasoning applies.

We are now ready to calculate the expected value of the estimator~$|\mc I_b^a|/|\mc I_b|$. By~\eqref{eq:prob}, we have
\begin{align}
    \nonumber
    \EE\left[\frac{|\mc I^a_b|}{|\mc I_b|}\right] & = \sum\limits_{N=\bar N}^\infty \sum\limits_{n=2}^{N-2} \, \frac{n}{N} \PP[|\mc I_b| = N ,\, |\mc I_b^a| = n] \\
    & = \sum\limits_{n=2}^{\bar N - 2} \frac{n}{\bar N} \, \PP\left[|\mc I_b = \bar N|,\, |\mc I_b^a| = n \right]  +\sum\limits_{N=\bar N + 1}^{\infty}\frac{2}{N} \, \PP\left[|\mc I_b = N|, \,|\mc I_b^a| = 2 \right]  \label{eq:marginal_estimation_expectation_formula} \\
    &\hspace{2cm}+\sum\limits_{N=\bar N + 1}^{\infty} \frac{N-2}{N} \, \PP\left[ |\mc I_b = N|,\, |\mc I_b^a| = N-2 \right].
\nonumber
\end{align}
Next, we may use~\eqref{eq:prob_barN_n} to reformulate the first sum in~\eqref{eq:marginal_estimation_expectation_formula} as
\begin{align*}
    \sum\limits_{n=2}^{\bar N - 2}\frac{n}{\bar N} \, \PP\left[|\mc I_b = \bar N|, \, |\mc I_b^a| = n \right]  &= \sum\limits_{n=2}^{\bar N - 2} \frac{n}{\bar N} \cdot \frac{\bar N!}{n! (\bar N - n)!} \, p_a^n p_{a'}^{\bar N - n}\\
    &=p_a \sum\limits_{n=2}^{\bar N - 2} \binom{\bar N - 1}{n-1} p_a^{n-1} p_{a'}^{\bar N -n} =  p_a \sum\limits_{n=1}^{\bar N - 3} \binom{\bar N - 1}{n} p_a^{n} p_{a'}^{\bar N -1 - n}\\
    &= p_a \left(1 - p_a^{\bar N - 1} - p_{a'}^{\bar N - 1}- (\bar N - 1) p_a^{\bar N - 2} p_{a'} \right),
\end{align*}
where the last equality follows from the binomial expansion of $(p_a+p_{a'})^{\bar N-1}=1$. Smilarly, we may then use~\eqref{eq:prob_N_n} to reformulate the second sum in~\eqref{eq:marginal_estimation_expectation_formula} as
\begin{align}
    \nonumber
    \sum\limits_{N=\bar N + 1}^{\infty}\frac{2}{N} \, \PP\left[|\mc I_b = N|, \,|\mc I_b^a| = 2 \right] & =
    \sum\limits_{N= \bar N + 1}^\infty \frac{2(N-1)}{N} \, p_a^2 p_{a'}^{N-2}\\ 
    &= \frac{2 p_a^2}{p_{a'}^2} \left(\sum\limits_{N= \bar N +1}^\infty p_{a'}^N \! -\! \sum\limits_{N= \bar N +1}^\infty\frac{p_{a'}^N}{N}\right) \nonumber \\ 
    &= \frac{2p_a^2}{p_{a'}^2} \left(\frac{p_{a'}^{\bar N+ 1}}{1-p_{a'}} \! -\! \left(  \sum\limits_{N=1}^\infty\frac{p_{a'}^N}{N} \!-\! \sum\limits_{N=1}^{\bar N}  \frac{p_{a'}^N}{N}  \right)\right) \label{eq:second_term_formulation}\\
    &=\frac{2p_a^2}{p_{a'}^2} \left(\frac{p_{a'}^{\bar N+ 1}}{1-p_{a'}} \! + \log(1-p_{a'})\!+\! \sum\limits_{N=1}^{\bar N}  \frac{p_{a'}^N}{N}  \right) \nonumber\\
    &=\frac{2p_a^2}{p_{a'}^2} \left(\frac{p_{a'}^{\bar N+ 1}}{p_a} \! + \log(p_a)\!+\! \sum\limits_{N=1}^{\bar N}  \frac{p_{a'}^N}{N}  \right),
    \nonumber
\end{align}
where the third equality exploits the standard formula for infinite geometric series, which applies because~$p_{a'} < 1$, while the fourth equality follows from power series representation of $\log(1-p_{a'})$ ({\em i.e.}, the Newton-Mercator series). The last equality holds because~$p_a = 1-p_{a'}$. Finally, we may use~\eqref{eq:prob_barN_n} once again to reformulate the third sum in~\eqref{eq:marginal_estimation_expectation_formula} as
\begin{align}
    & \sum\limits_{N=\bar N + 1}^{\infty} \frac{N-2}{N} \, \PP\left[ |\mc I_b = N|,\, |\mc I_b^a| = N-2 \right] \; = \sum\limits_{N= \bar N + 1}^\infty \frac{N-2}{N} (N-1)\, p_a^{N-2} p_{a'}^2 \nonumber \\
    & \hspace{3cm}= \sum\limits_{N= \bar N +1}^\infty (N-1)\, p_a^{N-2} p_{a'}^2  - \sum\limits_{N = \bar N + 1}^\infty \frac{2}{N} (N-1)\, p_a^{N-2}p_{a'}^2.
    \label{eq:third_term_formulation}
\end{align}
The first sum in~\eqref{eq:third_term_formulation} is equivalent to
\begin{align*}
    \sum\limits_{N=\bar N + 1}^\infty (N-1) p_a^{N-2} p_{a'}^2 &= p_{a'}^2 \,\frac{\diff}{\diff p_a}\left( \sum\limits_{N= \bar N +1}^\infty p_a^{N-1} \right) = p_{a'}^2 \frac{\diff }{\diff p_a} \left(\sum\limits_{N= \bar N}^\infty p_a^N\right) \\
    &= p_{a'}^2 \frac{\diff}{\diff p_a} \left(\frac{p_a^{\bar N}}{1- p_{a}}\right) = \bar N p_{a'} p_a^{\bar N - 1} + p_a^{\bar N},
\end{align*}
where we use again the standard formula for the geometric series. By swapping the roles of~$a$ and~$a'$ and then repeating the derivations in~\eqref{eq:second_term_formulation}, one can further show that the second term in~\eqref{eq:third_term_formulation} equals
\[
    \sum\limits_{N= \bar N +1}^\infty \frac{2}{N} (N-1) p_a^{N-2} p_{a'}^2 = \frac{2p_{a'}^2}{p_a^2} \left(\frac{p_a^{\bar N+ 1}}{p_{a'}} \! + \log(p_{a'})\!+\! \sum\limits_{N=1}^{\bar N}  \frac{p_a^N}{N}  \right).
\]
Substituting the formulas for the different sums into~\eqref{eq:marginal_estimation_expectation_formula} finally yields
\begin{align*}
    \EE\left[\frac{|\mc I_b^a|}{|\mc I_b|}\right] &=  p_a\left(1 - p_a^{\bar N - 1}  - p_{a'}^{\bar N -1 }- (\bar N - 1) p_a^{\bar N - 2} p_{a'} \right) + \frac{2p_a^2}{p_{a'}^2} \left(\frac{p_{a'}^{\bar N+ 1}}{p_a} \! + \log(p_a)\!+\sum\limits_{N=1}^{\bar N}  \frac{p_{a'}^N}{N}  \right)\\
    &\hspace{1cm} + \bar N p_{a'} p_a^{\bar N - 1} + p_a^{\bar N} -
    \frac{2p_{a'}^2}{p_a^2} \left(\frac{p_a^{\bar N+ 1}}{p_{a'}} \! + \log(p_{a'})\!+\! \sum\limits_{N=1}^{\bar N}  \frac{p_a^N}{N}  \right) = p_a  (1+\beta_a),
\end{align*}
where~$\beta_a$ is defined as in the statement of the lemma. 
\end{proof}
Armed with Lemma~\ref{lemma:marginal_est_biased}, we are now prepared to show that the empirical prediction loss with respect to all the samples in the (random) batch~$\mc I_b$ constitutes a biased estimator for the true expected loss.
\begin{proposition}
The empirical prediction loss~$|\mc I_b|^{-1} \sum_{i \in \mc I_b} L(h_\theta(\hat X_i), \hat Y_i)$ constitutes a biased estimator for the true expected loss~$\EE[L(h_\theta(X), Y)]$, that is, defining $\beta_a$ as in Lemma~\ref{lemma:marginal_est_biased}, we have
\[
    \EE\left[\frac{1}{|\mc I_b|} \sum_{i \in \mc I_b} L(h_\theta(\hat X_i), \hat Y_i)\right] =\EE[L(h_\theta(X), Y)] +\!\sum\limits_{a \in \mc A}  \beta_a\cdot p_a\cdot \EE[L(h_\theta(X), Y) | A = a].
\]
\label{prop:empirical_loss_Ib_bias}
\end{proposition}
\begin{proof}
The expected empirical loss can be reformulated as
\begin{align*}
    \EE\left[\frac{1}{|\mc I_b|} \sum_{i \in \mc I_b} L(h_\theta(\hat X_i), \hat Y_i)\right] &= \EE\left[\frac{1}{|\mc I_b|} \EE\left[\sum\limits_{i \in \mc I_b} L(h_\theta(\hat X_i), \hat Y_i) \Big| |\mc I_b^0|, |\mc I_b^1| \right]\right]\\
    &= \EE\left[\sum\limits_{a\in\mc A}\frac{|\mc I_b^a|}{|\mc I_b|} \EE\left[\sum\limits_{i \in \mc I^a_b} \frac{1}{|\mc I_b^a|} L(h_\theta(\hat X_i), \hat Y_i) \Big| |\mc I_b^a|\right]\right]\\
    &= \sum\limits_{a \in \mc A} \EE\left[ \frac{|\mc I_b^a|}{|\mc I_b|} \right] \EE[L(h_\theta(X), Y) | A= a],
\end{align*}
where the first equality follows from law of iterated conditional expectations, whereas the second equality holds because the sets $\mc I_b^a$, $a\in\mc A$, form a partition of~$\mc I_b$. The third equality exploits our earlier insight that, conditional on fixing the number~$|\mc I_b^a|$ of samples from class~$a$ in the current batch, $\{(\hat X_i, \hat Y_i)\}_{i\in\mc I_b^a}$ represent independent samples from~$\PP_{(X,Y)|A=a}$, implying that the sample average~$|\mc I_b^a|^{-1}\sum_{i \in \mc I_b^a} L(h_\theta(\hat X_i), \hat Y_i)$ constitutes an unbiased estimator for~$\EE[L(h_\theta(X), Y) |A =a]$. Recall now from Lemma~\ref{lemma:marginal_est_biased} that~$\EE[|\mc I_b^a|/ |\mc I_b|] = p_a (1+ \beta_a)$. Thus, we have
\begin{align*}
    \EE\left[\frac{1}{|\mc I_b|} \sum_{i \in \mc I_b} L(h_\theta(\hat X_i), \hat Y_i)\right] &= \sum\limits_{a\in\mc A}p_a  \left(1+ \beta_a \right) \EE[L(h_\theta(X), Y)| A=a],
\end{align*}
and this observation completes the proof.
\end{proof}

Next, we show that the empirical class probability $|\mc I_b^a|/|\mc I_b|$ admits an explicit correction term that eliminates the bias identified in Lemma~\ref{lemma:marginal_est_biased}. To this end, we define~$\Delta(N,n)$ as in the main paper. 

\begin{lemma}
\label{lemma:correcting_bias_marginals}
For any fixed~$a\in\mc A$, the adjusted empirical class probability $\Delta(|\mc I_b|, |\mc I_b^a|)\frac{|\mc I_b^a|}{|\mc I_b|}$ constitutes an unbiased estimator for the class probability~$p_a$.
\end{lemma}
\begin{proof}
Recall from the proof of Lemma~\ref{lemma:marginal_est_biased} that $|\mc I_b|$ can adopt any integer value~$N\ge\bar N$. In addition, recall that if~$N=\bar N$, then $|\mc I_b^a|$ can adopt any integer value $n\in\{2,\ldots,\bar N-2\}$ and that if~$N>\bar N$, then $|\mc I_b^a|$ can adopt only one of the two integer values $n\in\{2,\bar N-2\}$. By the definition of $\Delta(N,n)$, the expected value of the adjusted empirical class probability thus satisfies 
\begin{align}
    &\EE\left[\Delta(|\mc I_b|, |\mc I_b^a|) \frac{|\mc I_b^a|}{|\mc I_b|}\right] = 
    \sum\limits_{n=2}^{\bar N - 2} \frac{n}{\bar N}\, \PP\left[|\mc I_b = \bar N|,|\mc I_b^a| = n \right]  +\!
    \sum\limits_{N=\bar N + 1}^{\infty}\frac{1}{N-1} \, \PP\left[|\mc I_b = N|, |\mc I_b^a| = 2 \right] \nonumber\\
    &\hspace{4cm} +\sum\limits_{N=\bar N + 1}^{\infty} \frac{N-2}{N-1} \, \PP\left[ |\mc I_b= N|, |\mc I_b^a| = N-2 \right].
    \label{eq:expectation_corrected_bias}
\end{align}
For ease of notation, we henceforth use~$a'$ as a shorthand for~$1-a$. From the proof of Lemma~\ref{lemma:marginal_est_biased} we already know that the first sum in~\eqref{eq:expectation_corrected_bias} evaluates to~$p_a (1 - p_a^{\bar N - 1} - p_{a'}^{\bar N - 1}- (\bar N - 1) p_a^{\bar N - 2} p_{a'})$. By~\eqref{eq:prob_N_n}, we may then reformulate the second sum in~\eqref{eq:expectation_corrected_bias} as
\begin{align*}
    \sum\limits_{N=\bar N + 1}^{\infty}\frac{1}{N-1} \, \PP\left[|\mc I_b = N|, |\mc I_b^a| = 2 \right] 
    = p_a^2 \sum\limits_{N = \bar N + 1}^\infty p_{a'}^N= p_a^2 \frac{p_{a'}^{\bar N - 1}}{1- p_{a'}} = p_a p_{a'}^{\bar N - 1},
\end{align*}
where the second equality exploits the formula for infinite geometric series, which applies because~$p_{a'} < 1$, and the third equality follows from the observation that~$p_{a'}=1-p_a$. Using~\eqref{eq:prob_N_n} once again, the third sum in~\eqref{eq:expectation_corrected_bias} can be re-expressed as
\begin{align*}
    & \sum\limits_{N=\bar N + 1}^{\infty} \frac{N-2}{N-1} \, \PP\left[ |\mc I_b= N|, |\mc I_b^a| = N-2 \right] \; =\; p_{a'}^2  \sum\limits_{N = \bar N + 1}^\infty (N-2) p_a^{N-2}\\
    &\hspace{1cm} =\; p_{a'}^2\, p_a \frac{\diff }{\diff p_a}\bigg( \sum\limits_{N = \bar N +1}^\infty p_a^{N-2}\bigg)\; =\; p_{a'}^2 \, p_a \frac{\diff}{\diff p_a} \left(\frac{p_a^{\bar N - 1}}{1-p_a} \right) \; =\; (\bar N - 1) p_{a'}\, p_{a}^{\bar N - 1} + p_a^{\bar N},
\end{align*}
where the third equality follows again from the formula for geometric series. Substituting the formulas for the different sums into~\eqref{eq:expectation_corrected_bias} finally shows that~$\EE[\Delta(|\mc I_b|, |\mc I_b^a|) |\mc I_b^a|/ |\mc I_b|] = p_a$.
\end{proof}

Using Lemma~\ref{lemma:correcting_bias_marginals}, we can now show that the adjusted empirical prediction loss with respect to all the samples in the (random) batch $\mc I_b$ constitutes an unbiased estimator for the true expected loss.

\begin{proposition} 
The statistic~$\hat R_b(\theta)$ constitutes an unbiased estimator for~$\EE[L(h_\theta(X), Y)]$.
\label{prop:empirical_risk_unbiased_est}
\end{proposition}
\begin{proof}
The expectation of the adjusted empirical prediction loss satisfies
\begin{align*}
    &\EE\left[\!\sum\limits_{a \in \mc A}\Delta(|\mc I_b|, |\mc I_b^a|)\frac{1}{|\mc I_b|}
    \sum\limits_{i \in \mc I_b^a} L(h_\theta(\hat X_i), \hat Y_i)\right]\\
    &\hspace{2cm}=\EE\Bigg[\sum\limits_{a \in \mc A} \Delta(|\mc I_b|, |\mc I_b^a|)\frac{|\mc I_b^a|}{|\mc I_b|} \,
    \EE\Bigg[\frac{1}{|\mc I_b^a|}\sum\limits_{i \in \mc I_b^a} L(h_\theta(\hat X_i), \hat Y_i)  \Bigg||\mc I_b^0|, |\mc I_b^1|\Bigg]\Bigg]\\
    &\hspace{2cm}=\EE\Bigg[\sum\limits_{a \in \mc A}\Delta(|\mc I_b|, |\mc I_b^a|)\frac{|\mc I_b^a|}{|\mc I_b|} \,
    \EE[L(h_\theta(X), Y) | A=a]\Bigg]\\
    &\hspace{2cm}=\sum\limits_{a \in \mc A}\PP[A=a]\,\EE[L(h_\theta(X), Y) | A=a] =\EE[L(h_\theta(X), Y)],
\end{align*}
where the first equality follows from law of iterated conditional expectations and the observation that the sets $\mc I_b^a$, $a\in\mc A$, form a partition of~$\mc I_b$. The second equality exploits our earlier insight that, conditional on fixing the number~$|\mc I_b^a|$ of samples from class~$a$ in the current batch, $\{(\hat X_i, \hat Y_i)\}_{i\in\mc I_b^a}$ represent independent samples from~$\PP_{(X,Y)|A=a}$, implying that the sample average~$|\mc I_b^a|^{-1}\sum_{i \in \mc I_b^a} L(h_\theta(\hat X_i), \hat Y_i)$ constitutes an unbiased estimator for~$\EE[L(h_\theta(X), Y) |A =a]$. Finally, the forth equality holds thanks to Lemma~\ref{lemma:correcting_bias_marginals}. Thus, the claim follows.
\end{proof}

We are now ready to prove Theorem~\ref{thm:unbiasedness_batch}.

\begin{proof}[Proof of Theorem~\ref{thm:unbiasedness_batch}]
Propositions~\ref{prop:unbiased_mmd} and~\ref{prop:empirical_risk_unbiased_est} imply that~$\hat R_b(\theta)+ \lambda \hat U_b(\theta)$ represents an unbiased estimator for the objective function of the fair learning problem~\eqref{eq:fair-metric-learning}. To prove the theorem statement, it thus remains to be shown that gradients and expectations can be interchanged. However, this follows immediately from~\cite[Lemma~1]{ref:Glassermann-differentiability}, which applies thanks to Assumption~\ref{ass:regularity}.
\end{proof}

\section{Discussions}
\label{sec:discussions}
This section contains useful background information on kernels and discusses possible extensions of our main results. Specifically, in Section~\ref{app:dist-kernel} we briefly elucidate the connection between the energy distance and the family of MMD metrics, and in Section~\ref{sec:extensions} we sketch extensions of the proposed IPM-based fairness regularizers to other fairness criteria beyond statistical parity. 

\subsection{Distance-Induced Kernel}
\label{app:dist-kernel}
Fix any norm~$\|\cdot\|$ on~$\mathbb R^n$, and set $w(z)=1+\|z\|$. The energy distance with respect to~$\|\cdot\|$ between two probability distributions~$\mathbb Q_1,\mathbb Q_2\in \mathcal Q_w(\mathbb R^n)$ is defined as~\cite{ref:baringhaus2004new, ref:szekely2004testing}
\begin{align*}
    \mc E(\mathbb Q_1, \mathbb Q_2) & = 2\int_{\mathbb R\times\mathbb R} \|z-z'\| \, \mathbb Q_1(\diff z)\, \mathbb Q_2(\diff z')-\int_{\mathbb R\times\mathbb R} \|z-z'\| \, \mathbb Q_1(\diff z)\, \mathbb Q_1(\diff z')\\
    &\hspace{2.5cm} -\int_{\mathbb R\times\mathbb R} \|z-z'\| \, \mathbb Q_2(\diff z)\, \mathbb Q_2(\diff z').
\end{align*}
Recall that the univariate energy distance (for $n=1$) is intimately connected to the $\mc L^2$-distance through the relation~$\mc E(\mathbb Q_1, \mathbb Q_2)=2 d^2_2(\mathbb Q_1, \mathbb Q_2)$ \cite[Theorem~1]{ref:szekely2002E-Statistics}. In addition, the multivariate energy distance (for $n\geq 1$) can also be expressed in terms of an MMD for a suitable choice of the underlying kernel \cite[Theorem~2]{ref:sejdinovic2013equivalence}. To keep this paper self-contained, we derive this expression below.

\begin{definition}[Distance-induced kernel~\cite{ref:sejdinovic2013equivalence}]
The distance-induced kernel $K\in\mc L(\mathbb R^n\times\mathbb R^n, \mathbb R)$ corresponding to the norm~$\|\cdot\|$ on~$\mathbb R^n$ and the anchor point~$z_0\in\mathbb R^n$ is given by
\[
K(z,z') = \frac{1}{2}\left( \|z- z_0\| + \|z' - z_0\| - \|z - z'\|\right).
\]
\end{definition}
The distance-induced kernel is positive definite if and only if the norm $\|\cdot\|$ is of negative type (\cite[Lemma~2.1]{berg1984harmonic}), such as the Euclidean norm~\cite[Proposition~3]{ref:sejdinovic2013equivalence}. In addition, it satisfies the identity $\|z -z'\| = K(z,z) + K(z', z') - 2 K(z, z')$ irrespective of the anchor point~$z_0$. This implies that
\begin{align*}
    \mc E(\mathbb Q_1, \mathbb Q_2) & = 2\int_{\mathbb R^n\times\mathbb R^n} K(z,z) + K(z', z') - 2 K(z, z') \, \mathbb Q_1(\diff z)\, \mathbb Q_2(\diff z')\\
    &\phantom{=} -\int_{\mathbb R^n\times\mathbb R^n} K(z,z) + K(z', z') - 2 K(z, z') \, \mathbb Q_1(\diff z)\, \mathbb Q_1(\diff z')\\
    &\phantom{=}  -\int_{\mathbb R^n\times\mathbb R^n} K(z,z) + K(z', z') - 2 K(z, z') \, \mathbb Q_2(\diff z)\, \mathbb Q_2(\diff z')\\
    & = 2\int_{\mathbb R^n\times\mathbb R^n} K(z, z') \, \mathbb Q_1(\diff z)\, \mathbb Q_1(\diff z') + 2\int_{\mathbb R^n\times\mathbb R^n} K(z, z') \, \mathbb Q_2(\diff z)\, \mathbb Q_2(\diff z')\\
    & \hspace{2.5cm} -4\int_{\mathbb R^n\times\mathbb R^n} K(z, z') \, \mathbb Q_1(\diff z)\, \mathbb Q_2(\diff z')\\
    & = 2\cdot d_{\rm MMD}^2(\mathbb Q_1, \mathbb Q_2),
\end{align*}
where $d_{\rm MMD}$ is the MMD corresponding to~$K$. This derivation shows that the distance-induced kernel corresponding to any fixed anchor point and, in fact, any mixture of distance-induced kernels corresponding to different anchor points generates the energy distance.


\subsection{Extensions to other Fairness Criteria}
\label{sec:extensions}
The ideas of this paper can be generalized along several dimensions. For example, SP-fairness may be enforced at the level of the losses instead of the output distributions. In this case, the conditional loss distributions $\PP_{L(h(X),Y)|A=a}$ must be similar across all~$a\in\mc A$. As explained in Section~\ref{sec:group-fairness}, our techniques for solving fair learning problems readily extend to popular group fairness notions other than SP such as equal opportunity~\citep{ref:hardt2016equality}, probabilistic equal opportunity~\citep{ref:pleiss2017fairness} or log-probabilistic equal opportunity~\citep{ref:taskesen2020distributionally}. In the context of classification, these fairness criteria enforce similarity of the distributions~$\PP_{h(X)|Y=1,A=a}$ across all~$a\in\mc A$. Conditioning on~$Y=1$ requires no new ideas but means that only the positive training samples are used to estimate the fairness of a hypothesis. Another fairness criterion used in classification is equalized odds~\citep{ref:hardt2016equality}, which requires the distributions~$\PP_{h(X)|Y=y,A=a}$ to be similar across all~$a\in\mc A$ and~$y\in\mc Y$. This criterion can be accommodated by introducing separate unfairness penalties for all~$y\in\mc Y$. Last but not least, our methods remain applicable if the protected attribute has more than two realizations or if there is more than one protected attribute. In this case it is again expedient to introduce multiple unfairness penalties. Details are omitted for brevity of exposition.
	
\section{Experiments}
\label{app:experiments}
We now describe all datasets underlying our numerical experiments (Section~\ref{app:numerical}), provide implementation details (Section~\ref{app:trainig_details}) and define the AUC metric used to quantify the accuracy-fairness trade-off (Section~\ref{app:auc}). 

\subsection{Datasets}
	\label{app:numerical}

	\begin{itemize}[leftmargin =3mm, itemsep=0.5mm]
    \item \textbf{Drug}:\footnote{\url{https://archive.ics.uci.edu/ml/datasets/Drug+consumption+\%28quantified\%29}} This dataset contains records for 1885 respondents. Each respondent is described by~12 features including personality type, level of education, age, gender, country of residence and ethnicity. Additionally, each respondent's self-declared drug usage history is recorded. That is, for each drug the respondents declare the time of the last consumption (possible responses are: never, over a decade ago, in the last decade/year/month/week, or on the last day). The classification task is to predict the response `never used' versus `others' ({\em i.e.}, `used') for heroin, and we treat race as the protected attribute. For more details see~\cite{ref:fehrman2017five}.
    
    \item \textbf{Communites$\&$Crime (CC)}:\footnote{\url{http://archive.ics.uci.edu/ml/datasets/Communities+and+Crime}} This dataset contains socio-economic, law enforcement, and crime data for~1,994 different communities in the US described by 99 attributes. We use this dataset both for regression and classification experiments. The regression task is to predict the number of violent crimes per~100,000 residents, and the classification task is to predict whether the incidence of violent crime in a community exceeds the national average. As in~\cite{ref:calders2013controlling}, we create a binary protected attribute by thresholding the percentage of black residents at the median across all communities.
    
    \item \textbf{Compas}:\footnote{\url{https://www.kaggle.com/danofer/compass}} This dataset contains records on 10,000 criminal defendants in Broward County, Florida. The features correspond to the information used by the popular COMPAS~(Correctional Offender Management Profiling for Alternative Sanctions) algorithm to predict recidivism within~2 years of the original offense. Specifically, each defendant is described by~7 features including age, gender and prior offenses as well as race, which we use as the protected attribute. The classification task is to predict whether a defendant will reoffend. The full dataset consists of three subsets, and we only use the one that focuses on violent recividism, which comprises 6,172 records.
    
    \item \textbf{Adult}:\footnote{\url{http://archive.ics.uci.edu/ml/datasets/Adult}} This dataset comprises 45,222 records for different individuals described by 14 features including age, work class, education, marital status, gender, race and yearly income. The classification task is to predict whether or not a person's income exceeds~50,000\$ per year. We use gender as the protected attribute. The dataset is split into fixed training and test sets that contain~32,561 and~12,661 records, respectively. We use these prescribed training and test sets for our experiments.

	\item \textbf{Student Grades}:\footnote{\url{https://archive.ics.uci.edu/ml/datasets/student+performance}} This dataset contains information about the academic performance and about demographic, social and school-related features of students of two Portuguese high schools. After converting all categorical features into binary variables via one-hot encoding, each student is described by 40 features including gender as the protected attribute. The regression task is to predict the final grade (on a scale from 0 to 20). From the full dataset we extract two partial datasets used to predict the final grade in Portuguese language (Student Portuguese, comprising 649 records) and in Mathematics (Student Math, comprising 395 records), respectively.
	
    
\end{itemize}

All datasets are randomly partitioned into training and test sets containing 75\% and 25\% of the samples using the function \code{sklearn.model\_selection.train\_test\_split} from the Scikit-learn toolbox in Python~\citep{ref:sklearn_api} with different seeds. The only exception is the Adult dataset, where this partition is predefined.

\label{app:real_data}

\subsection{Implementation Details}
\label{app:trainig_details}
This section provides additional information on the implementation of the MFL method and the various baseline methods.

\noindent\textbf{Online Learning.}
The MFL method and its biased variant solve problem~\eqref{eq:fair-metric-learning} with the Adam optimizer using the default Adam parameters~$\beta_1=0.9$ and~$\beta_2=0.999$. In the regression experiment, we set the learning rate to 1e-2 without decay, and in the classification experiment, we set the initial learning rate and its decay factor to 5e-4 and 0.9, respectively. The infinite data stream underlying the classification experiment repeatedly outputs a limited number of training samples that differ form the test samples, which can lead to overfitting. To mitigate this effect, we add an $\ell_2$-regularization term with weight 5e-3 to the objective function of the fair leaning problem. This is achieved by using the \code{weight\_decay} option of the Adam optimizer in PyTorch. All codes are available from~\url{https://github.com/RAO-EPFL/Metrizing-Fairness}.


\noindent\textbf{Offline Learning.}
We first provide additional information on the training processes of the MFL method used in the regression experiment as well as the two baseline methods by \citet{ref:berk2017convex}.
\begin{table}[h!]
\caption{Hyperparameters of the~MFL method used in the regression experiments}
\vspace{-5pt}
    \centering
    \begin{tabular}{lccc}
    \toprule
        Dataset & Communities\&Crime & Math &Portugese\\ \midrule
        Batch size $\bar N$ &128 &128 &128\\
        $(\beta_1, \beta_2)$ (Adam) &(0.9, 0.999)&(0.9, 0.999)&(0.9, 0.999) \\
        Learning rate~(LR)  &1e-4 &1e-3 &1e-3\\
        LR decay factor  &None &None &None\\
        Number of epochs &1,000 &2,000 &2,000\\\bottomrule
    \end{tabular}
    \label{tab:hyper-params-reg}
\end{table}

\begin{itemize}[leftmargin = 2.5mm, itemsep=0.5mm]
    \item The code of the proposed MFL method is available from~\url{https://github.com/RAO-EPFL/Metrizing-} \url{Fairness}. Table~\ref{tab:hyper-params-reg} lists all hyperparameters. 
    
    \item Since there is no publicly available code for the two methods by~\citet{ref:berk2017convex}, we reimplemented and calibrated them ab initio using CVXPY~\cite{agrawal2018CVXPY}. To ensure a fair comparison with MFL, which uses no regualizers in addition to the unfairness penalty, we set the weight of the Tikhonov regularization term to~$0$. Full implementation details can be found in our GitHub repository~\url{https://github.com/RAO-EPFL/Metrizing-Fairness}.
\end{itemize}

Next, we provide additional information on the training processes of the MFL method  used in the classification experiment as well as the three baseline methods by \citet{ref:zafar2017fairness},~\citet{ref:cho2020fair} and~\citet{ref:oneto2020expoliting}. For all datasets, we use thresholding at~$0.5$ for converting scores to labels. 
\begin{table}[h!]
    \centering
    \caption{Hyperparameters of the~MFL method used in the classification experiments}
    \vspace{-5pt}
    \begin{tabular}{lcccc}
    \toprule
        Dataset & Drug & Communities\&Crime & Compas &Adult\\ \midrule
        Batch size $\bar N$ & 128 &128 &2,048 &2,048\\
        $(\beta_1, \beta_2)$ (Adam) &(0.9, 0.999)&(0.9, 0.999)&(0.9, 0.999)&(0.9, 0.999) \\
        Learning rate~(LR) &5e-4 &5e-4 &5e-4 &5e-4\\ 
        LR decay factor & 0.99 &0.99 &0.99 &0.99\\
        Number of epochs &500 &500 &500 &500\\\bottomrule
    \end{tabular}
    
    \label{tab:hyper-params}
\end{table}
\begin{itemize}[leftmargin = 2.5mm, itemsep=0.5mm]
    \item The code of the MFL method is available from~\url{https://github.com/RAO-EPFL/Metrizing-Fairness}. Table~\ref{tab:hyper-params} lists all hyperparameters. 
    \item The code of the method by \citet{ref:zafar2017fairness} can be downloaded from~\url{https://github.com/mbilalzafar/fair-classification}. This method can only optimize over linear hypotheses. 
    \item The code of the method by~\citet{ref:cho2020fair} is available from~\url{https://proceedings.neurips.cc/paper/2020/file/ac3870fcad1cfc367825cda0101eee62-Supplemental.zip}. The underlying hypothesis space is the class of all neural networks with one hidden layer accommodating 16 nodes with ReLU activation functions. Note that our MFL method relies on the exact same hypothesis space. For the Adult and Compas datasets, we adopt the hyperparameters proposed in~\cite[Supplementary Material, \S~5.2]{ref:cho2020fair}. For the Adult dataset, we thus set the learning rate to 1e-1 with a decay rate of 0.98, the batch size to 512 and the number of epochs to 200, and for the Compas dataset we set the learning rate to 5e-4 without decay, the batch size to 2,048 and the number of epochs to 500. For all other datasets not considered in~\cite{ref:cho2020fair}, we use the default hyperparameters predefined in the code, that is, we set the learning rate to 2e-4 without decay, the batch size to~2,048 and the number of epochs to 500. Throughout all experiments, we use the default Adam parameters ($\beta_1=0.9$, $\beta_2=0.999$). All of these parameters are chosen to ensure best comparability with MFL.
    
    \item Since there is no publicly available code of the method by~\citet{ref:oneto2020expoliting}, we implemented and calibrated it ab initio. As in the MFL method, we define the hypothesis space as the class of neural networks with one hidden layer accommodating 16 nodes with ReLU activation functions, and we use the cross-entropy loss for training. In addition, we set the regularization parameter of the Sinkhorn divergence to~$0.1$. To ensure a fair comparison with MFL, which uses no regualizers in addition to the unfairness penalty, we further set the weight of the Tikhonov regularization term to~$0$. 
    Finally, we set the total number of iterations of the gradient descent algorithm to 500, initialize the learning rate as~$0.1$ and set the decay factor of the learning rate to~$0.99$. For full implementation details see~\url{https://github.com/RAO-EPFL/Metrizing-Fairness}.
\end{itemize}

\subsection{AUC Metric}
\label{app:auc}

\begin{wrapfigure}{r}{0.4\textwidth}
    \centering
    \includegraphics[width=0.4\textwidth]{ 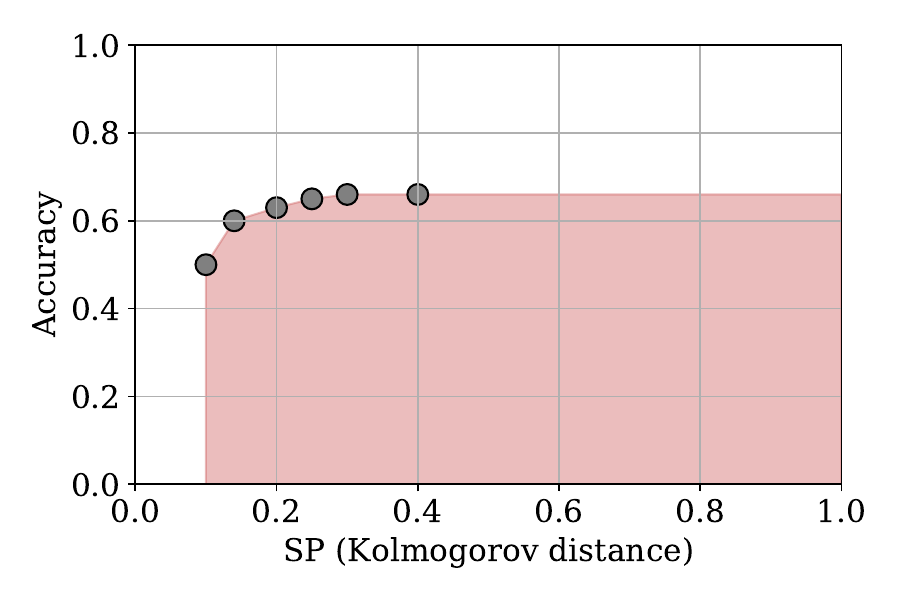}
    \caption{Construction of the AUC}
    \label{fig:auc_explanation}
\end{wrapfigure}
Each method for fair statistical learning considered in this paper involves an accuracy-fairness trade-off parameter. The accuracy versus SP-unfairness plots shown in Section~\ref{sec:numerical} are obtained by sweeping this parameter in equal steps on a logarithmic scale. The most appropriate search grid may depend on the method and the dataset at hand (concrete specifications are given in the main paper). At each grid point we compute an optimal hypothesis on the training set and evaluate its accuracy (correct classification rate for classification tasks or coefficient of determination for regression tasks) as well as its SP-unfairness measure (using the Kolmogorov distance) on the test set. The resulting tuples are conveniently represented as points in the unfairness-accuracy plane; see the gray dots in Figure~\ref{fig:auc_explanation} for a schematic. By the definitions of the correct classification rate, the coefficient of determination and the SP-unfairness measure based on the Kolmogorov distance, all of these points must reside inside the unit box~$[0,1]^2$. 
The red area visualizes all Pareto-dominated points, which attain a smaller accuracy as well as a higher unfairness than at least one of the gray dots, and its boundary can be interpreted as the corresponding Pareto frontier. The {\em area under the curve} (AUC) is then defined as the size of the red area, which is necessarily a number in~$[0,1]$. For an ideal classifier, the red area would span the whole graph, which corresponds to an AUC value of~1. Such a classifier would attain perfect accuracy at zero unfairness. We use the function \code{sklearn.metrics.auc} from the Scikit-learn toolbox in Python~\citep{ref:sklearn_api} to compute the AUC values.

\bibliographystyle{abbrvnat} 
\bibliography{references}

\end{document}